\documentclass[10pt,journal,compsoc]{IEEEtran}
\usepackage{cite}
\usepackage{amsfonts,amsmath,amsthm}
\usepackage{color}
\usepackage[ruled]{algorithm2e}
\usepackage{array,graphicx,subfigure}
\newtheorem{defenition}{Definition}

\newtheorem{assumption}{Assumption}
\newtheorem{remark}{Remark}
\newtheorem{theorem}{Theorem}

\newtheorem{lemma}{Lemma}
\newcommand{\tabincell}[2]{\begin{tabular}{@{}#1@{}}#2\end{tabular}}

\begin{document}

\title{Asynchronous Federated Learning with Differential Privacy for Edge Intelligence}

\author{Yanan~Li, Shusen~Yang, Xuebin~Ren, and Cong~Zhao
\thanks{Y. Li and S. Yang are with School of Mathematics and Statistics, Xi'an Jiaotong University, Shaanxi 710049, P.R. China; X. Ren is with School of Computer Science and Technology, Xi'an Jiaotong University, Shaanxi 710049, P.R. China; C. Zhao is with the Department of Computing, Imperial College London, London SW7 2AZ, UK.
}}


\IEEEtitleabstractindextext{%
\begin{abstract}
Federated learning has been showing as a promising approach in paving the last mile of artificial intelligence, due to its great potential of solving the data isolation problem in large scale machine learning. 
Particularly, with consideration of the heterogeneity in practical edge computing systems, asynchronous edge-cloud collaboration based federated learning can further improve the learning efficiency by significantly reducing the straggler effect. 
Despite no raw data sharing, the open architecture and extensive collaborations of asynchronous federated learning (AFL) still give some malicious participants great opportunities to infer other parties' training data, thus leading to serious concerns of privacy. To achieve a rigorous privacy guarantee with high utility, we investigate to secure asynchronous edge-cloud collaborative federated learning with differential privacy, focusing on the impacts of differential privacy on model convergence of AFL. Formally, we give the first analysis on the model convergence of AFL under DP and propose a multi-stage adjustable private algorithm (MAPA) to improve the trade-off between model utility and privacy by dynamically adjusting both the noise scale and the learning rate. Through extensive simulations and real-world experiments with an edge-could testbed, we demonstrate that MAPA significantly improves both the model accuracy and convergence speed with sufficient privacy guarantee. 

\end{abstract}

}

\maketitle

\IEEEdisplaynontitleabstractindextext

\IEEEpeerreviewmaketitle

\IEEEraisesectionheading{\section{Introduction}\label{sec:introduction}}
\IEEEPARstart{M}{achine} learning (ML), especially the deep learning, can sufficiently release the great utility in big data, and has achieved great success in various application domains, such as natural language processing \cite{hu2014convolutional,hu2018dual}, objection detection \cite{wan2019minientropy,shen2019object}, and face recognition \cite{lu2019learning,ding2019trunk}. 
 However, with the increasing public awareness of privacy, more and more people are reluctant to provide their own data \cite{shokri2015privacy,mohassel2017secureml,mcmahan2018learning}. At the same time, large companies or organizations also begin to realize that the curated data is their coral assets with abundant business value \cite{zhang2016privacy,yang2019federated}. Under such a circumstance, a series of ever-strictest data regulations like GDPR \cite{voigt2017eu} have also been legislated to forbid the arbitrary data usage without user permission as well as any kind of cross-organization data sharing. The increasing concern of data privacy has been causing serious data isolation problems across domains, which poses great challenges in various ML applications.


Aiming to realize distributed ML with privacy protection, federated learning \cite{konevcny2016federated,mcmahan2017communication} (FL) has demonstrated the great potential of conducting large scale ML on enormous users' edge devices or distributed network edge servers via parameter based collaborations, which avoid the direct raw data sharing. For example, Google embedded FL into Android smartphones to improve mobile keyboard prediction without collecting users' input \cite{konevcny2016ondevice}, which may include sensitive data like the credit numbers and home addresses, etc. Besides, with the great ability of bridging up the AI services of different online platforms, FL has been seen as a promising facility for a series of innovative AI business models, such as health-care \cite{liu2018fadl}, insurance \cite{wang2019interpret} and fraud detection \cite{yang2019ffd}. Compared with distributed ML in the Cloud server, FL relies on a large number of heterogeneous edge devices/servers, which would have heterogeneous training progress and cause severe delays for the collaborative FL training. Therefore, asynchronous method has long been leveraged in deep learning to improve the learning efficiency via reducing the straggler effect~\cite{recht2011hogwild,liu2015asynchronousconvergence,sun2017asynchronous,zhang2014asynchronous,hannah2017more}. In this paper, we focus on asynchronous federated learning (AFL) in the context of edge-cloud system with heterogeneous delays \cite{lu2019DPAFL}, as shown in Fig.~\ref{fig:scenario}.

\begin{figure}
	\centering
	\includegraphics[width=0.95\linewidth]{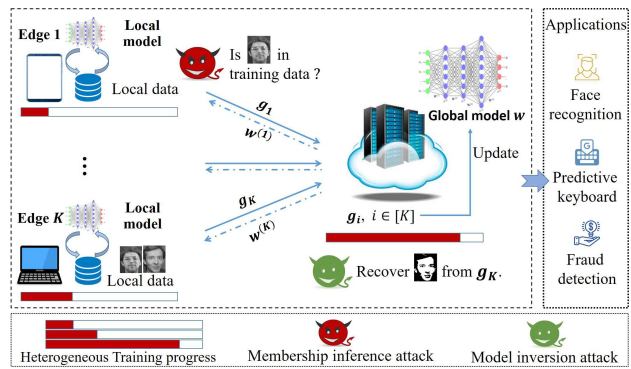}
	\caption{Application scenarios of asynchronous federated learning.}
	\label{fig:scenario}
\end{figure} 

The basic privacy protection of FL benefits from the fact that all raw data are stored locally and close to their providers. However, this is far from sufficient privacy protection. On the one hand, it has been proved that various attacks \cite{fredrikson2015inversion,melis2019exploiting,shokri2017mia} can be launched against
either ML gradients or trained models to extract the private information of data providers. For example, both the membership inference attack \cite{shokri2017mia} and model inversion attack \cite{fredrikson2015inversion} have been validated to be able to infer the individual data or recover part of training data, as shown in Fig.~\ref{fig:scenario}. On the other hand, the open architecture and extensive collaborations make FL systems rather vulnerable to these privacy attacks. Particularly, considering the extensive attacks in other distributed systems like cyber-physical systems, it is not hard to imagine that both the participating edges or the Cloud server in the AFL may act as the honest but curious adversaries to silently infer the private information from the intermediate gradients or the trained models.

To further secure FL, encryption based approaches like secure multi-party computation \cite{mohassel2017secureml} and homomorphic encryption \cite{giacomelli2018privacy} have been proved to be highly effective and able to provide strong security guarantee. However, these approaches are based on complicated computation protocols, leading to potentially unaffordable overheads for edge devices such as mobile phones. Alternatively, by adding proper noises, differential privacy (DP \cite{Dwork2006Calibrating}) can prevent privacy leakage from both the gradients and the trained models with high efficiency, therefore, has also attracted great attentions in machine learning as well as FL~\cite{abadi2016deep,geyer2017differentially,agarwal2018cpsgd,mcmahan2018learning,lu2019DPAFL,shokri2015privacy}. 

Nonetheless, most of the existing work on DP with FL consider synchronous FL, which are different from our research on DP for general edge-cloud collaboration based AFL. 
Specifically, we study the analytical convergence of AFL under DP in this paper. Based on the analytical results, we propose the \underline{m}ulti-stage \underline{a}djustable \underline{p}rivate \underline{a}lgorithm (MAPA), a gradient-adaptive privacy-preserving algorithm for AFL to provide both high model utility  under the rigorous guarantee of differential privacy. 
Our contributions are listed as follows.

\begin{enumerate}
	\item We theoretically analyze the error bound of AFL with considering DP. In particular, the average error bound after $T$ iterations under expectation is dominated by $O\left(\frac{1}{\sqrt{T}}\left(\frac{\sigma}{\sqrt{b}}+\frac{\Delta S}{\varepsilon}\right)+\frac{\tau_{max}^{2}\log T}{T}\right)$ (Theorem \ref{theorem:convex_errorbound}), which extends the result $O(\sigma/\sqrt{bT})$ for general ML and the result  $O\left(\frac{\sigma}{\sqrt{bT}}+\frac{\tau_{max}^{2}\log T}{T}\right)$ for AFL without considering DP. 	
	\item We prove that the gradient norm can converge at the rate $O(1/T)$ to a ball under expectation, the radius of which is determined by the variances of random sampling and added noise. We further propose MAPA to adjust both the DP noise scales and learning rates dynamically to achieve a tighter and faster model convergence without complex parameters tuning.  
    \item
    We conducted extensive simulations and real-world edge-cloud testbed experiments\footnote{Source code available at https://github.com/IoTDATALab/MAPA.} to thoroughly evaluate MAPA's performance in terms of model utility, training speed, and robustness.
    During our evaluation, three types of ML models including logistic regression (LR), support vector machine (SVM), and convolutional neural network (CNN) were adopted.
    Experimental results demonstrate that, for AFL under DP, MAPA manages to guarantee high model utilities.
    Specifically, for CNN training on our real-world testbed, MAPA manages to achieve nearly the same model accuracy as that of centralized training without considering DP.
\end{enumerate}

The rest of this paper is structured as follows. Section \ref{sec:relatedwork} reviews the related work. Section \ref{sec:preliminary} presents the system models of AFL and gives the problem definition. Section \ref{sec:pre_dp} introduces our privacy model of differential privacy. Section \ref{sec:scenario_and_baseline} describes a baseline algorithm with DP and derives the analytical results on its model convergence. Section \ref{sec:algorithm_TAPA} proposes the main algorithms in details and Section \ref{section:Experiments} demonstrates the extensive experimental results. Lastly, we conclude this paper in Section \ref{sec:conclusion}.

\section{Related Work}
\label{sec:relatedwork}

Machine learning privacy has gradually become the crucial obstacle for the data-hungry ML applications \cite{mohassel2017secureml,bonawitz2017practical,papernot2016semi,papernot2018sok,abadi2016deep,geyer2017differentially,agarwal2018cpsgd,mcmahan2018learning,lu2019DPAFL,shokri2015privacy}. In spite of restricting raw data sharing, FL, as a new paradigm of ML, still suffers from various indirect privacy attacks existed in ML, such as membership inference attack \cite{shokri2017mia} and model inversion attack \cite{fredrikson2015inversion}. To enhance the privacy guarantee of FL, many different techniques have been leveraged to prevent the indirect leakage, such as secure multiparty computation \cite{mohassel2017secureml}, homomorphic encryption \cite{giacomelli2018privacy}, secret sharing \cite{bonawitz2017practical}, and differential privacy \cite{abadi2016deep}. However, most of the schemes like secure multiparty computation, homomorphic encryption, secret sharing rely on complicated encryption protocols and would incur unaffordable overheads for edge devices.

Due to the high effectiveness and efficiency, DP has been extensively applied in general machine learning \cite{abadi2016deep,bassily2014private,Chaudhuri2011dpempirical,mcmahan2018general,chanyaswad2018mvg} as well as federated learning algorithms\cite{du2018privacy,shokri2015privacy,agarwal2018cpsgd,geyer2017differentially,thakkar2019differentially,mcmahan2018learning}. 
When implementing DP in machine learning, Laplace or Gaussian mechanism is usually adopted to add properly calibrated noise according to the global sensitivity of gradient's norm, which, however, is difficult to estimate in many machine learning models, especially the deep learning.

For centralized machine learning, \cite{lecuyer2018connection} proposes to leverage the reparametrization trick from \cite{kingma2013auto} to estimate the optimal global sensitivity. Also, \cite{cisse2017parseval} presents the idea of conducting a projection after each gradient step to bound the global sensitivity. However, both \cite{lecuyer2018connection} and \cite{cisse2017parseval} incur great computational overhead in the optimization or projection. Recently, with a slight sacrifice of training utility, \cite{abadi2016deep} introduces to clip the gradient to bound the gradient sensitivity and propose the momentum account mechanism to accurately track the privacy budget. 
However, it remains unclear how to set the optimal clipping bound for achieving a good utility. 

For Federated learning, the similar idea of bounding the global sensitivity is adopted. For example, \cite{shokri2015privacy} samples a subset of gradients and truncate the gradients in the subset, thus reducing the communication cost as well as the variance of the noise. With the similar goal, \cite{agarwal2018cpsgd} designs Binomial mechanism, a discrete version of Gaussian mechanism, to transmit noisy and discretized gradient. Besides the sample-level DP considered in the above research, \cite{geyer2017differentially} proposes to provide client-level DP to hide the existence of participant edge servers and adopts the moment account technique proposed in \cite{abadi2016deep}. Furthermore, \cite{mcmahan2018learning} considers both sample-level and client-level for FedSGD and FedAvg respectively.

In all these works, to reduce the noise, the gradient is clipped by a fixed estimation, which would still incur an overdose of noise in the subsequent iterations since the gradient variance will generally decrease as the model converges. 
Besides, empirical clipping cannot be easily applicable to general ML algorithms. Recently, \cite{thakkar2019differentially} introduces a new adaptive clipping technique for SFL with user-level DP, which can realize adaptive parameter tuning. However, no theoretical analysis on model convergence is given, which means the clipped gradient may not guarantee the convergence or obtain any model utility.

In this paper, we propose an adaptive gradient clipping algorithm by analyzing the impact of DP on AFL model convergence, which ensures that the differentially private AFL model can converge to a high utility model without complicated parameters tuning.

\section{System Models and Problem Statement}
\label{sec:preliminary}
In this section, we introduce the system model of an asynchronous federated learning.

\subsection{Stochastic Optimization based Machine Learning}
Generally, the trained learning model can be defined as the following stochastic optimization problem
\begin{align}\label{eq:stochastic_optimization}
\min_{x \in \mathbb{R}^N}f(x):= \mathbb{E}_{\xi \in \mathcal{P}} F(x;\xi),
\end{align}
where $\xi$ is a random sample whose probability distribution $\mathcal{P}$ is supported on the set $\mathcal{D} \subseteq \mathbb{R}^N$ and $x$ is the global model weight. $F(\cdot,\xi)$ is convex differentiable for each $\xi\in\mathcal{D}$, so the expectation function $f(x)$ is also convex differentiable and $\nabla f(x)=\mathbb{E}_{\xi}[\nabla F(x,\xi)]$. 

\begin{assumption}\label{assumption:unbiased_smooth_variance}
	Assumptions for stochastic optimization.
	\begin{itemize}
		\item (Unbiased Gradient) The stochastic gradient $\nabla F(x,\xi)$ is bounded and unbiased, that is to say,
		\begin{align}\label{eq:bound}
		\|\nabla F(x,\xi)\|\leq G,~\nabla f(x)=\mathbb{E}_{\xi}\nabla F(x,\xi).
		\end{align}
		\item (Bounded Variance) The variance of stochastic gradient is bounded, that is, $\forall x\in\mathbb{R}^N$,
		\begin{align}
		\mathbb{E}_{\xi}[\|\nabla_x F(x,\xi)-\nabla f(x)\|_*^2]\leq \sigma^2.
		\end{align}
		\item (Lipschitz Gradient) The gradient function $\nabla F(\cdot)$ is Lipschitzian, that is to say,  $\forall x,y\in \mathbb{R}^N$,
		\begin{align}
		\|\nabla_x F(x,\xi)-\nabla_x F(y,\xi)\|_*\leq L\|x-y\|.
		\end{align}
	\end{itemize}
\end{assumption}
It should be noted that under these assumptions, $\nabla f(x)$ is Lipschitz continuous with the same constant $L$ \cite{xiao2010dual}.

\subsection{Asynchronous Update based Federated Learning}
\label{sec:pre_asyn}
As shown in Fig. \ref{figure:scenario_APML}, we consider an asynchronous update based federated learning, in which, a common machine learning model is trained via iterative collaborations among a Cloud server and $K$ edge servers. In particular, the Cloud server maintains a global model $x_{t}$ at the $t$-th iteration while each edge server maintains a delayed local model $x_{t-\tau(t,k)}$, where $\tau(t,k)\geq 0$ means the staleness of the $k$-th edge server compared to the current model $x_{t}$. The edge servers and the Cloud server perform the following collaborations during the learning process.
\begin{itemize}
	\item At first, models in the Cloud server and edge servers are initialized as the same $x_{1}$ and the number of iterations $t$ increased by one once the global model in the Cloud server is updated.
	\item Then, the $k$-th edge server at $t$-th iteration computes the gradient $g_{t-\tau(t,k)}$ on a data batch $\mathcal{B}_{k}$ with $b$ random samples $\{\xi_{t,i}\}_{i=1}^{b}$ of its local dataset $\mathcal{D}_{k}$ and sends $g_{t-\tau(t,k)}$ to the Cloud server, where $g_{t-\tau(t,k)}=\frac{1}{b}\sum_{\xi_{i}\in\mathcal{B}}\nabla F(x_{t-\tau(t,k)},\xi_{i})$.
	\item The Cloud server each time picks up a gradient $g_{t-\tau(t)}$ from the buffer $\{g_{t-\tau(t,k)}\}_{k=1}^{K}$ with the "first-in first-out" principle to update the global model from $x_t$ to $x_{t+1}$, which is immediately returned to the corresponding $k(t)$-th edge server for next local gradient computation. 
	\item This collaboration continues until the predefined number of iterations $T$ is satisfied.
\end{itemize}
The AFL architecture in our considered scenario is open and scalable. That means any new edge servers obey the protocols can join in the system and begins training by downloading the trained model from the Cloud server. Then, like the existing edge servers, they can compute the gradient independently and just communicates with the Cloud server.


\begin{assumption}\label{assumption:indepenence_boundeddelay}
	Assumptions for asynchronous update.
	\begin{itemize}
		\item (Independence) All random samples in $\{\xi_{t,i}\}$ are independent to each other, where $t=1,\cdots,T, i=1,\cdots b$;
		\item (Bounded delay) All delay variables $\tau(t,k)$ are bounded: $\max_{t,k}\tau(t,k)\leq \tau_{max}$, where $k=1,\cdots,K$.
	\end{itemize}
\end{assumption}
The independence assumption strictly holds if all edge servers selects samples with replacement. The assumption on bounded delay is commonly used in the  asynchronous algorithms \cite{recht2011hogwild,liu2015asynchronousconvergence,feyzmahdavian2016asynchronous,lian2015asynchronous}. Intuitively, the delay (or staleness) should not be too large to ensure the convergence.


\subsection{Adversary Model}
We focus on data privacy in machine learning and consider a practical federated learning scenario that both the Cloud server and distributed edge servers may be \textit{honest-but-curious}, which means they will honestly follow the protocol without modifying the interactive data but may be curious about and infer the private information of other participant edge servers. In particular, we assume that the untrustworthy Cloud server can infer the private information from the received gradient and some adversarial edge servers may infer the information through the received global models. This adversary model is quite practical in federated learning as all participating entities in the system may locate far from each other but have the knowledge of the training model and related protocols \cite{hitaj2017DeepunderGAN,melis2019exploiting}.

Therefore, in this paper, we aim to design an effective privacy-preserving mechanism for an asynchronous update based federated learning.
For convenience, main notations are listed in Table \ref{table:notations}.
\begin{table}[htbp]
	\caption{Notations}
	\label{table:notations}
	\begin{tabular}{|c|c|}
		\hline
		$\nabla F(x,\xi)$	& gradient computed on a sample $\xi$ \\
		\hline
		$\nabla f(x)$	&  unbiased estimation of $\nabla F(x,\xi)$\\
		\hline
		$g(x)$	& average gradient $1/b\sum_{i=1}^{b}\nabla F(x,\xi_{i})$ \\
		\hline
		$\tilde{g}(x)$	& noisy gradient $\tilde{g}=g(x)+\eta$ \\
		\hline
		$b,\eta$ & mini-batch size, random noise vector\\
		\hline
		$L,\sigma^{2}$	& Lipschitz smooth constant, variance of $\nabla F(x,\xi)$  \\
		\hline
		$\tau_{max},\Delta S$ & maximal delay, global sensitivity in DP\\
		\hline
		$\varepsilon_{k}$ & privacy level for the $k$-th edge server\\
		\hline
		$R,G$	& space radius $R$ and upper bound of $\|\nabla F(x,\xi)\|$ \\
		\hline
		$T,K$ & number of total iterations and edge servers\\
		\hline
		$\Delta_{0}$ &  maximal noise variance $\max_{k=1,\cdots,M}\{2\Delta S^{2}/\varepsilon_{k}^{2}\}$\\
		\hline
		$\Delta_{b}$ & notation denotes $\Delta_{b}=\sigma^{2}/b+\Delta_{0}$\\
		\hline
		$\gamma_{t}$ & the learning rate used in the $t$-th iteration\\
		\hline
	\end{tabular}
\end{table}

\section{Differential Privacy}
\label{sec:pre_dp}

DP is defined on the conception of the adjacent dataset \cite{dwork2014algorithmic}. By adding random noise, DP guarantees the probability of outputting any same result on two adjacent datasets is less than a given constant. In this article, we aim to guarantee the impact of any single sample will not affect the mini-batch stochastic gradient too much by injecting noise from a certain distribution.

\begin{defenition} (Differential Privacy)
	A randomized algorithm $\mathcal{A}$ is $(\varepsilon,\delta)$-DP if for two datasets $\mathcal{D}$, $\mathcal{D'}$ differing one sample, and for all $\omega$, in the output space  $\Omega$ of $\mathcal{A}$, it satisfies that
	\begin{align}
	\Pr[\mathcal{A}(D)=\omega]\leq e^{\varepsilon}\Pr[\mathcal{A}(D')=\omega]+\delta.
	\end{align}
\end{defenition}
The probability is flipped over the randomness of $\mathcal{A}$. The additive term $\delta$ allows for breaching $\varepsilon$-DP with the probability $\delta$. Here $\varepsilon$ denotes the protection level and smaller $\varepsilon$ means higher privacy preservation level.

DP can be usually achieved by adding a noise vector $\eta$ \cite{huang2015differentially,pathak2010multiparty} to the gradient. The norm of the noise vector $\eta$ has the density function as
$$h(\eta;\lambda)=1/(2\lambda)\exp(-\|\eta\|/\lambda)$$
where, $\lambda$ is the scale parameter decided by the privacy level $\varepsilon$ and the global sensitivity $\Delta S$ as $\lambda=\Delta S/\varepsilon$. 
\begin{defenition}(Global Sensitivity $\Delta S$)
	For any two mini-batches $\mathcal{B}$ and $\mathcal{B}'$, which differ in exactly one sample, the global sensitivity $\Delta S$ of gradients is defined as
	\begin{align*}
	\Delta S= \max_{t,\mathcal{B},\mathcal{B}'}\{\|g_{t}(\mathcal{B})-g_{t}(\mathcal{B}')\|\}.
	\end{align*}
\end{defenition}

\section{Baseline Algorithm with Differential Privacy}
\label{sec:scenario_and_baseline}
Before presenting our adaptive-clipping algorithm MAPA for AFL, we first propose a comparable straightforward DP algorithm for AFL based on the system model and analyze its convergence.

%
%

\subsection{AUDP: an Asynchronous Update Federated Learning algorithm with Differential Privacy}

According to the asynchronous federated learning framework listed in Section \ref{sec:pre_asyn}, we propose a baseline scheme, called Asynchronous Update with Differential Privacy (AUDP), to fulfill the privately asynchronous federated learning, in which all edge servers inject DP noise to perturb the gradients before uploading to the Cloud. The detailed collaborations among edge servers and the Cloud server are listed as follows.

On each edge server's side (e.g., the $k$-th edge server), the following steps are performed independently.
\begin{enumerate}
	\item Send the privacy budget $\varepsilon_{k}$ to the Cloud server;
	\item Pull down the current global model $x_{t}$ from the Cloud server;
	\item Compute a noisy gradient $\tilde{g}_{t}\leftarrow g_{t}+\eta_{t}$ by adding a random noise $\eta_{t}$ drawn from the distribution with the density function 
	\begin{align}\label{eq:noise_density}
	h(\eta,\varepsilon_{k})=\frac{\varepsilon_{k}}{2\Delta S}\exp\left(-\frac{\varepsilon_{k}\|\eta\|}{\Delta S}\right);
	\end{align} 
	\item Push $\tilde{g}_{t}$ back to the Cloud server;
\end{enumerate}
Meanwhile, the Cloud server performs the following steps.
\begin{enumerate}
	\item At the current $t$-th iteration, pick a stale gradient $\tilde{g}_{t-\tau(t,k)}$ delayed by $\tau(t,k)$ iterations provided by the $k(t)$-the edge server from the buffers, where $\tau(t,k)$\footnote{For simplicity, $\tau(t,k)$ is written as $\tau(t)$ later.} ranging from $0$ to the maximum delay $\tau_{max}$;
	\item Update the current global model $x_{t}$ using gradient descent method
	\begin{align*}
	x_{t+1}=x_{t}-\gamma_{t} \tilde{g}_{t-\tau(t)},
	\end{align*}
	where, $\gamma_{t}$ is the learning rate at $t$-th iteration and has relation to $\varepsilon_{k}$;
	\item Send $x_{t+1}$ to the $k(t)$-th edge server;
\end{enumerate}

The basic workflow of AUDP is also shown in Fig.~\ref{figure:scenario_APML}. For example, based on the global model $x_{2}$, edge server 3 computes a local gradient $g_{2}$ and sends a noisy gradient $\tilde{g}_{2}$ to the buffer in the Cloud server. When $\tilde{g}_{2}$ is picked up, the original model $x_{2}$ has been updated by 6 updates and becomes $x_{8}$ at now. So, the Cloud server has to use the stale gradient $\tilde{g}_{2}$ to update $x_{8}$ and sends the newly updated $x_{9}$ back to edge server 3 for the next local computing. Other edge servers perform a similar process without waiting for others.



\begin{figure}[tbp]
	\centering
	\includegraphics[width=0.8\linewidth]{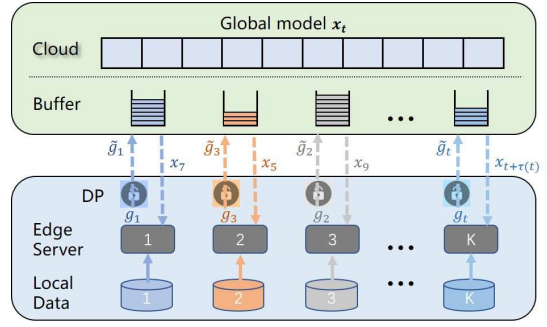}
	\caption{An secure asynchronous federated learning framework.} 
	\label{figure:scenario_APML}
\end{figure}

Now, we prove that the $t$-th iteration of AUDP satisfies $\varepsilon_{k(t)}$-DP.
\begin{theorem}\label{theorem:varepsilon_DP}
	Assume the upper bound of the gradients is $G$, i.e., $\|\nabla F(x,\xi)\|\leq G$ for all $x$ and $\xi$. If the global sensitivity is set as $\Delta S=2G/b$ and the noise is drawn from the distribution in Eq.\;(\ref{eq:noise_density}), then the $t$-th iteration of AUDP satisfies $\varepsilon_{k(t)}$-DP.
\end{theorem}

\begin{proof}
	For any two mini-batches differing one sample denoted as $\xi_{b}\in\mathcal{B}$ and $\xi_{b}'\in\mathcal{B}'$ without loss of the generality, because
	\begin{align*}
	&\max_{t,\mathcal{B},\mathcal{B}'}\{\|g_{t}(\mathcal{B})-g_{t}(\mathcal{B}')\|\}\\
	&=\max_{t,\mathcal{B},\mathcal{B}'}\left\{\|\nabla F(x,\xi_{b})-\nabla F(x,\xi_{b}')\|/b\right\}\leq 2G/b,
	\end{align*}
	so the global sensitivity $\Delta S=2G/b$. For any possibly noisy gradient $\nu$, we have
	\begin{align*}
	&\frac{\Pr\{g_{t}(\mathcal{B})+\eta=\nu\}}{\Pr\{g_{t}(\mathcal{B}')+\eta=\nu\}}
	=\frac{\exp(-\varepsilon_{k(t)}\|\nu-g_{t}(\mathcal{B})\|/\Delta S)}{\exp(-\varepsilon_{k(t)}\|\nu-g_{t}(\mathcal{B}')\|/\Delta S)}\\
	&\leq \exp\left(\frac{\varepsilon_{k(t)}\|g_{t}(\mathcal{B})-g_{t}(\mathcal{B}')\|}{\Delta S}\right)\leq \exp(\varepsilon_{k}).
	\end{align*}
	So, the $t$-th iteration of AUDP satisfies $\varepsilon_{k(t)}$-DP.
\end{proof}

\subsection{Convergence Analysis of AUDP }

Without the consideration of DP, the known order of the optimal convergence rate for convex function with smooth gradient in asynchronous update is $O(1/\sqrt{Tb})$ in terms of the iteration number $T$ \cite{recht2011hogwild,agarwal2011distributed,agarwal2018cpsgd,feyzmahdavian2016asynchronous}. Here, we extended the corresponding analytical result with the consideration of DP. In particular, the optimal convergence rate for AUDP has the order of $O(\sqrt{\Delta_{b}/T})$, where $\Delta_{b}=\sigma^{2}/b + \Delta_{0}$ and $\Delta_{0}=\max_{k=1,\cdots,M}\{2\Delta S^{2}/\varepsilon_{k}^{2}\}$.

The convergence of AUDP is shown in the following lemma. 

\begin{lemma}\label{lemma:convex_errorbound}
	Let Assumptions \ref{assumption:unbiased_smooth_variance} and \ref{assumption:indepenence_boundeddelay} hold. Then the output of AUDP satisfies the following result
	\begin{align*}
	& \sum_{t=1}^{T}\mathbb{E}f(x_{t+1})-f(x^*) \\
	&\leq RG\tau_{max} +\frac{L(\tau_{max}+1)^2}{2}\sum_{t=1}^{T}\gamma_{t}^2\mathbb{E}\|\tilde{g}_{t-\tau(t)}\|^{2}\\
	&\quad + \sum_{t=1}^{T}\gamma_{t}\mathbb{E}\|\tilde{g}_{t-\tau(t)}-\nabla f(x_{t-\tau(t)})\|^2 \\
	&\quad - \sum_{t=1}^{T}\frac{1}{2}(\gamma_{t} - L\gamma_{t}^2)\mathbb{E}\|\tilde{g}_{t-\tau(t)}\|^2\\
	&\quad +\sum_{t=1}^{T}\frac{1}{2\gamma_{t}}[\|x_{t}-x^*\|^2 - \|x_{t+1}-x^*\|^2].
	\end{align*}
\end{lemma}
\begin{proof}
	See Appendix \ref{proof:lemma_convex_errorbound}.
\end{proof}

When $\gamma_{t}\in(0,1/L)$, The term $- \sum_{t=1}^{T}\frac{1}{2}(\gamma_{t} - L\gamma_{t}^2)\mathbb{E}\|\tilde{g}_{t-\tau(t)}\|^2$ can be removed. Noticing that
\begin{align}\label{eq:eq1}
\mathbb{E}\|\tilde{g}_{t-\tau(t)}\|^{2} &=\mathbb{E}\|\tilde{g}_{t-\tau(t)}-\nabla f(x_{t-\tau(t)})\|^{2}\\
&\quad + \mathbb{E}\|\nabla f(x_{t-\tau(t)})\|^{2}\notag
\end{align}
and
\begin{align}\label{eq:eq2}
&\mathbb{E}\|\tilde{g}_{t-\tau(t)}-\nabla f(x_{t-\tau(t)})\|^{2}\\
&=\mathbb{E}\|g_{t-\tau(t)}-\nabla f(x_{t-\tau(t)})\|^{2}+\mathbb{E}\|\eta_{t}\|^{2}\\\notag
&\leq \sigma^{2}/b + 2\Delta S^{2}/\varepsilon_{k(t)}^{2}\leq \Delta_{b},\notag
\end{align}
we can obtain the following theorem by substituting Eqs.~(\ref{eq:eq1}) and (\ref{eq:eq2}) into Lemma \ref{lemma:convex_errorbound}.

\begin{theorem}\label{theorem:convex_errorbound}
	Assume that Assumptions \ref{assumption:unbiased_smooth_variance} and \ref{assumption:indepenence_boundeddelay} hold. Let $\|\nabla f(x)\|\leq G$,  $\|x_{t}-x^{*}\|\leq R$ and $x_{ave}(T)=1/T\sum_{t=1}^{T}x_{t}$. If the learning rate $\gamma_{t}$ satisfies
	\begin{align}\label{eq:convex_errorbound_varingrate}
	\gamma_{t}^{-1}= L(\tau_{max}+1)+\sqrt{\Delta_{b}+1}\sqrt{t}
	\end{align}
	then, the average error of AUDP under expectation satisfies
	\begin{align*}
	&\mathbb{E}f(x_{ave}(T))-f(x^{*})\leq\frac{1}{T}\sum_{t=1}^{T}\mathbb{E}f(x_{t})-f(x^{*})\\
	&\leq\frac{RG\tau_{max}}{T} + \frac{L(\Delta_{b}+G^{2})}{2(\Delta_{b}+1)}\frac{(\tau_{max}+1)^{2}\log T}{T}\\
	&\quad +\frac{(4+R^{2})\sqrt{\Delta_{b}+1}+R^{2}/\gamma_{1}}{2\sqrt{T}},
	\end{align*}
\end{theorem}
\begin{proof}
	See Appendix \ref{proof:theorem_convex_errorbound}.
\end{proof}

Theorem \ref{theorem:convex_errorbound} claims that AUDP can converge even when the gradient is out-of-the-date and perturbed by noises. From the proofs, this result holds for any noise distribution with zero mean and bounded variance. This is consistent with \cite{mania2017perturbed,chaturapruek2015asynchronous}, which regards the stale gradient as a perturbation of the current gradient. In Theorem \ref{theorem:convex_errorbound}, $\Delta_{b}=\sigma^{2}/b+\Delta_{0}$ reflects the error caused by the randomness in both batch sampling and noise for privacy preserving.

\begin{remark}
	Fix other parameters, we can observe that, to achieve $\varepsilon$-DP, the error bound of AUDP has the order of $O(\Delta S^{2}/\varepsilon^{2})$, which can be derived from $\sqrt{(\Delta_{b}+1)/{T}}$. This result is consistent with \cite{huang2015differentially}, which shows, the higher global sensitivity and privacy level require much more (polynomial order) iterations to achieve the same error bound.
\end{remark}
\begin{remark}
	Without consideration of DP (i.e., $\Delta_{b}=\sigma^{2}/b$), the average error of Theorem \ref{theorem:convex_errorbound} is simplified as
	\begin{align*}
	O\left(\frac{RG\tau_{max}}{T}+\frac{LG^{2}(\tau_{max}+1)^{2}\log T}{T}+ \frac{R^{2}\sigma}{\sqrt{bT}}\right).
	\end{align*}
	Therefore, the convergence rate achieves $O(1/\sqrt{T})$ as long as $\tau_{max}=O(T^{1/4})$, which is known to be the best achievable rate of convex stochastic optimization \cite{bubeck2015convex}. This means that the penalty in convergence rate due to the delay $\tau(t)$ is asymptotically negligible.
	
	Furthermore, the $\log T$ factor in the last second term is not present when $\gamma_{t}$ is set as
	\begin{align}\label{eq:convex_errorbound_varingrate2}
	\gamma_{t}^{-1}=L((\tau_{max}+1)^{2}+1)+\sqrt{\Delta_{b}+1}\sqrt{t},
	\end{align}
	which satisfies $\gamma_{t}-L\gamma_{t}^{2}-L(\tau_{max}+1)^{2}\leq 0.$
	In such case, the result becomes
	\begin{align*}
	O\left(\frac{RG\tau_{max}}{T}+\frac{R\sigma}{\sqrt{bT}}\right),
	\end{align*}
	which is better than $O(\frac{LR^{2}(\tau_{max}+1)^{2}}{T}+\frac{R\sigma}{\sqrt{bT}})$ at the factor $(\tau_{max}+1)^{2}$ (Theorem 2 of \cite{feyzmahdavian2016asynchronous}).
\end{remark}
\begin{remark}
	 Stale gradient can accelerate the training process if it is not too old. In the analysis of Theorem \ref{theorem:convex_errorbound}, the term
	 \begin{align*}
	 -1/2\sum_{t=1}^{T}(\gamma_{t}-L\gamma_{t}^{2})\mathbb{E}\|\tilde{g}_{t-\tau_{(t)}}\|^{2}
	 \end{align*}
	 originally in Lemma~\ref{lemma:convex_errorbound} is neglected for simplicity, which, however, can be used to eliminate part of other terms to reduce the error bound if $\|\tilde{g}_{t-\tau_{(t)}}\|$ has a lower bound.
	 
	 In fact, the lower bound can be commonly hold in the beginning of learning when the model is far away from the optimum. But if the lower bound still holds when the model is close enough to the optimum, the stale gradient will then harm the convergence. This means that too large staleness is not allowed in the asynchronous update (Assumption \ref{assumption:indepenence_boundeddelay}). The observation that a stale gradient may speed up the training is also consistent with \cite{mitliagkas2016asynchrony}.
\end{remark}

\section{Multi Stage Adjustable Private Algorithm for Asynchronous Federated Learning}
\label{sec:algorithm_TAPA}

In this section, we theoretically analyze how to estimate the global sensitivity and improve the model utility of the baseline algorithm AUDP. Subsequently, we propose the multi stage adjustable private algorithm (MAPA) to train general models by automatically adjusting the learning rate and the global sensitivity to achieve a better trade-off between model utility and privacy protection.

\subsection{Basic Idea}

In AUDP, an unsolved problem is how to estimate the parameter $G$ in Eq.\;(\ref{eq:bound}), which is the upper bound of gradients norm $\|\nabla F(x,\xi)\|$ and determines the noise scale $\lambda=\Delta S/\varepsilon=2G/b\varepsilon$. However, due to the complicated trained model $x$ and the randomness of sampling $\xi$, it is impossible to obtain an accurate value of $G$ while training. Therefore, to limit the noise, many existing work proposed to clip the gradient using an fixed bound $\bar{G}$ and calibrate the privacy noise scale as $2\bar{G}/b\varepsilon$. Nonetheless, this does not consider the fact that the gradients norm decreases with the training process and will lead to either an overestimated or underestimated estimation, as shown in Fig. \ref{fig:basic_idea} (a)-(c). For example, if $\bar{G}$ is larger than $G$, the global sensitivity $\Delta S=2\bar{G}/b$ will incur too more noise to the gradients, leading to a poor model accuracy (Fig.~\ref{fig:basic_idea} (b)). If $\bar{G}$ is much smaller than $G$, clipping may destroy the unbiasedness of the gradient estimate, also leading to a poor model accuracy (Fig.~\ref{fig:basic_idea} (c)). 
Although an adaptive clipping method is proposed in \cite{thakkar2019differentially}, it remains unclear how to set the learning rates based on the introduced noises to ensure the model convergence, making its adaptive method meaningless when the training is not convergent. 

To this end, we theoretically analyze the convergence of AFL with DP and study the relationship between the learning rate and AFL model convergence under DP. Inspired by the relationship, we propose an adaptive clipping method to improve the model accuracy of AUDP by changing the learning rates to ensure the gradients norm decreases below an expected level after some iterations. After reaching the expected level, we adjust the learning rate once again to make the gradient norm further converge. According to different learning rates, the training process is divided into different stages (Fig.~\ref{fig:basic_idea} (d)). By suppressing the gradients norm stage-wise, we can reduce the noises and improve the model utility while still providing the sufficient privacy protection. 
\begin{figure*}
	\centering
	\includegraphics[width=0.9\linewidth]{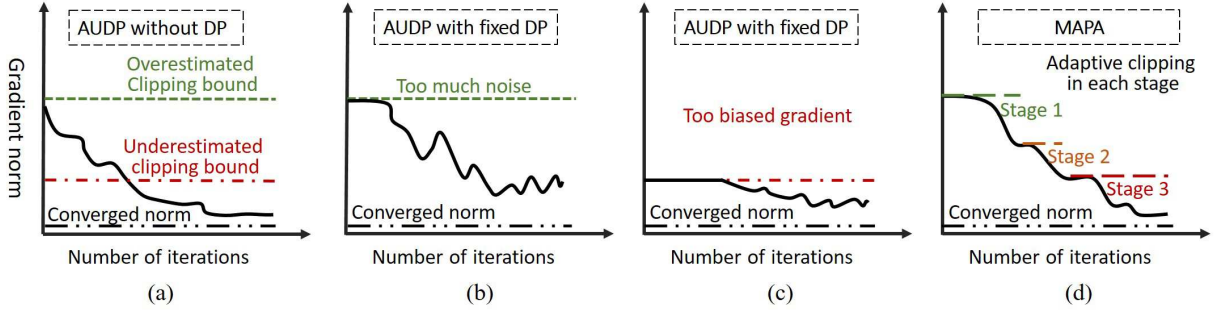}
	\caption{Illustration of multi stage adjustable DP mechanism.}
	\label{fig:basic_idea}
\end{figure*}

\subsection{Adaptive Gradient Bound Estimation}
We first show how to estimate the global sensitivity $\Delta S$ at the beginning. 

\begin{theorem}\label{theorem:gloal_sensitivity_cheb}
	For any failure probability $0<\delta<1$, if the global sensitivity $\Delta S$ satisfying
	\begin{align}\label{eq:global_sensitivity_cheb}
	\left(1-{4\sigma^2}/{(b^2\Delta S^2)}\right)^2\geq 1-\delta,
	\end{align}
	then the $t$-th iteration of AUDP satisfies $(\varepsilon_{k(t)},\delta)$-DP, where $k(t)$ means the noisy gradient is received from the $k(t)$-th edge server.
\end{theorem}
\begin{proof}
	For any two adjacent mini-batches differing the last sample, we have
	\begin{align*}
	&\Pr\{\|g_{t}(\mathcal{B})-g_{t}(\mathcal{B}')\|\leq \Delta S\}\\
	& = \Pr \left\{ \mathbb{E}\|\nabla F(x,\xi_{n})-\nabla F(x,\xi_{n}')\|\leq b\Delta S \right\}\\
	& \geq \Pr \{ \mathbb{E}\|\nabla F(x,\xi_{n})-\nabla f(x)\|  \\
	&\quad + \mathbb{E}\|\nabla f(x)-\nabla F(x,\xi_{n}')\|\leq b\Delta S \}\\
	& \geq \Pr \{ \mathbb{E}\|\nabla F(x,\xi_{n})-\nabla f(x)\|\leq b\Delta S/2\}\cdot\\
	& \quad \Pr \{ \mathbb{E}\|\nabla F(x,\xi_{n}')-\nabla f(x)\|\leq b\Delta S/2\}\\
	& \geq \left(1-{4\sigma^2}/{(b^2\Delta S^2)}\right)^2.
	\end{align*}
	So, according to Theorem \ref{theorem:varepsilon_DP}, if the sensitivity satisfies Eq.\;(\ref{eq:global_sensitivity_cheb}), the output of AUDP is $\varepsilon_{k(t)}$-DP with probability $1-\delta$. In other words, AUDP guarantees $(\varepsilon_{k(t)},\delta)$-DP.
\end{proof}

The Cloud server can set different $\Delta S$ to satisfy different requirement (i.e., the failure probability $\delta$) of edge servers based on Theorem \ref{theorem:gloal_sensitivity_cheb}. However, $\Delta S$ may be quite larger than the actual global sensitivity and will introduce predominant noise to gradients, possibly leading to the failure of model convergence. Therefore, to begin with a large global sensitivity $\Delta S$, we should adjust and update $\Delta S$ dynamically to ensure the model convergence while guaranteeing the privacy. In particular, considering that gradient converges with the convergence of model, we first analyze the convergence of the gradient. Theorem \ref{theorem:errorbound_gradient} shows that we can adjust the learning rate to ensure the convergence of the gradient norm.



\begin{theorem}\label{theorem:errorbound_gradient}
	Assume that Assumptions \ref{assumption:unbiased_smooth_variance} and \ref{assumption:indepenence_boundeddelay} hold. If the learning rate $\gamma_{t}$ is a constant $\gamma$ satisfying
	\begin{equation}\label{eq:stepsize_constant}
	\gamma^{-1} \geq 2L(\tau_{max}+1),
	\end{equation}
	then the output of AUDP satisfies the following result
	\begin{align}
	&\min_{t\in\{1,\cdots,T\}}\mathbb{E}\|\nabla f(x_{t})\|^{2}\leq \frac{1}{T}\sum_{t=1}^{T}\mathbb{E}\|\nabla f(x_{t})\|^{2}\notag\\
	&\leq \frac{2(f(x_{1})-f(x^{*}))}{T\gamma}+ 2\Delta_{b}L\gamma.\label{eq:errorbound_gradient}
	\end{align}
\end{theorem}
\begin{proof}
	See Appendix \ref{appendix:proofs_errorbound_gradient}.
\end{proof}

Theorem \ref{theorem:errorbound_gradient} shows that AUDP algorithm can converge to a ball at the rate $O(1/T)$ with a constant learning rate. Therefore, the average norm of gradient must have a upper bound relate to $\Delta_{b}$ after sufficient iterations. Recall $\Delta_{b}=\sigma^{2}/b+\Delta_{0}$, i.e., the radius of the ball consists of two parts: sampling variance $\sigma^{2}/b$ and noise variance $\Delta_{0}$. Due to Theorem \ref{theorem:gloal_sensitivity_cheb}, $\Delta_{0}$ is inversely proportional to $b$. Meanwhile, sampling variance $\sigma^{2}/b$ is also inversely proportional to $b$. Therefore, we can increase the mini-batch size to reduce the radius to control the upper bound.

In the following, we illustrate how to use Theorem \ref{theorem:errorbound_gradient} to set the learning rate to reduce the global sensitivity gradually. Let the learning rate be
\begin{align*}
\gamma=1/(2PL(\tau_{max}+1)),
\end{align*}
where $P$ is an undetermined coefficient and $P\geq 1$ satisfies Eq.(\ref{eq:stepsize_constant}). Then, the right hand side of Eq.\;(\ref{eq:errorbound_gradient}) becomes
\begin{align*}
\frac{4PL(\tau_{max}+1)(f(x_{1})-f(x^{*}))}{T}+\frac{\Delta_{b}}{P(\tau_{max}+1)}.
\end{align*}
Let the first term be less than $\frac{\Delta_{b}}{P(\tau_{max}+1)}$, we can derive that
\begin{align}\label{eq:multi_stage_T0}
T\geq T_{0}=\frac{4P^{2}L(\tau_{max}+1)^{2}(f(x_{1})-f(x^{*}))}{\Delta_{b}}.
\end{align}

Then the right hand side of Eq.\;(\ref{eq:errorbound_gradient}) becomes $\frac{2\Delta_{b}}{P(\tau_{max}+1)}$. Therefore, the upper bound of the gradient's norm is estimated as $\sqrt{2\Delta_{b}/(P(\tau_{max}+1))}$ and the new global sensitivity after $T_{0}$ iterations is estimated as $2\sqrt{2\Delta_{b}/(P(\tau_{max}+1))}/b$,  according to Theorem \ref{theorem:varepsilon_DP}. Denote the initial estimation by Theorem \ref{theorem:gloal_sensitivity_cheb} as $\Delta S$ and the new estimation as $\Delta S'$. Note that our purpose is to reduce the global sensitivity gradually. Therefore, making the new estimation less than the initial estimation, i.e.,
\begin{align*}
\Delta S'\leq \theta \Delta S,
\end{align*}
where $\theta\in(0,1)$ is used to control the reduction ratio. We further derive that
\begin{align}\label{eq:multi_stage_P}
P\geq \frac{8\Delta_{b}}{(\tau_{max}+1)b^{2}\Delta S^{2}\theta^{2}}.
\end{align}

Therefore, if we use the above $P$ to set $\gamma$, $\Delta S$ is reduced to $\theta\Delta S$. To avoid the randomnesses of sampling and noise, we use $\sqrt{2\Delta_{b}/(P(\tau_{max}+1))}$ to clip the gradient to ensure $\Delta S_{2}$ is the new global sensitivity in the following training after $T_{0}$ iterations. We can repeat this process to gradually reduce the global sensitivity while ensuring model convergence. 

\subsection{Multi Stage Adjustable Private Algorithm (MAPA)}
With the above analysis, we propose the Multi-Stage Adjustable Private Algorithm (MAPA) to adjust the global sensitivity and the learning rate dynamically according to the varying gradient during the training process to achieve a better model utility without complicated parameter tuning.
The formal description of MAPA is shown in Algorithm \ref{algorithm:MAPA}. We give the explanations as follows.
\begin{itemize}
	\item In the initialization phase ($t=1$), all edge servers send their privacy budget $\varepsilon_{k}$ ($k=1,...,K$) to the Cloud server, which then identifies the minimal privacy budget $\varepsilon_{0}$ and initializes the model $x_t$ and $\Delta{S}$ according to Theorem \ref{theorem:gloal_sensitivity_cheb}. (Line 1 on the edge server and Lines 1$\sim$4 on the Cloud server)
	\item The process on the Cloud server is divided into different stages. From the beginning, the Cloud server runs in the first stage. In each stage, the Cloud server computes the intermediate parameter $P$, the learning rate $\gamma$, and the needed iteration number $T_s$ for the current stage. (Lines 6$\sim$8 on the Cloud server) 
	\item Once the training begins, each edge server pulls down the model $x_t$ and $\Delta{S}$ from the Cloud server, and computes the gradient $g_t$ on the local mini-batch. Then, it clips and perturbs the gradient as $\tilde{g}_t$, which is sent to the Cloud server with privacy protection. Since the edge servers are heterogeneous in computation and communication, they would generally complete these procedures independently in different time. (Lines 3$\sim$8 on the edge server)
	\item In each stage, once the Cloud server receives a stale gradient $\tilde{g}_{t-\tau(t)}$ from any edge server $k(t)$, they will update the model $x_t$ immediately and sends the updated model $x_t$ and the current global sensitivity $\Delta_{S}$ to the corresponding edge sever $k(t)$. The process repeats until the model is updated by $T_s$ times, which means the current stage finishes and the Cloud server will turn into the next stage. (Lines 10$\sim$14 on the Cloud server)
	\item Once the Cloud server finishes the training the current stage, it will set the global sensitivity goal to be reduced as $\Delta_{S}=\theta\Delta_{S}$ and computes the variance $\Delta_{b}$, then turns into the next stage. (Lines 15$\sim$16 on the Cloud server)
	\item After the model updated by sufficient iterations (i.e., $t \geq T$), the Cloud server finishes the training and broadcasts the \textit{Halt} command to all edge servers. (Line 18 on the Cloud server)
\end{itemize}


\begin{algorithm}
	\label{algorithm:MAPA}
	\caption{Multi Stage Adjustable Private Algorithm (MAPA)}
	\LinesNumbered
	\KwIn{number of edge servers $K$ and iterations $T$, mini-batch size $b$, reduction ratio $\theta$, privacy level $\varepsilon_{k}$, and probability $\delta$.}
	\KwOut{final model $x_{T}$.}
	\tcp{\textbf{($k$-th) Edge Server Side}}	
		Send $\varepsilon_{k}$ to the Cloud server;\\	
		\While{not \textit{Halt}}	
		{	
		Pull down $x_{t}$ and $\Delta S$ from the Cloud server;\\
		Compute the gradient $g_{t}(\mathcal{B}_{k})$ with $|\mathcal{B}_{k}|=b$;\\
		Clip the gradient as $g_{t}=g_{t}/\max(1,\frac{\|g_{t}\|_2}{b\Delta S/2})$;\\
		Draw a noise $\eta_{t}$ according to Eq. (\ref{eq:noise_density});\\
		Compute the noisy gradient  $\tilde{g}_{t}=g_{t}+\eta_{t}$;\\
		Send $\tilde{g}_{t}$ to the Cloud server;\\
	}
	\tcp{\textbf{The Cloud Server Side}}
	\setcounter{AlgoLine}{0}
	Receive all $\varepsilon_{k}$ from edge servers;\\
	Set $\varepsilon_{0}=\min\{{\varepsilon_{1},...,\varepsilon_{K}}\}$;\\
	$t=1$; \tcp*[h]{total iteration count}\\
	Initialize $x_t$ and $\Delta S$ (Theorem \ref{theorem:gloal_sensitivity_cheb});\\		
	\While{$t \leq T$}
	{
		Compute $P$ according to Eq.(\ref{eq:multi_stage_P});\\
		Set $\gamma^{-1}$=$2PL(\tau_{max}+1)$;\\ 
		Compute $T_{s}$ according to Eq.(\ref{eq:multi_stage_T0});\\
		$t_{s}$=1; \tcp*[h]{stage iteration count}\\
		\While{Receiving $\tilde{g}_{t-\tau(t)}$ and $t_{s}\leq T_{s}$}  
		{
		Update $x_{t}$=$x_{t}-\gamma\tilde{g}_{t-\tau(t)}$;\\
		Send $x_{t}$, $\Delta S$ to the updating edge server;\\
		$t_{s}$=$t_{s}$+1,
		$t$=$t$+$1$; \\ 
		}
		 Set $\Delta S$=$\theta\Delta S$;\\
			Compute $\Delta_{b}$=$\sigma^{2}/b+2\Delta S^{2}/\varepsilon_{0}^{2}$;\\
	 }
 	Send \textit{Halt} command to edge servers;\\
	\Return $x_{t}=x_{T}$.
\end{algorithm}

\begin{remark}
	MAPA is differentially private. Because we use $b\Delta S/2$ to clip the gradient, so the global sensitivity is $\Delta S$. Therefore, the $t$-th iteration in MAPA is $\varepsilon_{k(t)}$-DP. We don't consider the privacy of judgment $t_{s}\leq T_{s}$ here. Indeed, this can be guaranteed by the sparse vector technique \cite{dwork2014algorithmic}.
	
	We omit the discussion of the total privacy cost in this paper. Because the privacy budget is fixed in each iteration, the total budget is an accumulation of individual privacy costs in all iterations. By using the simple composition theorem, the total budget is $\sum_{t=1}^{T}\varepsilon_{k}(t)$, which increases linearly with the number of iterations. If we use the advanced composition theorem \cite{dwork2014algorithmic} or moment account for Gaussian mechanism \cite{abadi2016deep}, then it becomes a sub-linear function.
\end{remark}

\section{Evaluation}
\label{section:Experiments}
In this section, we conducted extensive experimental studies to validate the efficiency and effectiveness of MAPA.

\subsection{Experimental Methodology}
\subsubsection{Simulation and Testbed Experiment Implementations}
For a thorough evaluation, MAPA was implemented in both Matlab and Python for simulations and testbed experiments respectively. Codes are available in github.com~\cite{mapa2019}.
Specifically, we encapsulated MAPA's Python implementations in docker containers\footnote{https://www.docker.com/} for the edge servers and the Cloud server respectively.
To verify MAPA's performance in practical AFL scenarios with different scales, different numbers (from 5 to 20) of container-based edge servers were deployed on a local workstation (with a 10-core CPU and 128 GB memory).
The container-based Cloud server was deployed on a virtual machine (with a 24-core CPU and 256 GB memory) of the Alibaba Cloud\footnote{https://www.alibabacloud.com/product/ecs}. 
Communications between each edge server and the Cloud server were based on Eclipse Mosquitto\footnote{https://hub.docker.com/$\_$/eclipse-mosquitto} through the Internet. 


To set up the staleness in AFL, we adopted the cyclic delayed method \cite {agarwal2011distributed} for simulations, where the maximum delay of edge-cloud communications equals the total number of edge servers.
For testbed experiments, the actual staleness caused by heterogeneous delays between different edge servers and the Cloud server was adopted.

\subsubsection{Learning Models.} For generality, we applied MAPA to three machine learning models: Logistic Regression (LR) for a 2-way classifier; Support Vector Machine (SVM) and Convolutional Neural Network (CNN) for a 10-way classifier. It should be noted that although our theoretical results are derived based on differentiable convex functions (for LR), we will show that MAPA is also applicable to non-differentiable (for SVM) and non-convex (for CNN) loss functions. In particular, CNN consists of five layers (two convolutional layers, two pooling layers, and one full connection layer), noise is only added to the gradient of the first convolutional layer, which still guarantees differential privacy for whole CNN model due to the post-processing property of DP \cite{dwork2014algorithmic}.

\subsubsection{Datasets.} We adopted two commonly-used image datasets USPS and MNIST in our evaluations. USPS contains 9,298 gray-scale images with 256 features (7,291 images for training and 2,007 images for testing). MNIST contains 70,000 gray-scale images with 784 features (60,000 for training and 10,000 for testing).

\subsubsection{Comparison Algorithms and Parameter Settings.} For comprehensive evaluations, we compared MAPA (Algorithm \ref{algorithm:MAPA}) with the baseline algorithm AUDP to show the utility improvement. Besides, we also compared MAPA with the state-of-the-art asynchronous learning algorithm, the asynchronous Stochastic Gradient Descent Algorithms (ASGD) \cite{agarwal2011distributed,feyzmahdavian2016asynchronous} in terms of fast convergence speed. Also, the standard centralized Stochastic Gradient Descent algorithm without privacy protection, denoted as CSGD, is also compared for reference. 

The compared algorithms with their detailed parameters settings, such as learning rates and global sensitivities, are all listed in Table \ref{table: key parameters}. For all algorithms, the regularized parameter was set as $\lambda=0.0001$. Without a particular explanation, the number of edge servers $K$ was set as $5$, and the mini-batch size was set as $12$. 
Additionally, $\theta$ was set as $0.5$ in MAPA.

\begin{table*}[htbp]\centering
	\caption{Comparison Algorithms and Parameters}
	\label{table: key parameters}
	\begin{center}
		\begin{tabular}{ccccc}
			\hline
			Algorithm&Description&learning rate ($\gamma_{t}^{-1}$)&global sensitivity ($\Delta S$)\\
			\hline
			CSGD&Centralized stochastic gradient descent \cite{bubeck2015convex}&$\gamma_{t}^{-1}$=$L+\sqrt{t+1}\cdot\sigma/(R\sqrt{b})$&N/A\\
			MAPA& Multi stage adjustable private algorithm &\tabincell{l}{Stage $s+1$: $\gamma^{-1}=2PL(\tau_{max}+1)$,\\ where $P$=$\max\left\{\frac{8\Delta_{b}}{(\tau_{max}+1)b^{2}\Delta S_{s}^{2}\theta^{2}},1\right\}$}&\tabincell{l}{Initial value $\Delta S_{0}$: by  Eq.(\ref{eq:global_sensitivity_cheb})\\Stage $s+1$: $\Delta S_{s+1}=2\sqrt{\Delta_{s}}/b$}\\
			AUDP&Asynchronous update with differential privacy  & $\gamma_{t}^{-1}=L(\tau_{max}+1)+\sqrt{\Delta_{b}+1}\sqrt{t}$&Determined by actual model\\			
			ASGD& Asynchronous stochastic gradient descent \cite{feyzmahdavian2016asynchronous} &$\gamma_{t}^{-1}$=$L(\tau_{max}+1)^2+\frac{\sqrt{t+1}\cdot\sigma}{R\sqrt{b}}$&N/A\\			
			\hline
		\end{tabular}
	\end{center}
\end{table*}

\subsection{Simulation Results}
\label{subsection: ex-dp}
In this section, we conducted MATLAB simulation for our proposed MAPA to demonstrate its effectiveness of privacy preserving, validate its trade-off between the model utility and privacy, as well as the efficiency in model convergence.


\subsubsection{Demonstration of Privacy Protection}

This subsection demonstrates the privacy-preserving effects and adaptive clipping bounds effects in the training process of MAPA. 

To show the privacy-preserving effect, two models, LR and SVM\footnote{For simplicity, we omitted the demonstration results for CNN.}, were trained on MNIST and the privacy budget in each iteration of MAPA was set as 0.01, 0.1 and 1 respectively. The iteration number ranges from 2000 to 14,000. To measure the privacy-preserving effects, we adopted the inferring method in \cite{phong2018privacy} to recover the digital images from the gradients during the iterations. Fig. \ref{figure:DP_protect_gradient} illustrates the inferred digital images under different levels of differential privacy. As shown in both LR and SVM, when the privacy is higher (i.e., $\varepsilon=0.01$), the inferred images are totally blurred compared with the original image, which shows MAPA can be resilient to the inference attack; when the privacy is lower (i.e., $\varepsilon=1$), some inferred images can be approximately restored, which also shows the privacy protection degrades with the increase of privacy budget $\varepsilon$. Therefore, with proper choice of privacy budget, MAPA can effectively control the privacy protection for the AFL system.

To show the adaptive bound clipping effect, LR was trained on USPS for 100 edge servers and the privacy budget in each iteration of AUDP and MAPA was set as $0.1$. Fig. \ref{figure:multi_stage_adjustable} demonstrates how the gradient norm varies with the iteration number. In particular, Fig. \ref{figure:ASGD_gradient} shows the general gradient evolution of ASGD without DP, where the learning rate was set as  $\gamma_{t}^{-1}=L(\tau_{max}+1)^2+\frac{\sqrt{t+1}\cdot\sigma}{R\sqrt{b}}$. Fig. \ref{figure:AUDP_gradient} illustrates the clipped gradients for AUDP with three different clipping bounds, 15, 3 and 0.2. As we can see, either too high or too low clipping bound would cause utility loss. Instead, a good model utility can be achieved when the clipping bound is set appropriate. However, this is hard to estimate before training. Fig. \ref{figure:MAPA_gradient} draws the results for MAPA using different initial clipping bounds 200, 100 and 10, respectively. As shown, MAPA can adaptively adjust the global sensitivity dynamically in the training process and obtain nearly the same converged model utility as AUDP, regardless of the initial estimation of the global sensitivity. 




\begin{figure*}[!htbp]
	\centering
	\subfigure[LR on MNIST]{
		\includegraphics[width=0.45\linewidth]{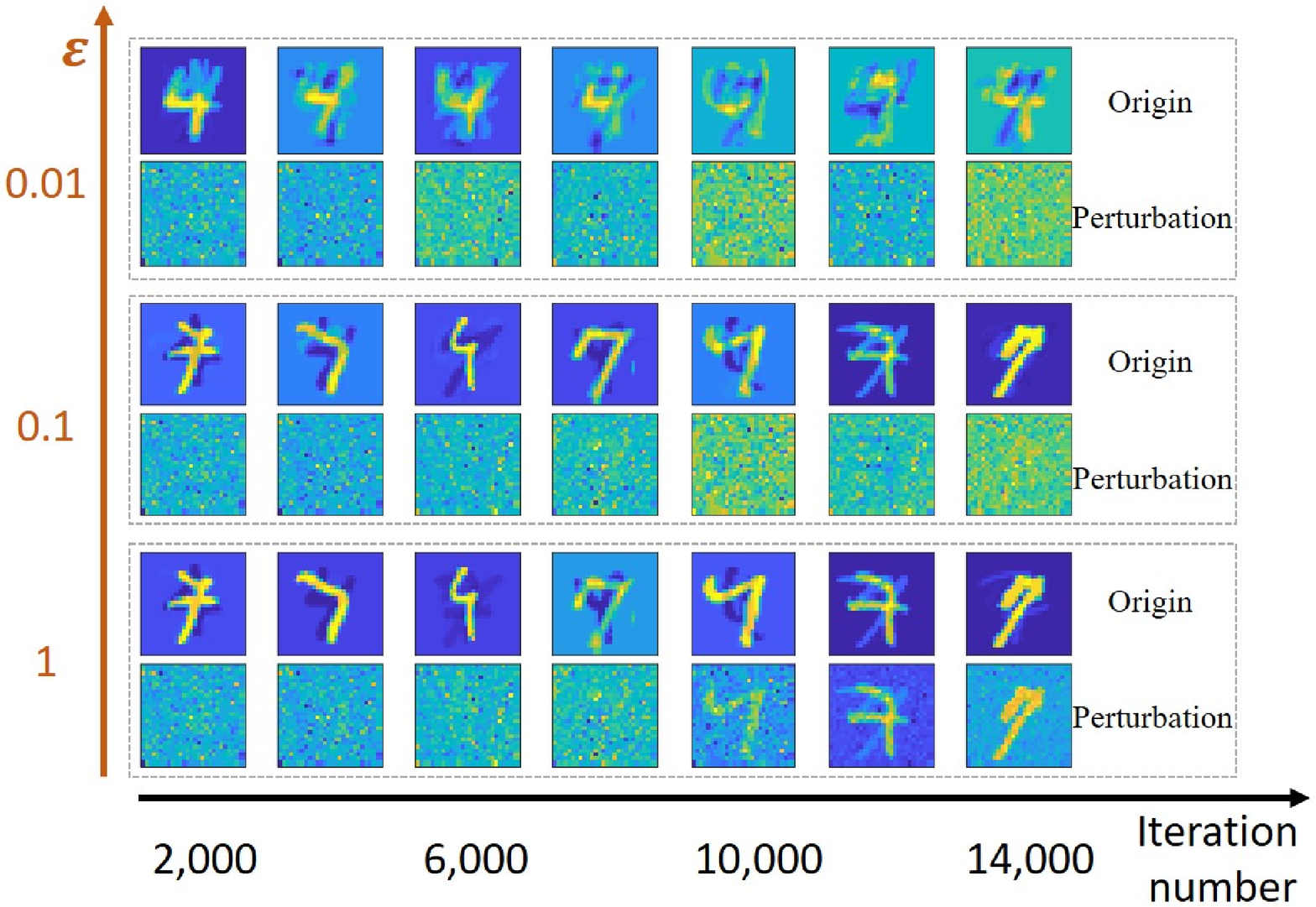}
		\label{figure:LR_gradient_information}
	}
	\subfigure[SVM on MNIST]{
		\includegraphics[width=0.45\linewidth]{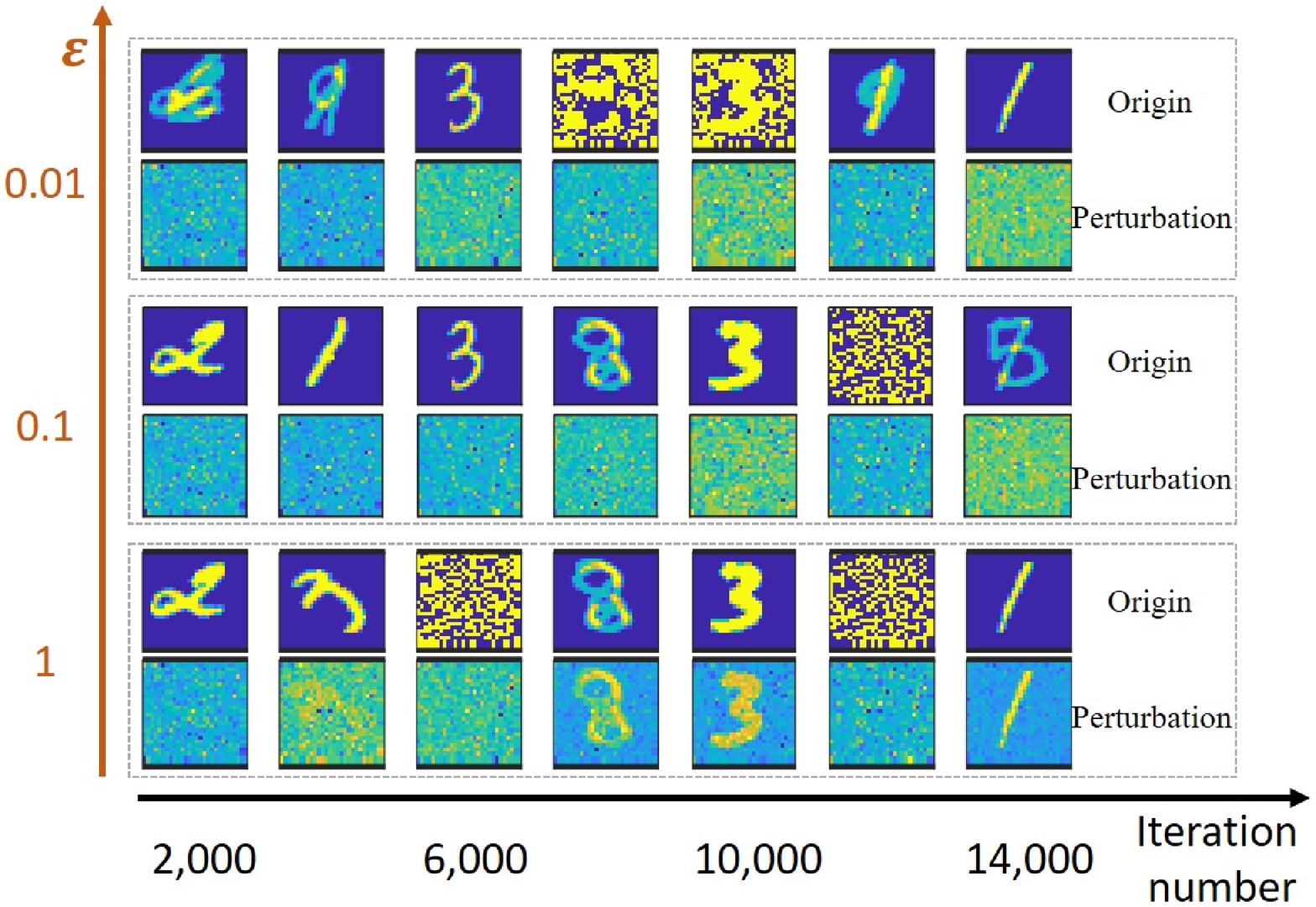}
		\label{figure:SVM_gradient_information}
	}
	\caption{Inference results under different privacy levels.}
	\label{figure:DP_protect_gradient}
\end{figure*}

\begin{figure*}[!htbp]
	\centering
	\subfigure[AUDP without DP]{
		\includegraphics[width=5.65cm]{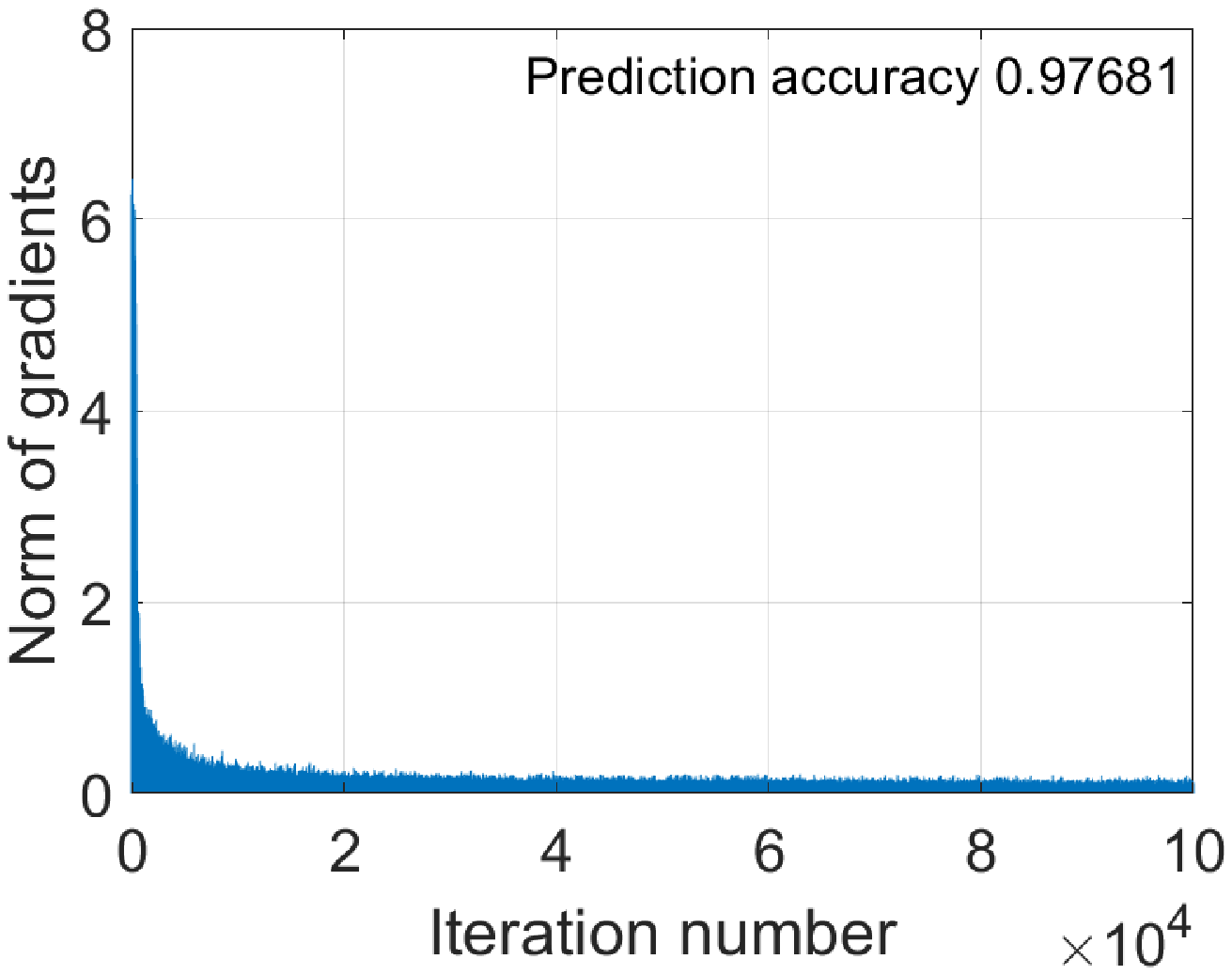}
		\label{figure:ASGD_gradient}
	}
	\subfigure[AUDP with different fixed clipping bounds]{
		\includegraphics[width=5.65cm]{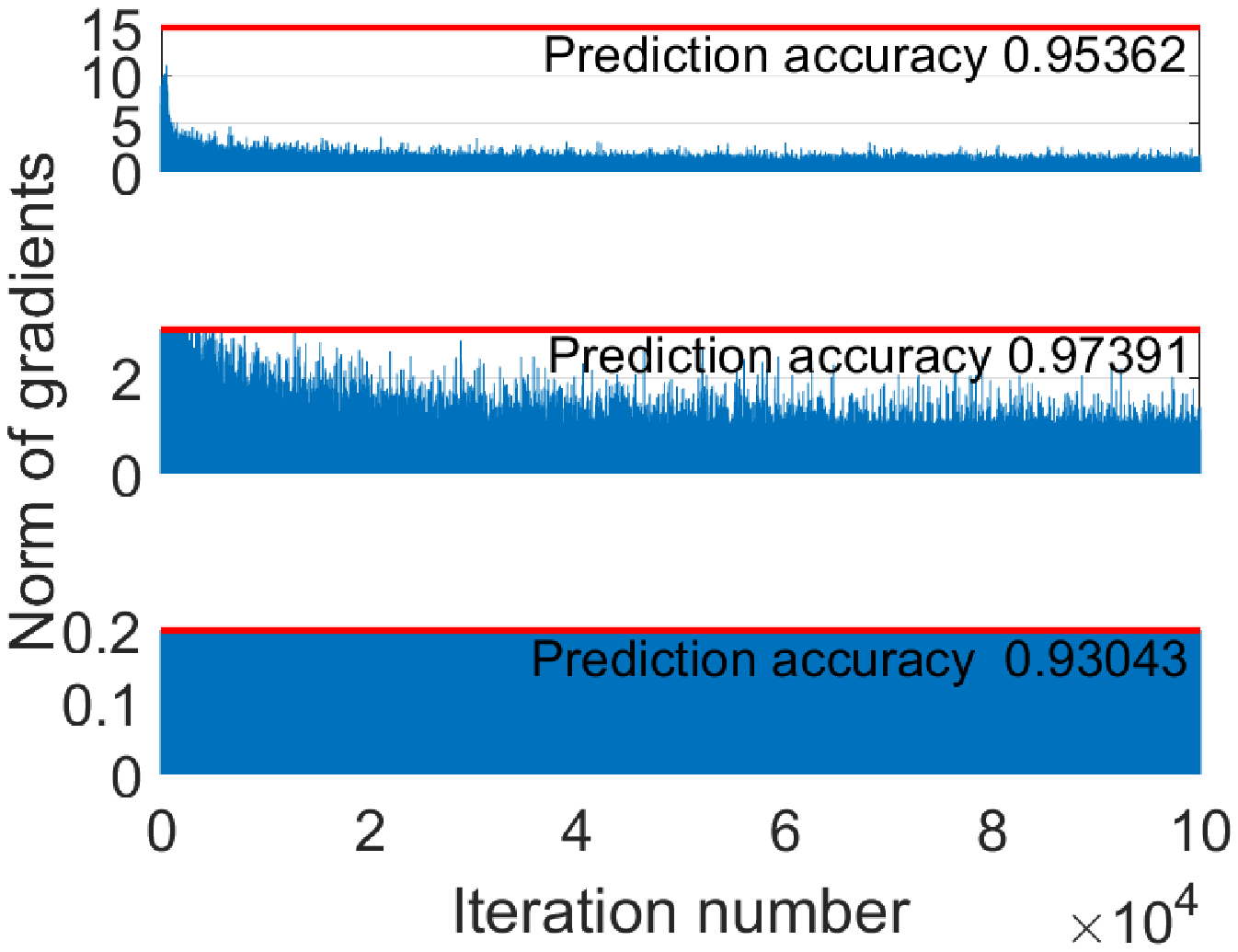}
		\label{figure:AUDP_gradient}
	}
	\subfigure[MAPA with different initial clipping bounds]{
		\includegraphics[width=5.65cm]{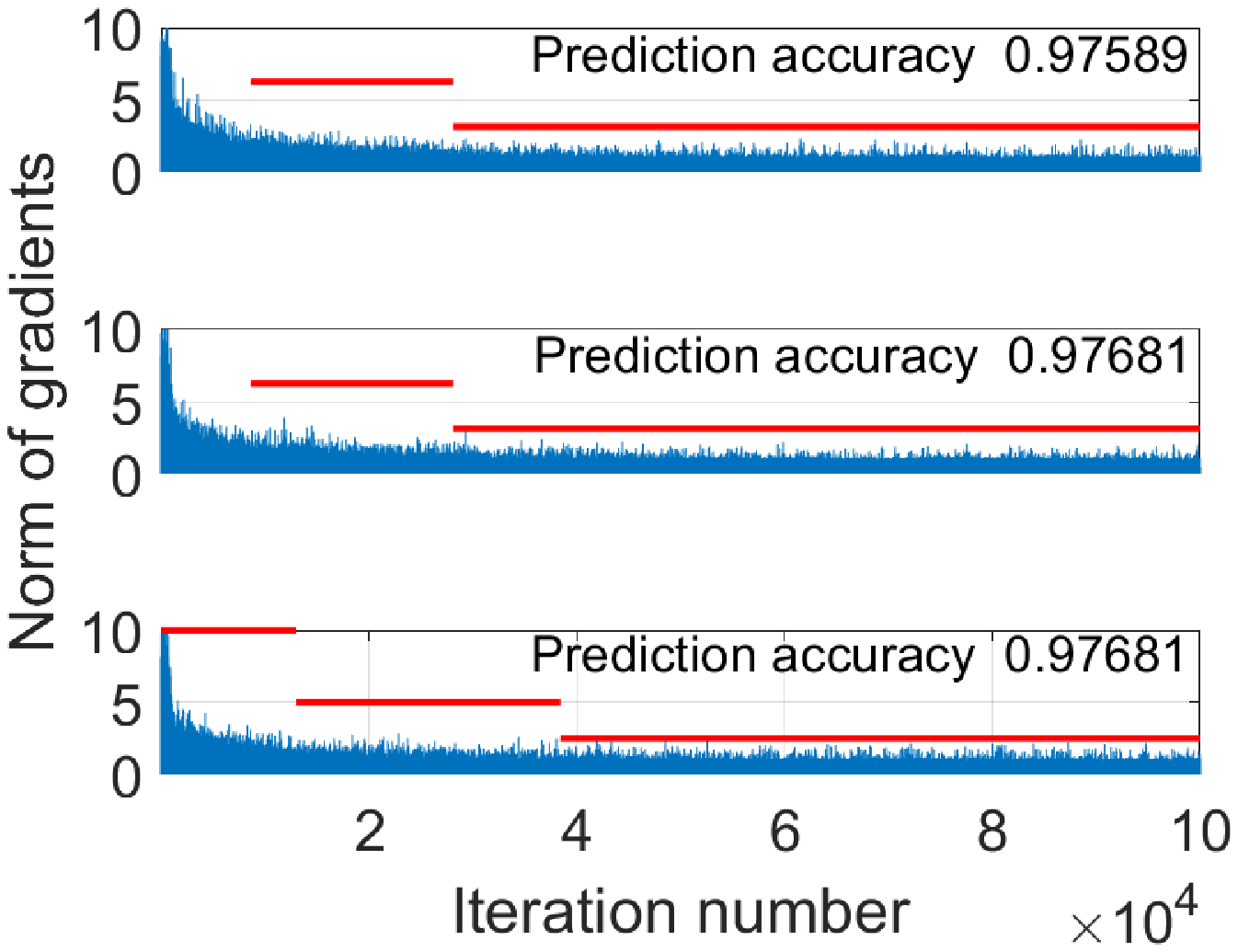}
		\label{figure:MAPA_gradient}
	}
	\caption{Inference results between AUDP and MAPA with different clipping bounds.}
	\label{figure:multi_stage_adjustable}
\end{figure*}

\subsubsection{Model Accuracy vs. Privacy Guarantee}
\label{section: simulations for TAPA}
In this subsection, we study the impacts of different privacy levels on the model utility. In particular, we simulated an edge-cloud FL system with five edge servers, where three models LR, SVM and CNN were trained for a given number of iterations (i.e., 15,000 for LR, 10,000 for SVM and 25,000 for CNN) on training datasets with the privacy budget in each iteration ranging from 0.1 to 0.5. Then the average prediction accuracy on testing datasets is collected.

Fig. \ref{figure:comparison_noise} compares the model accuracy of MAPA with the baseline algorithm AUDP under different levels of privacy. The results on both non-private algorithms CSGD and ASGD are also compared for reference. As we can see, firstly, both the prediction accuracy of privacy-preserving algorithms MAPA and AUDP increase with the differential privacy budget $\varepsilon$, which shows the genuine trade-off between the model accuracy and the privacy guarantee. 

Secondly, MAPA can effectively improve the prediction accuracy of AUDP in all sub-figures for  different $\varepsilon$ and the improvement is more significant for small privacy regimes. Especially, the maximal improvement can reach 20\% in Fig. \ref{figure:nl_mnist_cnn} and even 100\% in Fig. \ref{figure:nl_usps_cnn}. This shows that MAPA can achieve a better trade-off by effectively reducing the noise needed for privacy guarantee. 

Thirdly, MAPA can achieve a similar prediction accuracy as the non-private ASGD in all subplots with the increase of privacy budget. Particularly, for LR, the prediction accuracy of MAPA is even higher than ASGD. That is because the prediction accuracy of LR is mostly decided by the initiation phase and is very sensitive to the learning rate. Meanwhile, MAPA has a larger learning rate than ASGD at the beginning phase, leading to higher accuracy. In summary, MAPA can achieve much higher model utility with a sufficient different privacy guarantee. 


\begin{figure*}[!htbp]
	\centering
	\subfigure[LR on MNIST]{
		\includegraphics[width=5.7cm]{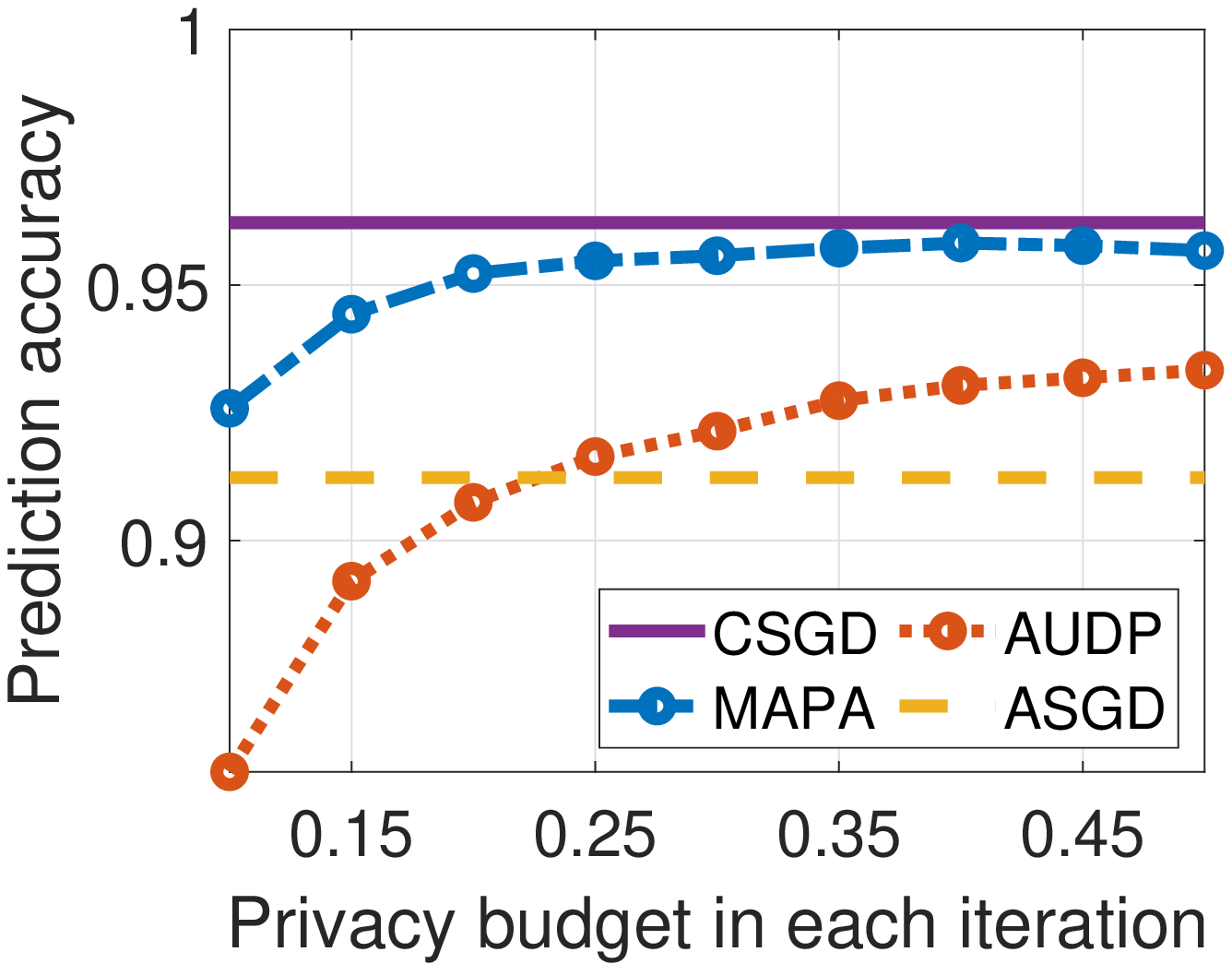}
		\label{figure:nl_mnist_lg}
	}
     \subfigure[SVM on MNIST]{
	    \includegraphics[width=5.7cm]{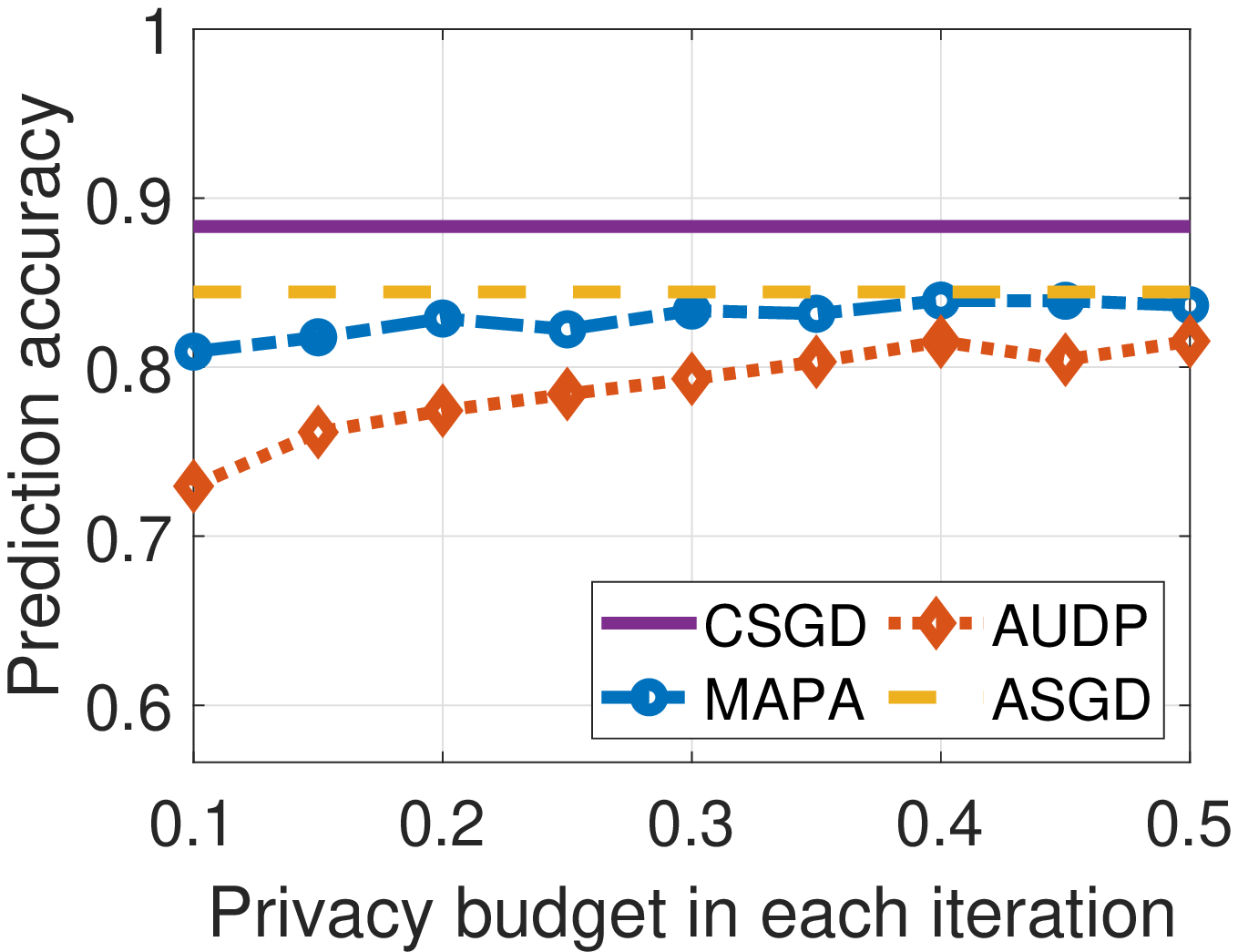}
	    \label{figure:nl_mnist_svm}
    }
     \subfigure[CNN on MNIST]{
     	\includegraphics[width=5.7cm]{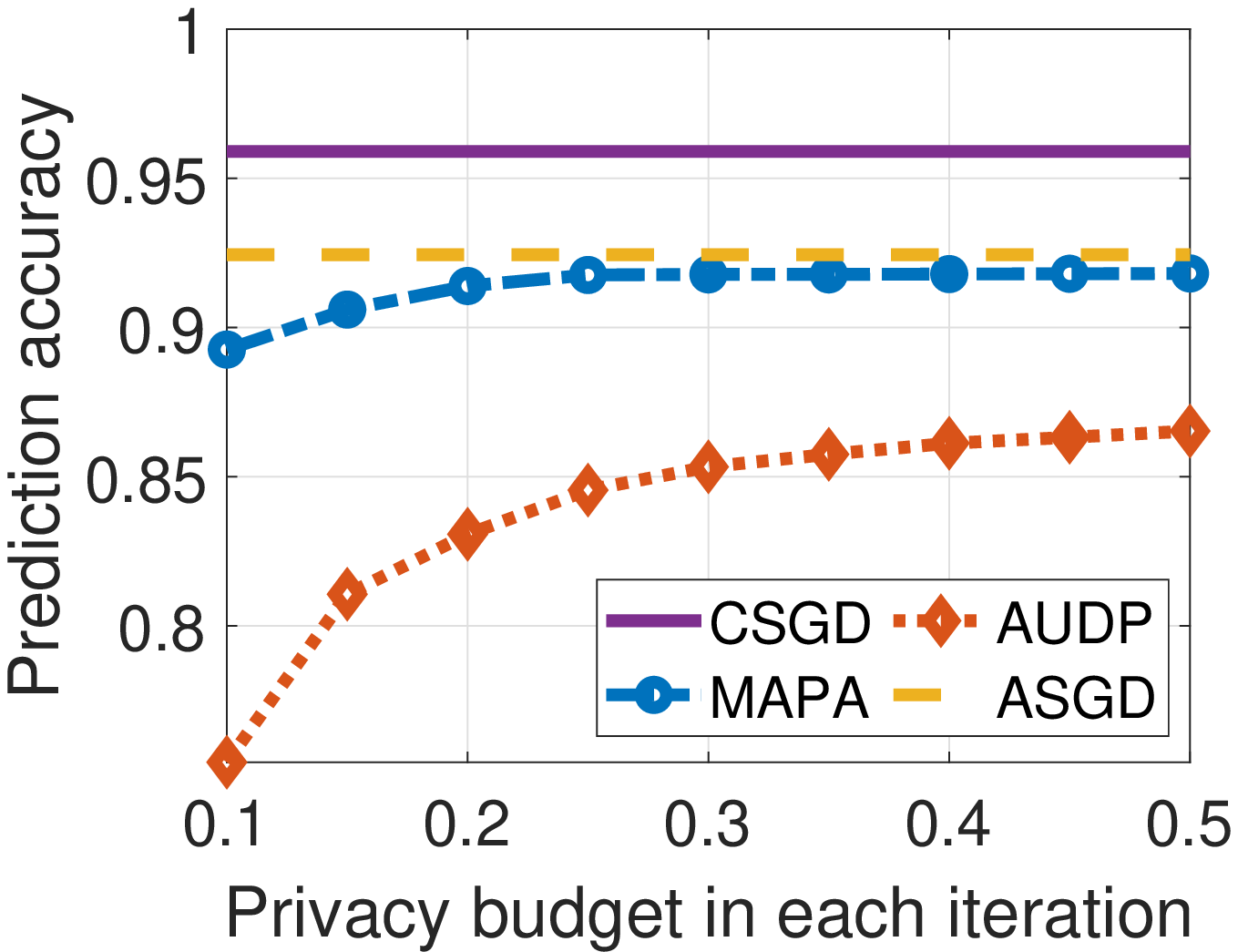}
    	\label{figure:nl_mnist_cnn}
    }
	\subfigure[LR on USPS]{
		\includegraphics[width=5.7cm]{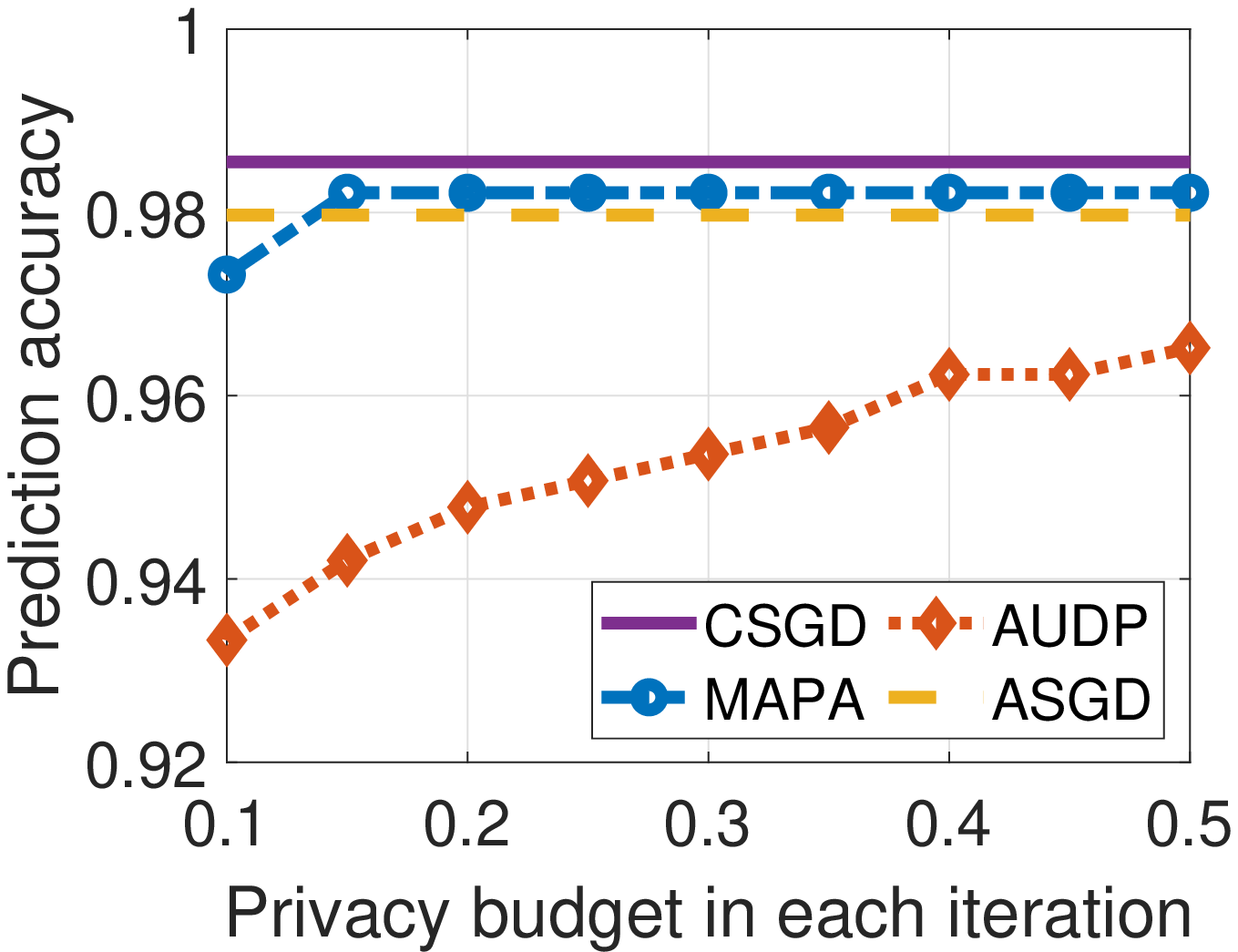}
		\label{figure:nl_usps_lg}
	}	
	\subfigure[SVM on USPS]{
		\includegraphics[width=5.7cm]{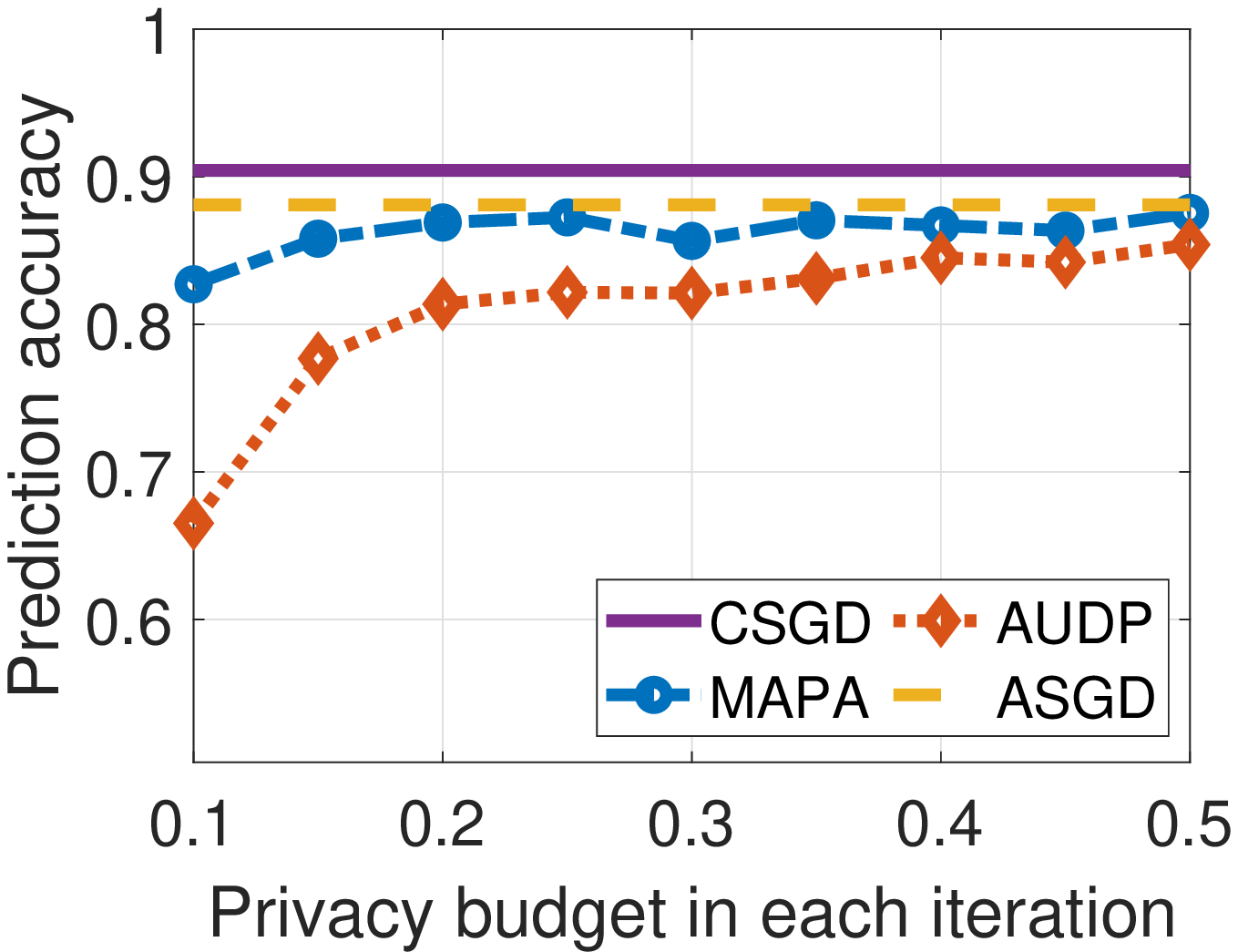}
		\label{figure:nl_usps_svm}
	}	
	\subfigure[CNN on USPS]{
		\includegraphics[width=5.7cm]{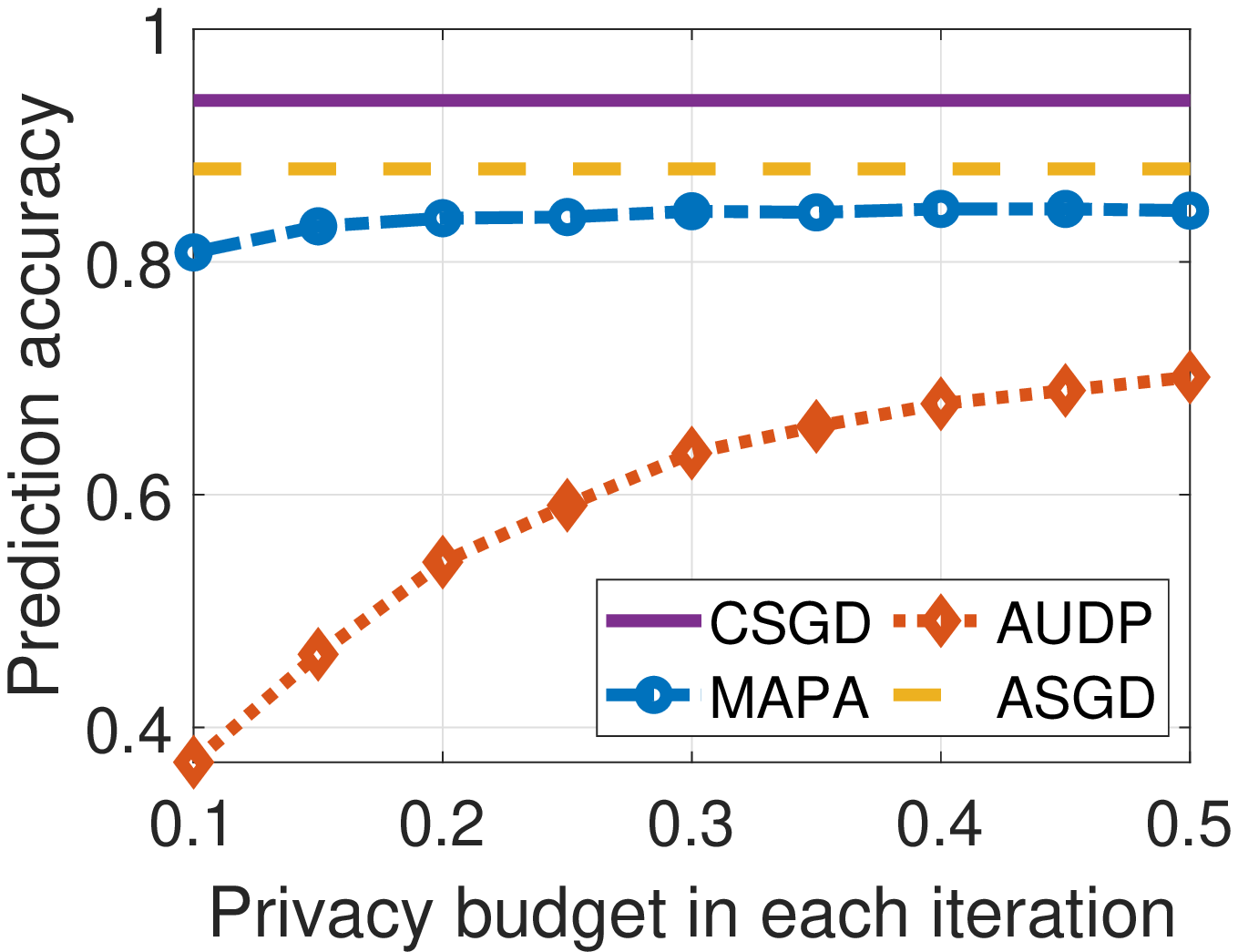}
		\label{figure:nl_usps_cnn}
	}
	\caption{Prediction accuracy vs. privacy budget $\varepsilon$.}
	\label{figure:comparison_noise}
\end{figure*}

\begin{figure*}[!htbp]
	\centering
	\subfigure[LR on MNIST (0.4)]{
		\includegraphics[width=5.7cm]{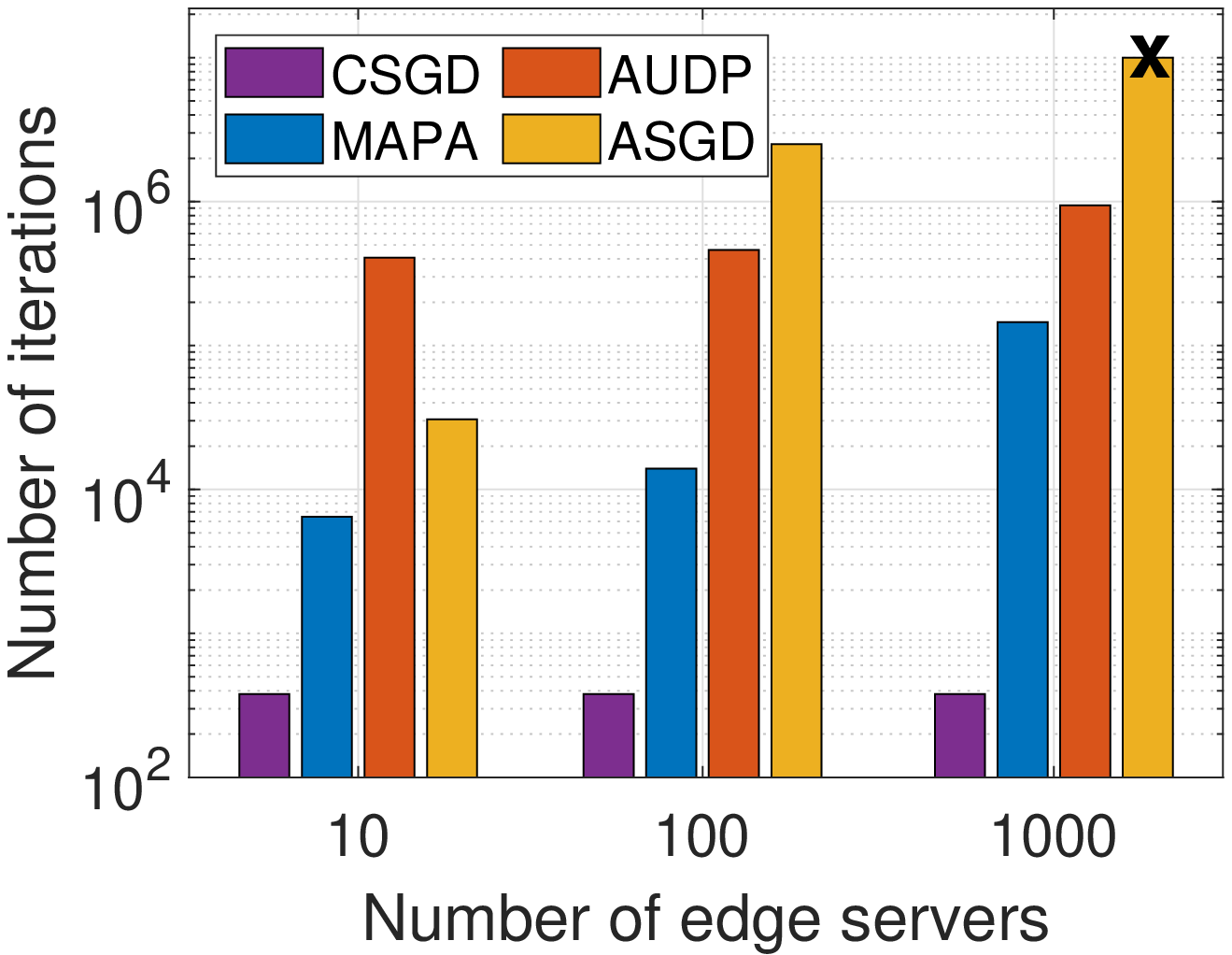}
		\label{figure:LG_Acc_edge_TAPA_TSA_SOA_MNIST}
	}
	\subfigure[SVM on MNIST (0.1)]{
		\includegraphics[width=5.7cm]{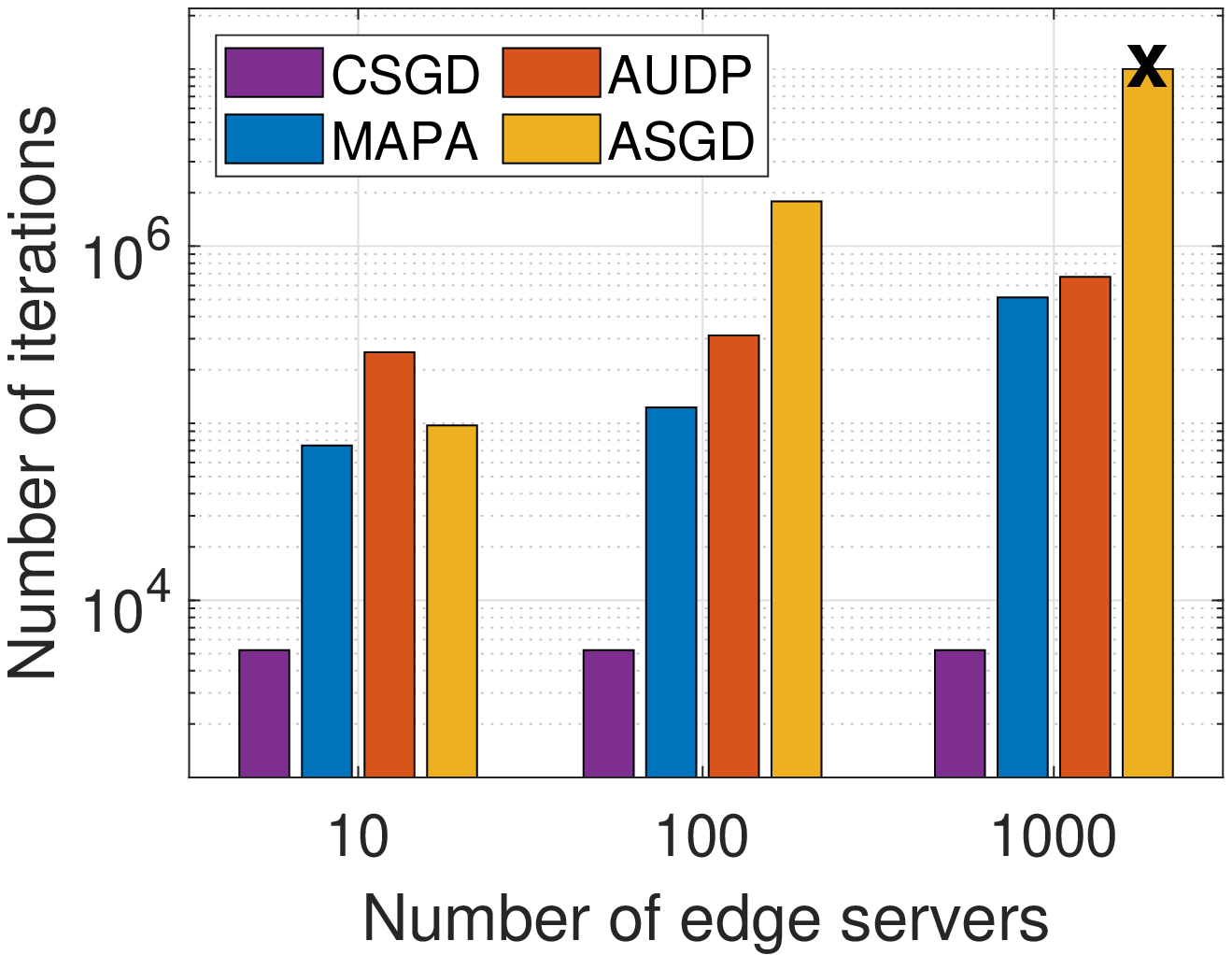}
		\label{figure:SVM_Acc_TAPA_TSA_SOA_MN}
	}
	\subfigure[CNN on MNIST (0.1)]{
		\includegraphics[width=5.7cm]{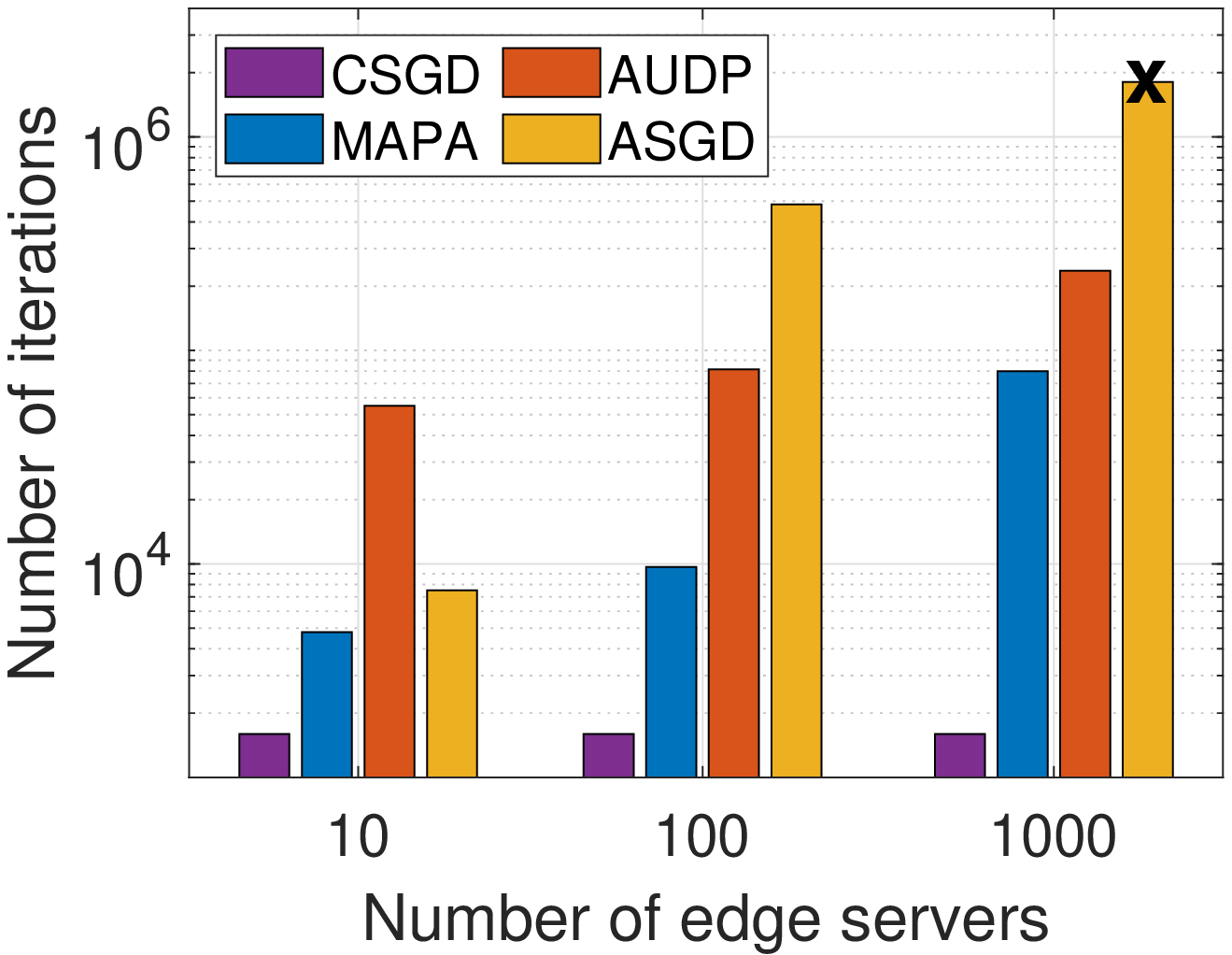}
		\label{figure:edgenumber-mnist-cnn_MNIST}
	}
	\subfigure[LR on USPS (0.2)]{
		\includegraphics[width=5.7cm]{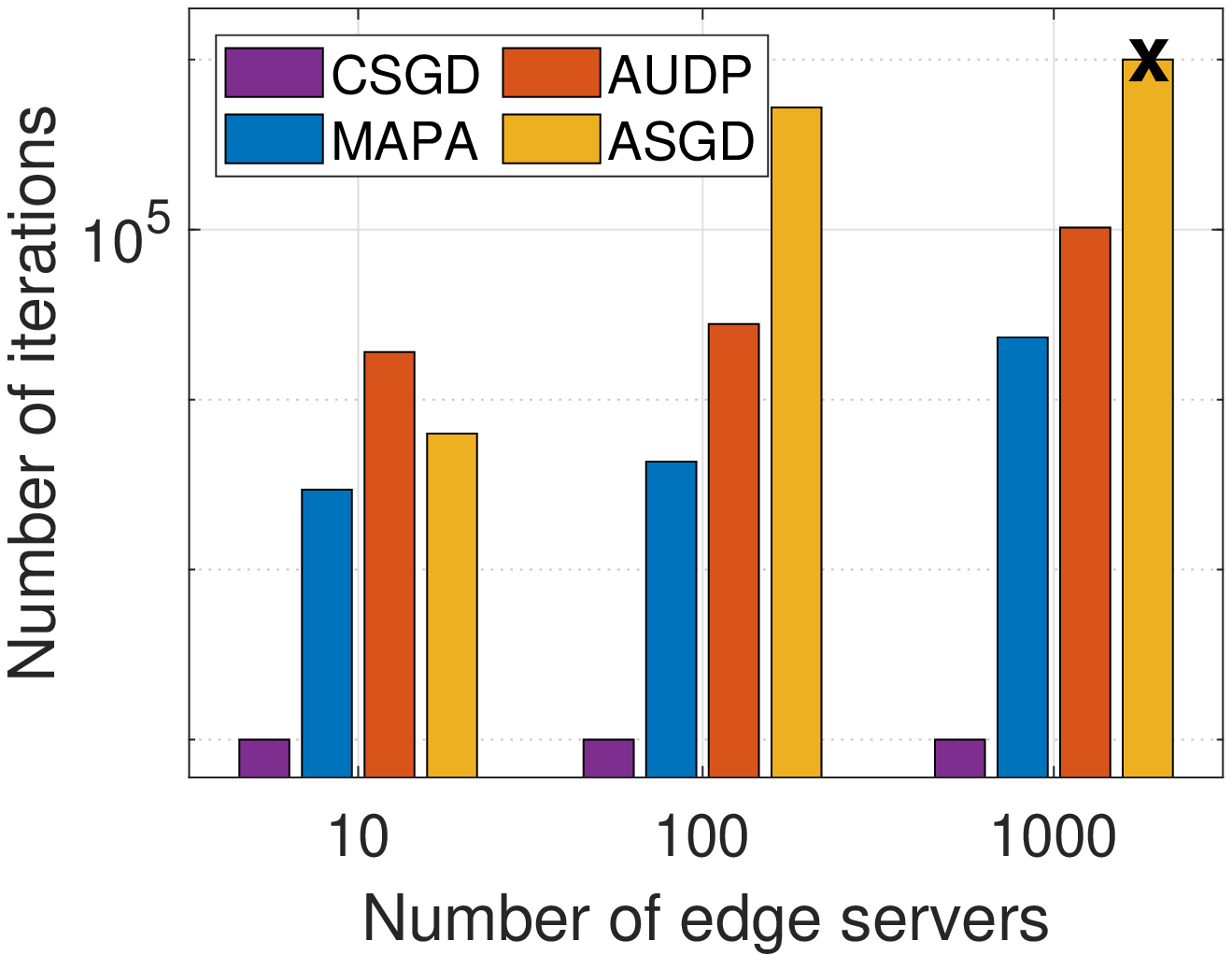}
		\label{figure:LG_Acc_edge_TAPA_TSA_SOA_USPS}
	}
	\subfigure[SVM on USPS (0.05)]{
		\includegraphics[width=5.7cm]{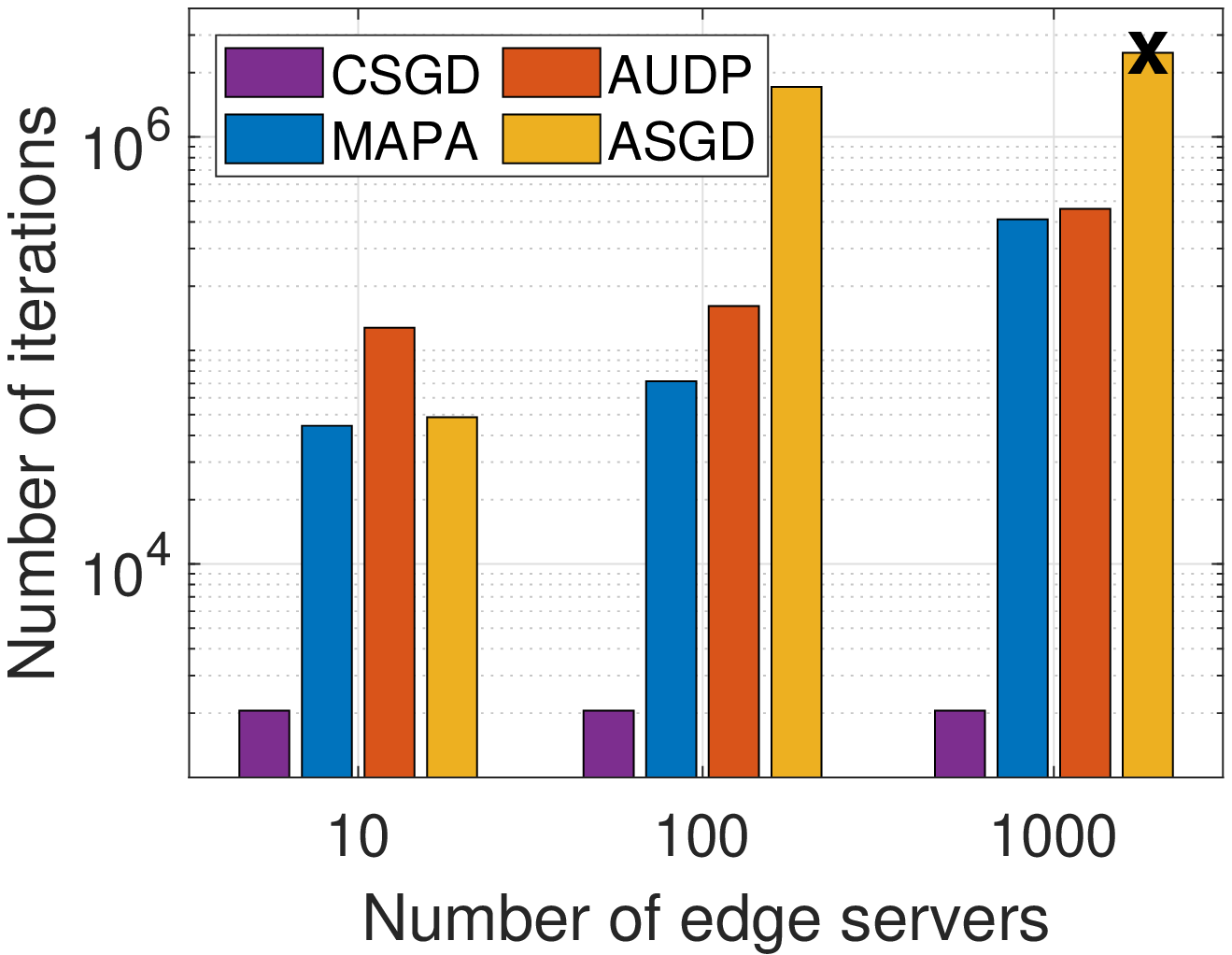}
		\label{figure:SVM_Acc_TAPA_TSA_SOA}
	}
	\subfigure[CNN on USPS (0.1)]{
		\includegraphics[width=5.7cm]{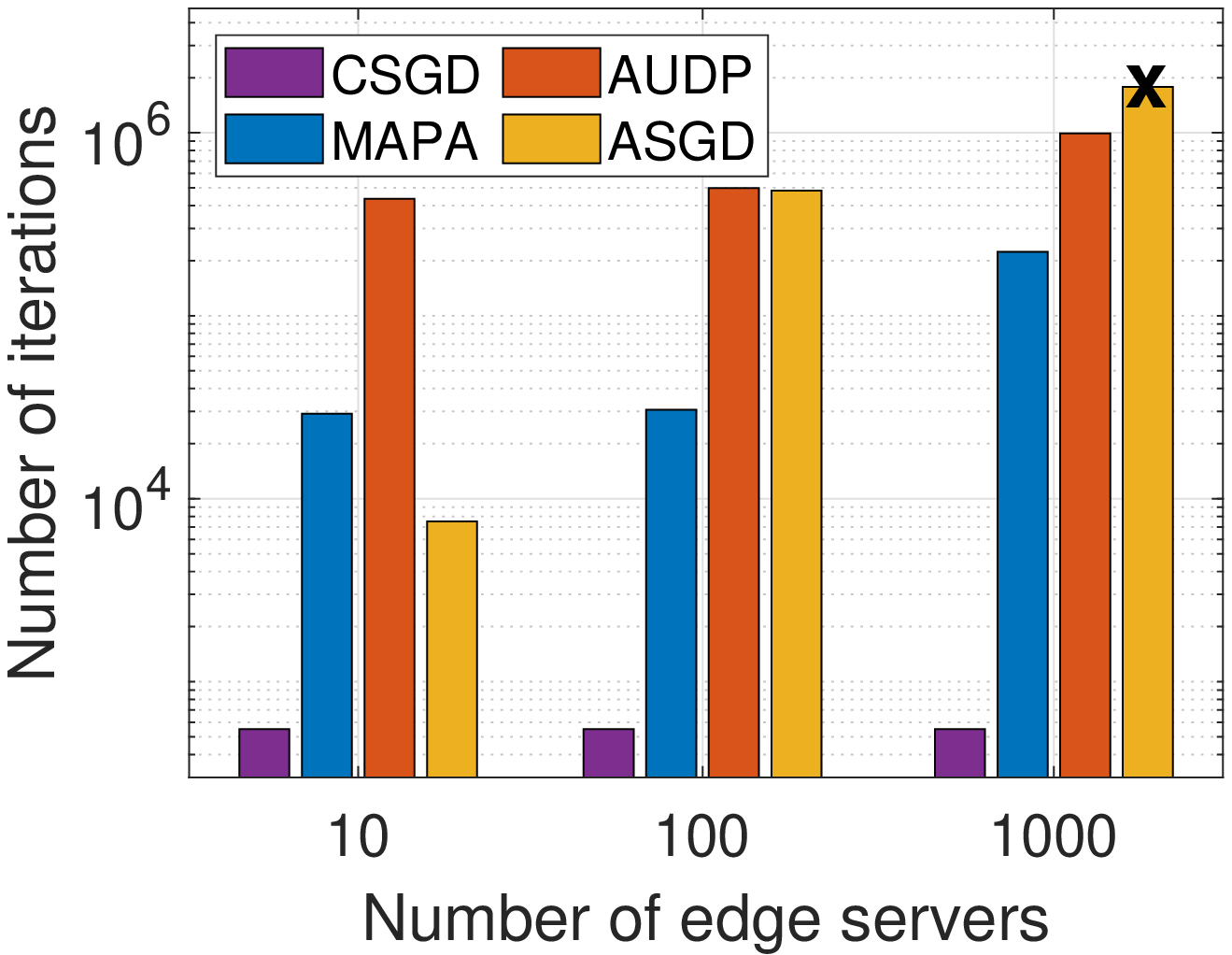}
		\label{figure:edgenumber-mnist-cnn}
	}
	\caption{Number of iterations for convergence vs. the number of edge servers.}
	\label{figure:comparison_different_edges}
\end{figure*}

\subsubsection{Model Convergence vs. Edge Staleness}
In this subsection, we study the impact of edge staleness on the model convergence efficiency. We simulated three learning models (LR, SVM, and CNN) on the edge-cloud collaborative FL with different numbers of edge servers, e.g., $K=10, 100, 1000$, respectively. In all simulations, the privacy budget in each iteration is $\varepsilon$=0.1 for MAPA and AUDP, then the average number of iterations for sufficient convergence (e.g., the average loss of 5 successive iterations is less than a given threshold) of all algorithms were reported. 

Fig. \ref{figure:comparison_different_edges} shows the iteration number of MAPA in comparison with both the private algorithm AUDP and non-private algorithms ASGD and CSGD under the different number of edge servers, which also represents different levels of edge staleness. 
As we can see, firstly, the number of iterations for all asynchronous algorithms, MAPA, AUDP and ASGD, increases with the number of edge servers $K$. This is because that, as $K$ increases, the gradients used in SGD are generally staler and contain very limited information, which therefore requires more iterations for convergence. The algorithm CSGD is performed on the central Cloud without collaborations with the edges and requires much fewer iterations. 

Secondly, MAPA achieves a faster convergence speed than AUDP. When $K$=10 and 100, MAPA can save 1-2 amplitudes of the number of iterations. For example, when $K$=10 in Fig. \ref{figure:LG_Acc_edge_TAPA_TSA_SOA_MNIST}, 2 amplitude saving is achieved. The reason is that the adjustable noise scale and learning rate together can ensure the model converges at the rate $O(1/T)$ (Theorem \ref{theorem:errorbound_gradient}) in each stage. 

Thirdly, MAPA achieves a faster convergence speed than ASGD and saved about 2 amplitude when $K$=100 and 1000. The reason is that a linear decaying learning rate with respect to $K$ (i.e., the $\tau_{max}$) is used in MAPA, but in ASGD, a second power polynomial decaying learning rate is designed to alleviate the effects of the staleness. However, as $K$ increases, the quickly decaying learning not only alleviates the staleness but also the useful information too much, leading to a long training process. In summary, MAPA can effectively tackle the edge staleness problem and have a better convergence efficiency for AFL.



\subsection{Testbed Experiment Results}
\label{subsection:ex-docker}
In this section, we verify the practical performance of MAPA based on real-world testbed experiments, as a complement to the simulations. Furthermore, the impacts of learning parameters on the practical performance of MAPA were validated. For simplicity, only the results of CNN model on the MNIST dataset are reported.

\subsubsection{Model Utility}
We implemented MAPA to train a CNN model in the testbed AFL system with the different number of edge servers $K$ as $5, 10, 15$, and $20$, respectively. The average prediction accuracy of trained models under different iterations on the edge servers are reported and drawn in Fig.~\ref{figure:docker_mnist_cnn_edge}.   

As shown, the prediction accuracy of MAPA is higher than AUDP in all cases. Also, with the increase of edge number, MAPA can even effectively outperform the non-private ASGD. These observations are consistent with the simulation results and validate the utility improvement of MAPA in practical systems.
Secondly, both MAPA and AUDP can obtain almost the same prediction accuracy as CSGD for CNN model training. That shows, adding proper noise will not significantly impact the model utility of CNN. As pointed out in \cite{dwork2014algorithmic}, appropriate random noises play the role of the regularization in machine learning and can enhance the robustness of the trained model. 

\subsubsection{Impacts of Parameters}
In this subsection, we demonstrated the impact of learning parameter on the model utility of MAPA in real-world testbed AFL system. When considering the impact of an individual parameter, others were fixed as default value, i.e., $\varepsilon=0.1, b=12, \sigma=30, L=10, \delta=10^{-3}, \theta=1/2$. 

Fig. \ref{figure:docker_mnist_cnn_para} shows the prediction accuracy of the trained model with MAPA concerning different parameters. We can have the following observations. 
Firstly, Figs. \ref{figure:Docker_CNN_MNIST_sigma} and \ref{figure:Docker_CNN_MNIST_delta} show that MAPA is robust to both $\sigma$ and $\delta$. That is, the estimation of the sample variance and the setting of probability loss are not crucial for convergence. 
Secondly, batch size and smooth constant have a little impact on prediction accuracy. For example, in Figs. \ref{figure:Docker_CNN_MNIST_batch} and \ref{figure:Docker_CNN_MNIST_L}, using a larger mini-batch size $b$ and smaller smooth constant $L$ can achieve a faster speed at the beginning, but will finally trend to the same accuracy at the given iterations. 
Thirdly, MAPA is sensitive to not only the privacy level but also the reduction ratio. In Fig. \ref{figure:Docker_CNN_MNIST_theta}, it is observed that a larger reduction ratio will lead to  lower model accuracy. The reason is that the learning rate will be adjusted too small for sufficiently achieving the larger reduction ratio (according to Theorem \ref{theorem:errorbound_gradient}), leading to much more iterations.


\begin{figure}[!htbp]
	\centering
	\subfigure[$K$ = 5]{
		\includegraphics[width=0.47\linewidth]{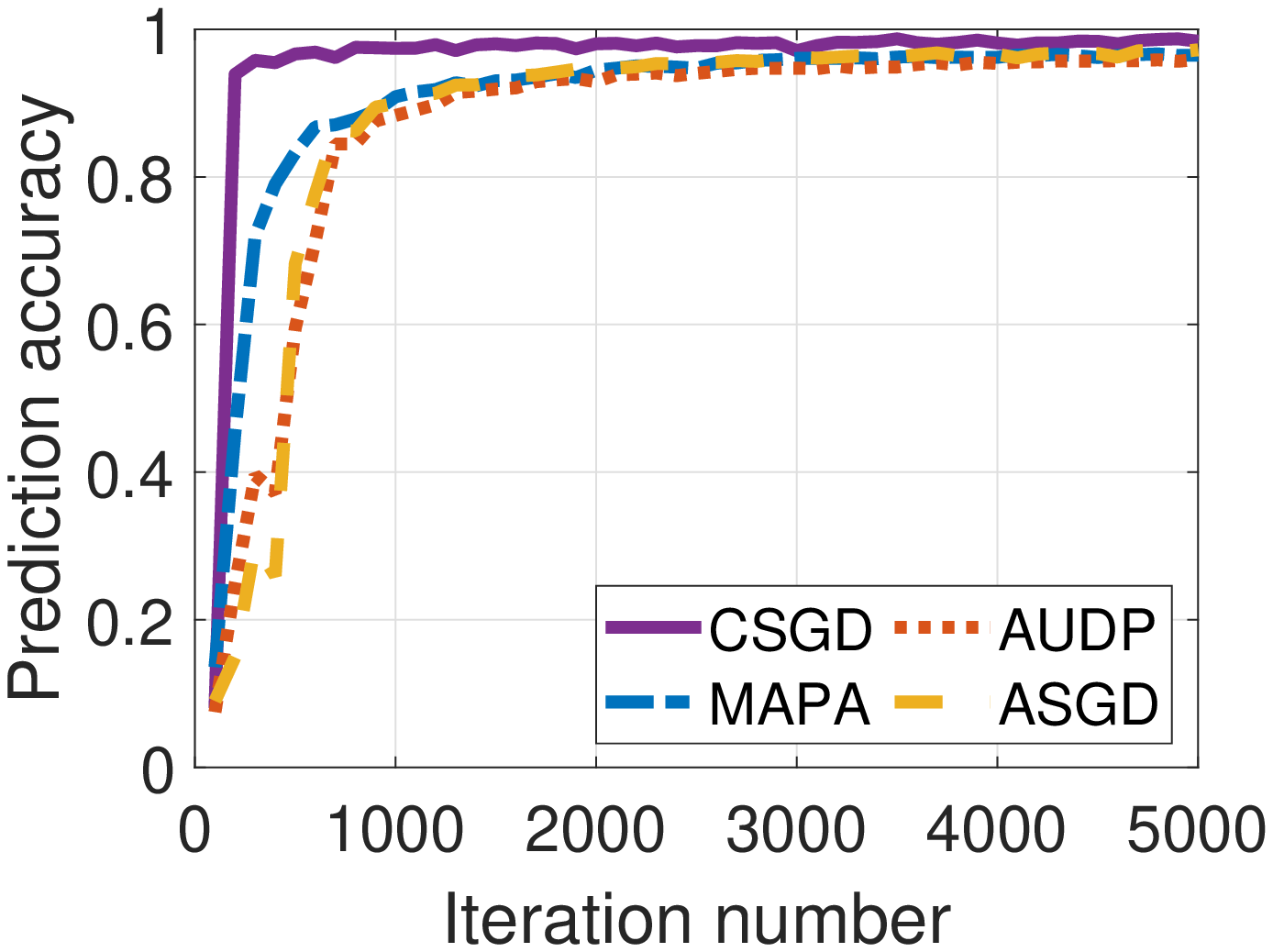}
		\label{figure:Docker_CNN_MNIST_edge05}
	}
	\subfigure[$K$ = 10]{
		\includegraphics[width=0.47\linewidth]{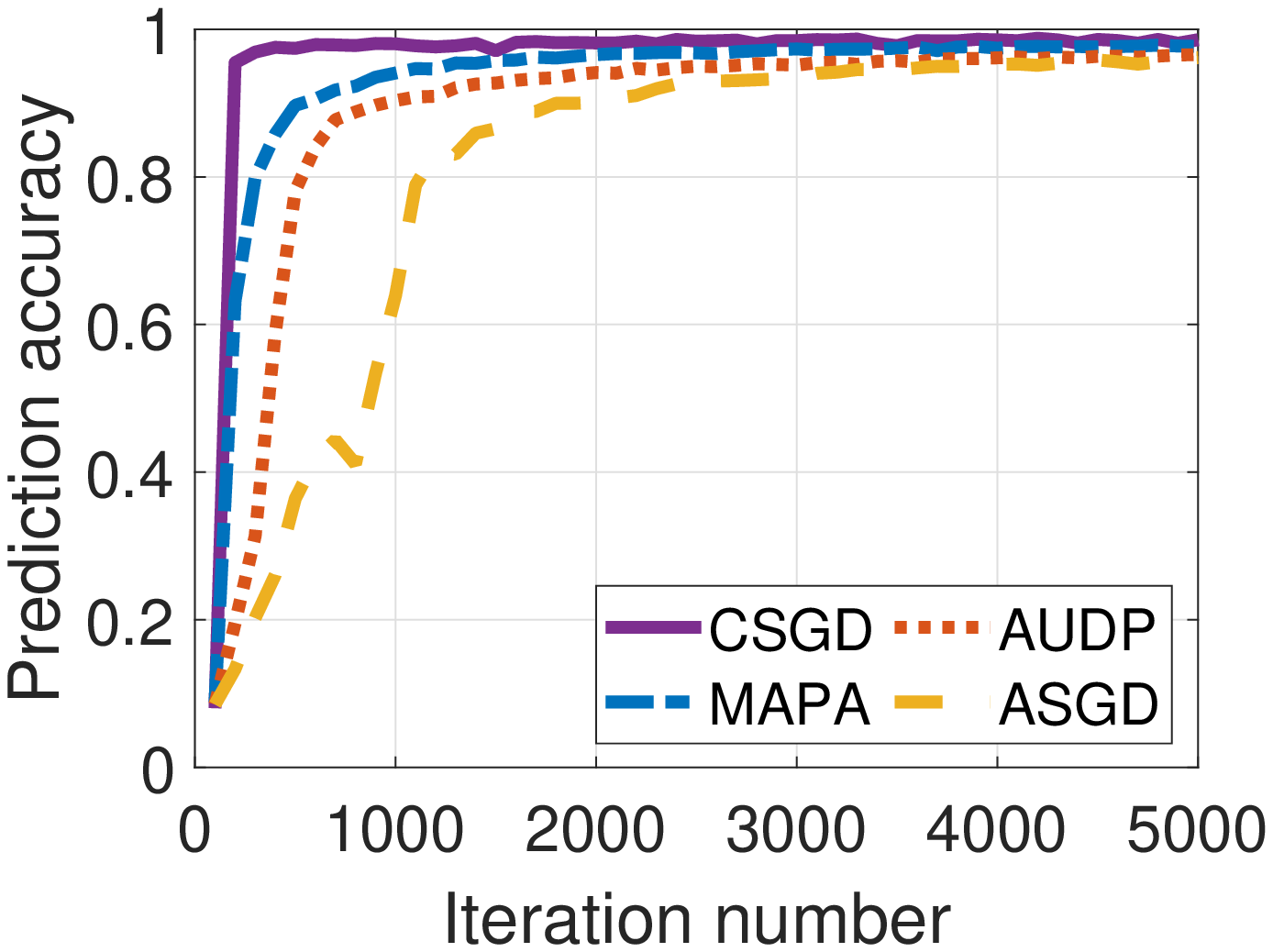}
		\label{figure:Docker_CNN_MNIST_edge0}
	}\\
	\subfigure[$K$ = 15]{
		\includegraphics[width=0.47\linewidth]{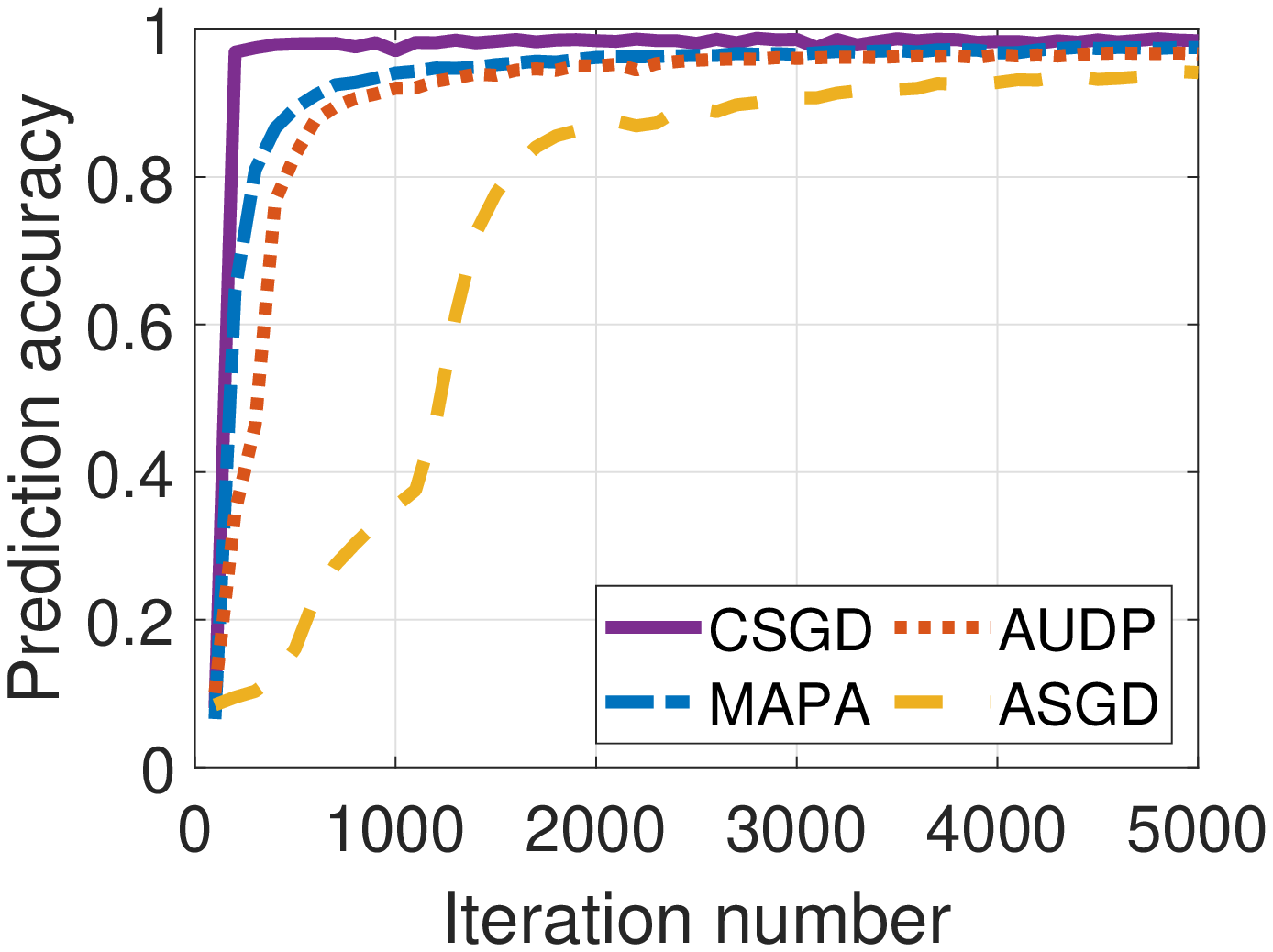}
		\label{figure:Docker_CNN_MNIST_edge15}
	}
	\subfigure[$K$ = 20]{
		\includegraphics[width=0.47\linewidth]{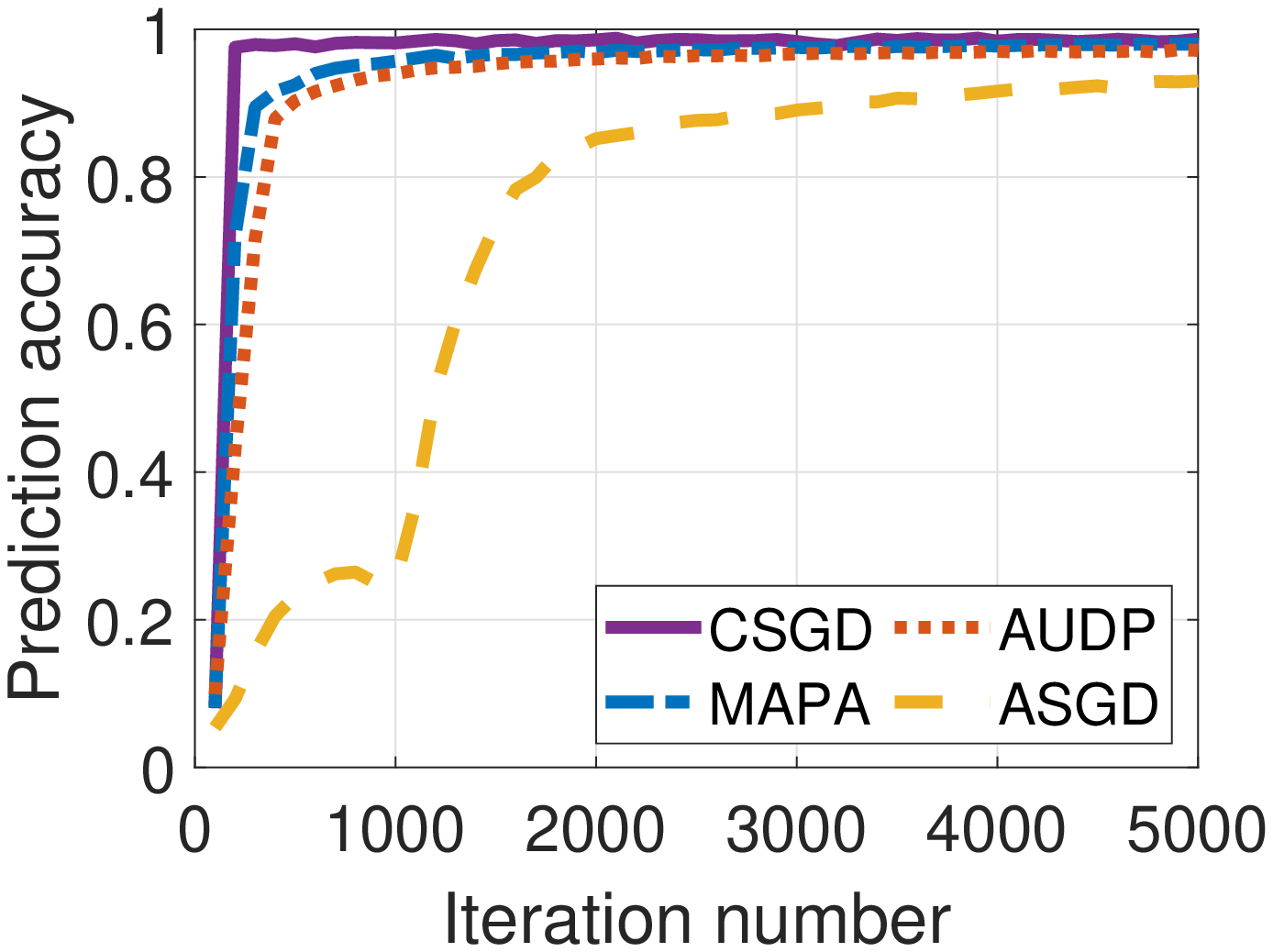}
		\label{figure:Docker_CNN_MNIST_edge20}
	}
	\caption{Prediction accuracy under different number of edge servers.}
	\label{figure:docker_mnist_cnn_edge}
\end{figure}

\begin{figure}[!htbp]
	\centering
	\subfigure[Privacy level]{
		\includegraphics[width=0.47\linewidth]{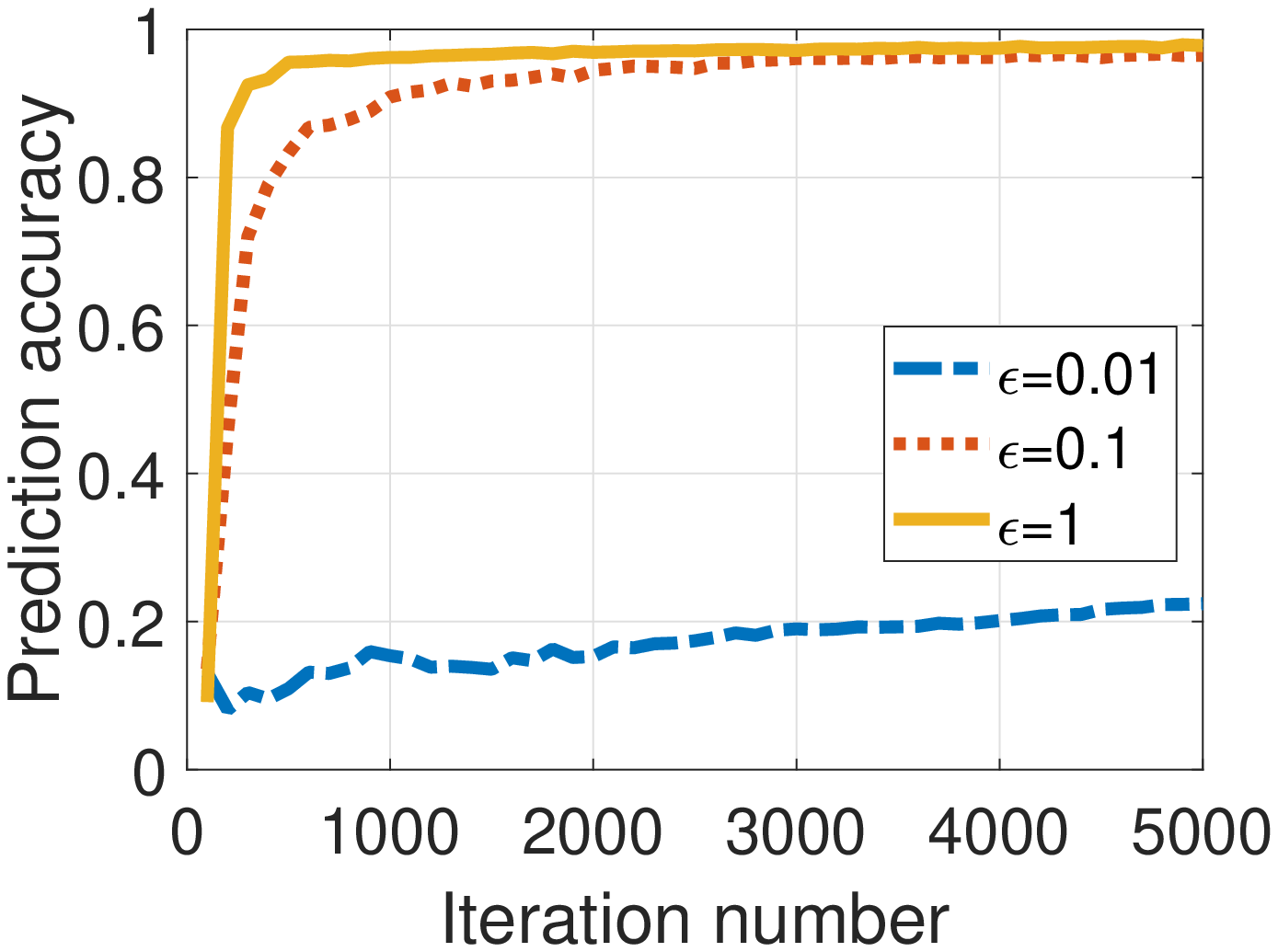}
		\label{figure:Docker_CNN_MNIST_noise}
	}
	\subfigure[Batch size]{
		\includegraphics[width=0.47\linewidth]{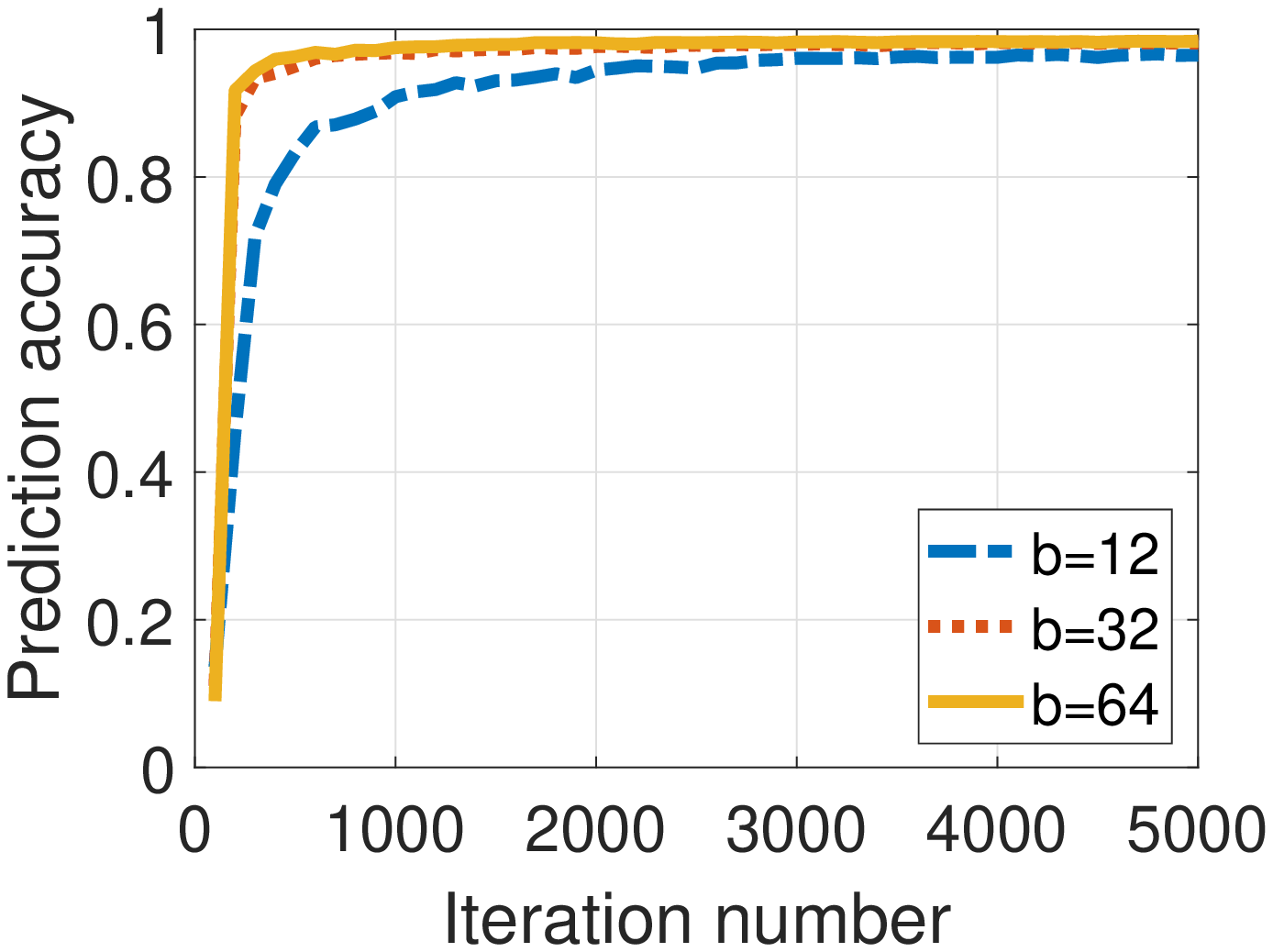}
		\label{figure:Docker_CNN_MNIST_batch}
	}\\
	\subfigure[Sample variance]{
		\includegraphics[width=0.47\linewidth]{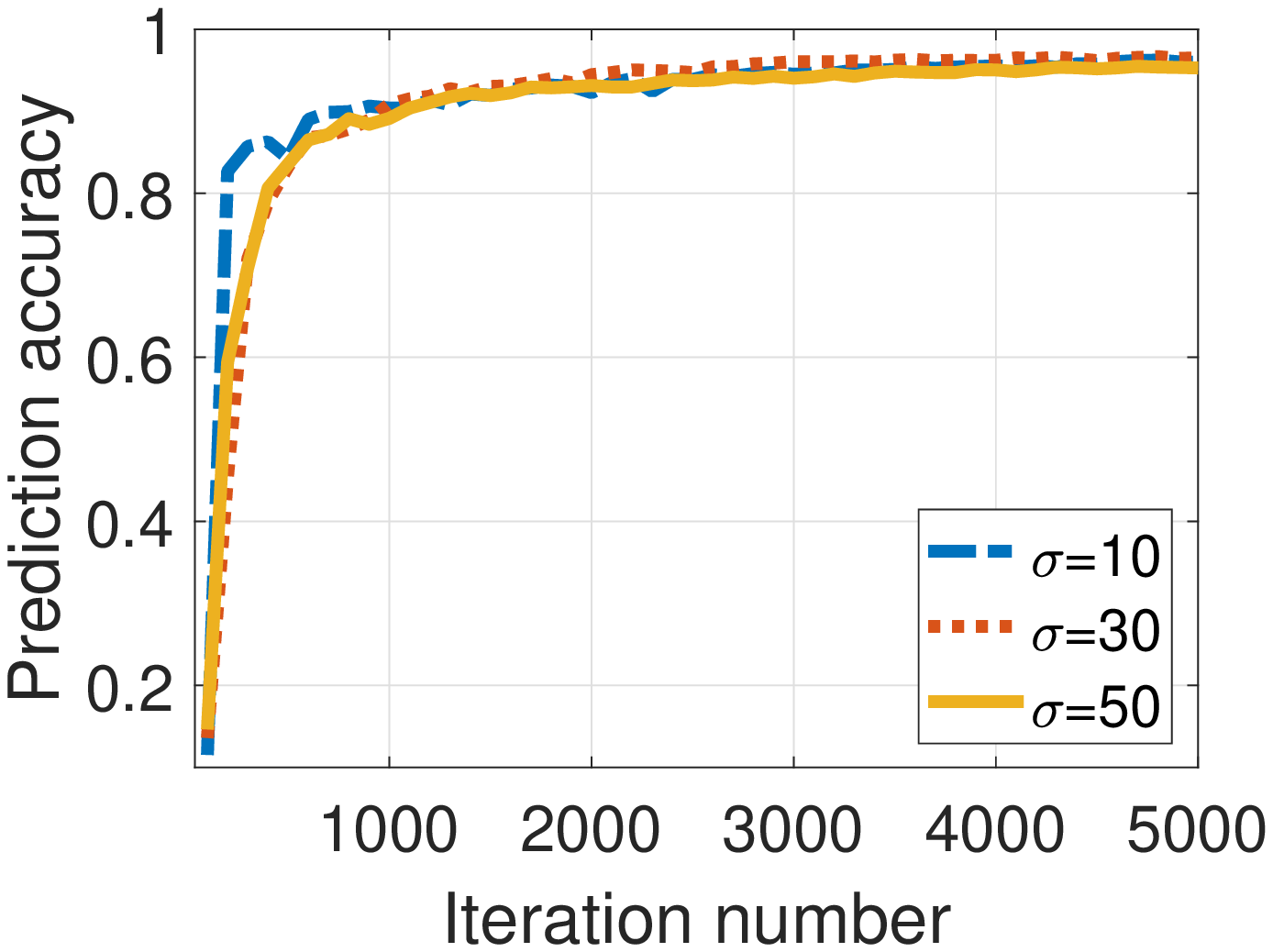}
		\label{figure:Docker_CNN_MNIST_sigma}
	}
	\subfigure[Lipschitz smooth constant]{
		\includegraphics[width=0.47\linewidth]{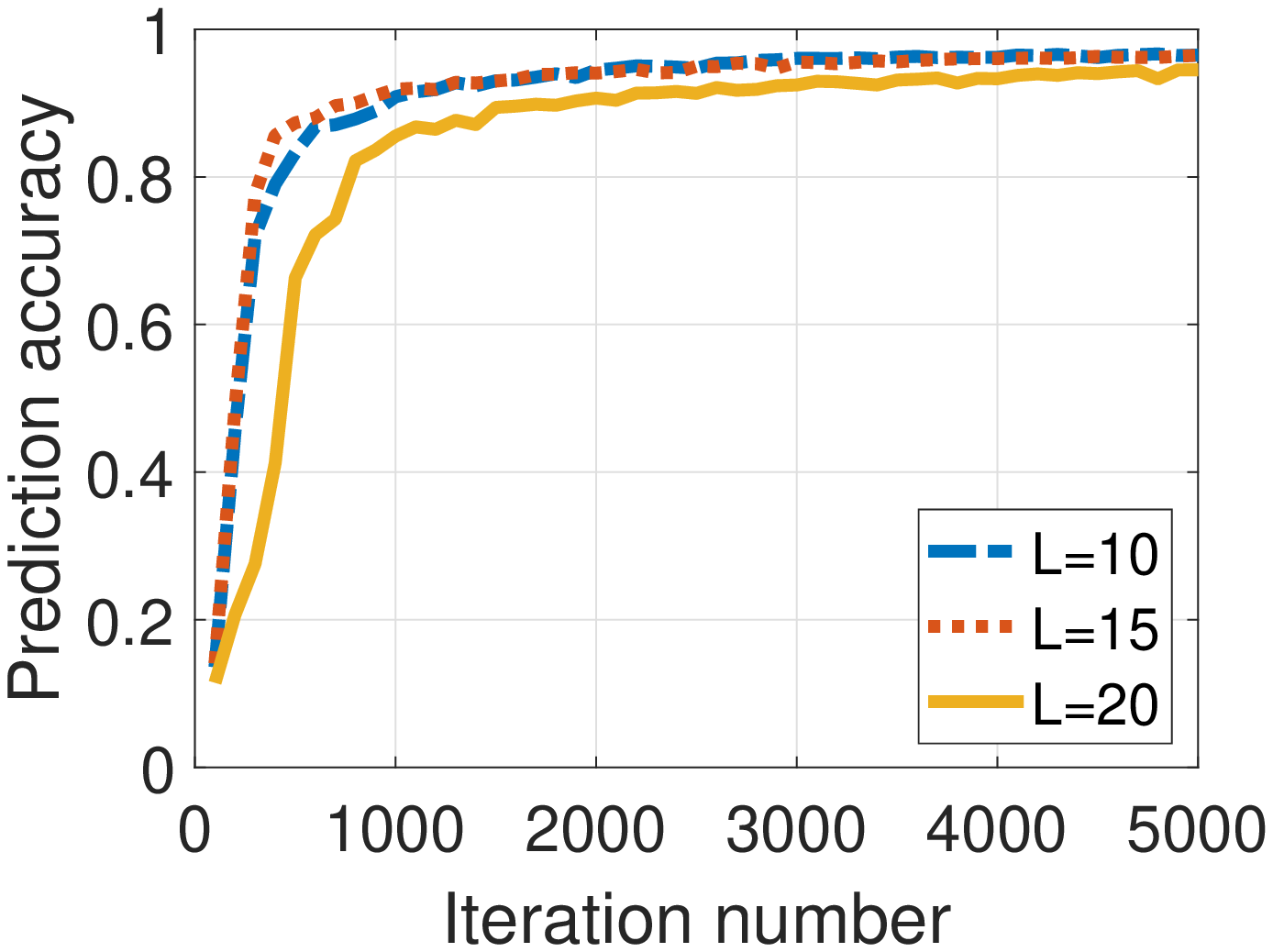}
		\label{figure:Docker_CNN_MNIST_L}
	}
	\subfigure[Probability loss]{
		\includegraphics[width=0.47\linewidth]{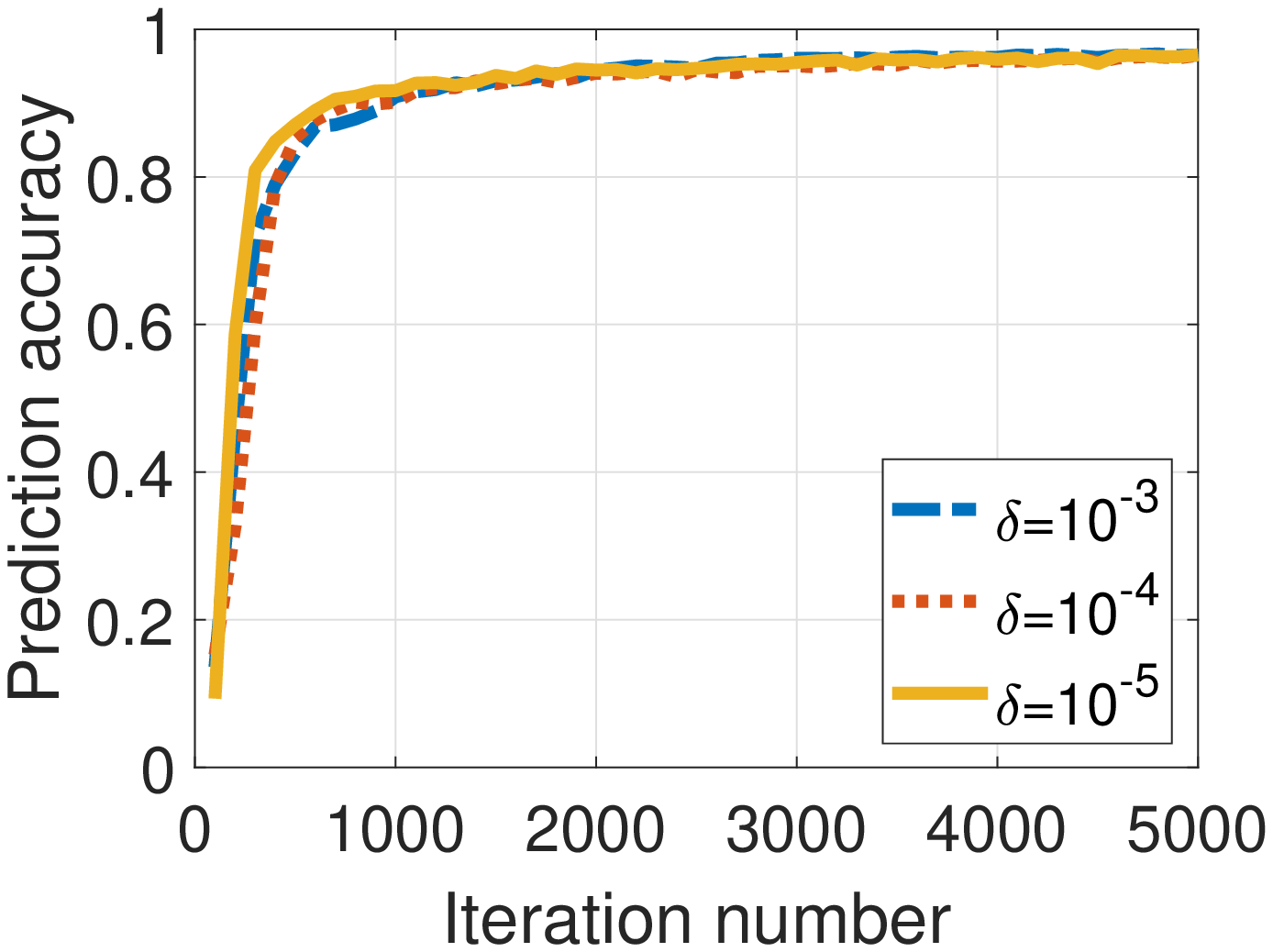}
		\label{figure:Docker_CNN_MNIST_delta}
	}
	\subfigure[Reduction ratio]{
		\includegraphics[width=0.47\linewidth]{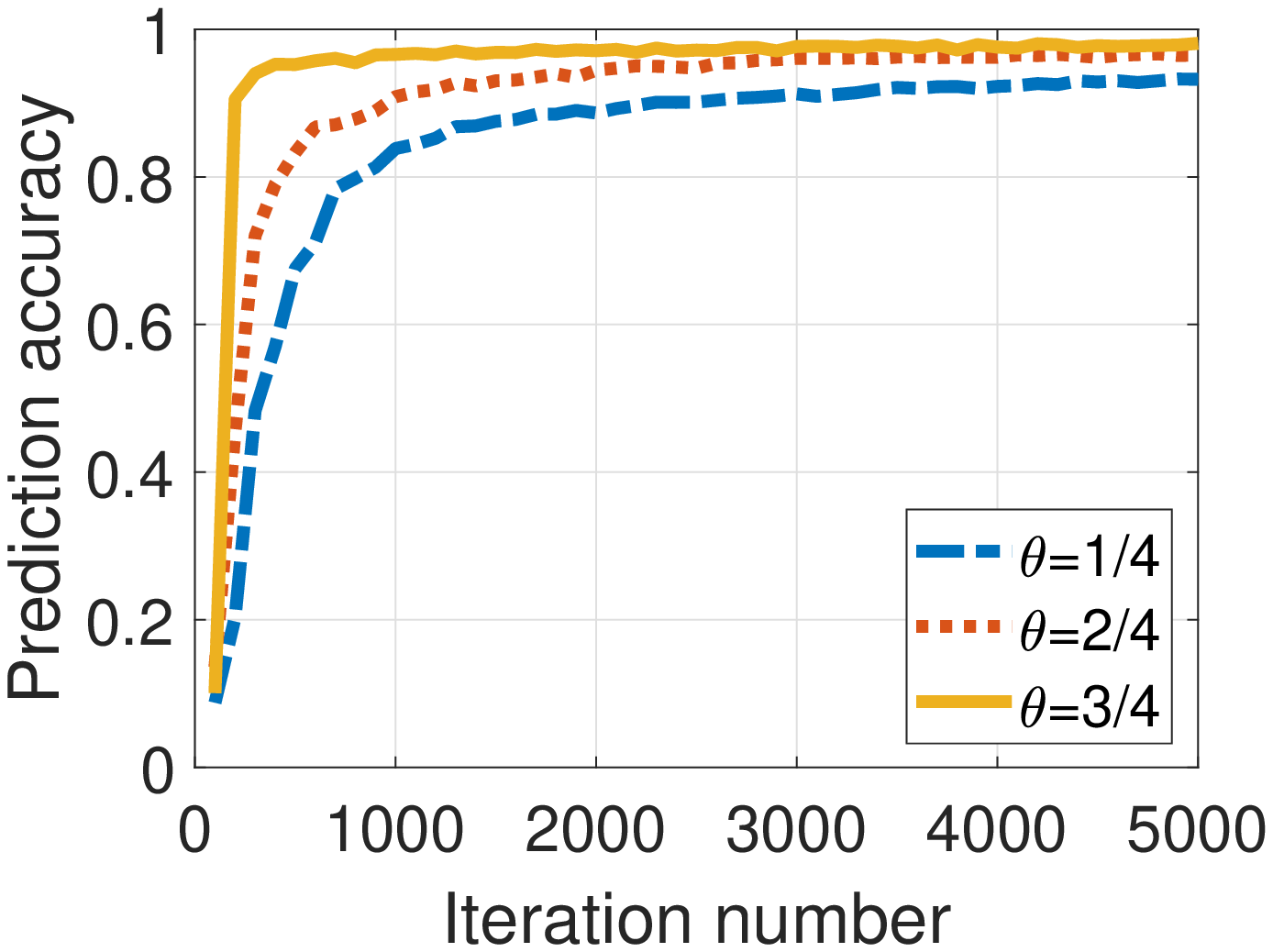}
		\label{figure:Docker_CNN_MNIST_theta}
	}
	\caption{Prediction accuracy with respect to different parameters.}
	\label{figure:docker_mnist_cnn_para}
\end{figure}
\vspace{-0.1cm}
\section{Conclusion}
\label{sec:conclusion}
This paper presents the first study on Asynchronous edge-cloud collaboration based Federated Learning (AFL) with differential privacy. Based on a baseline algorithm, we first theoretically analyzed the impact of differential privacy on the convergence of AFL. To enhance the learning utility, we then propose a Multi-Stage Adjustable Private Algorithm (MAPA) for AFL, which can adaptively clip the gradient sensitivity to reduce the privacy-preserving noise, thus achieving high model accuracy without complicated parameter tuning. We applied our proposed algorithms to several machine learning models, and validated their performance via both Matlab simulations and real-world testbed experiments. The experimental results show that, in comparison with the state-of-the-art AFL algorithms, MAPA can achieve not only much better trade-off between the model utility and privacy guarantee but also much higher convergence efficiency.

\bibliographystyle{IEEEtran}
\bibliography{mybibfile}

\begin{IEEEbiography}[{\includegraphics[width=1in,height=1.25in, clip,keepaspectratio]{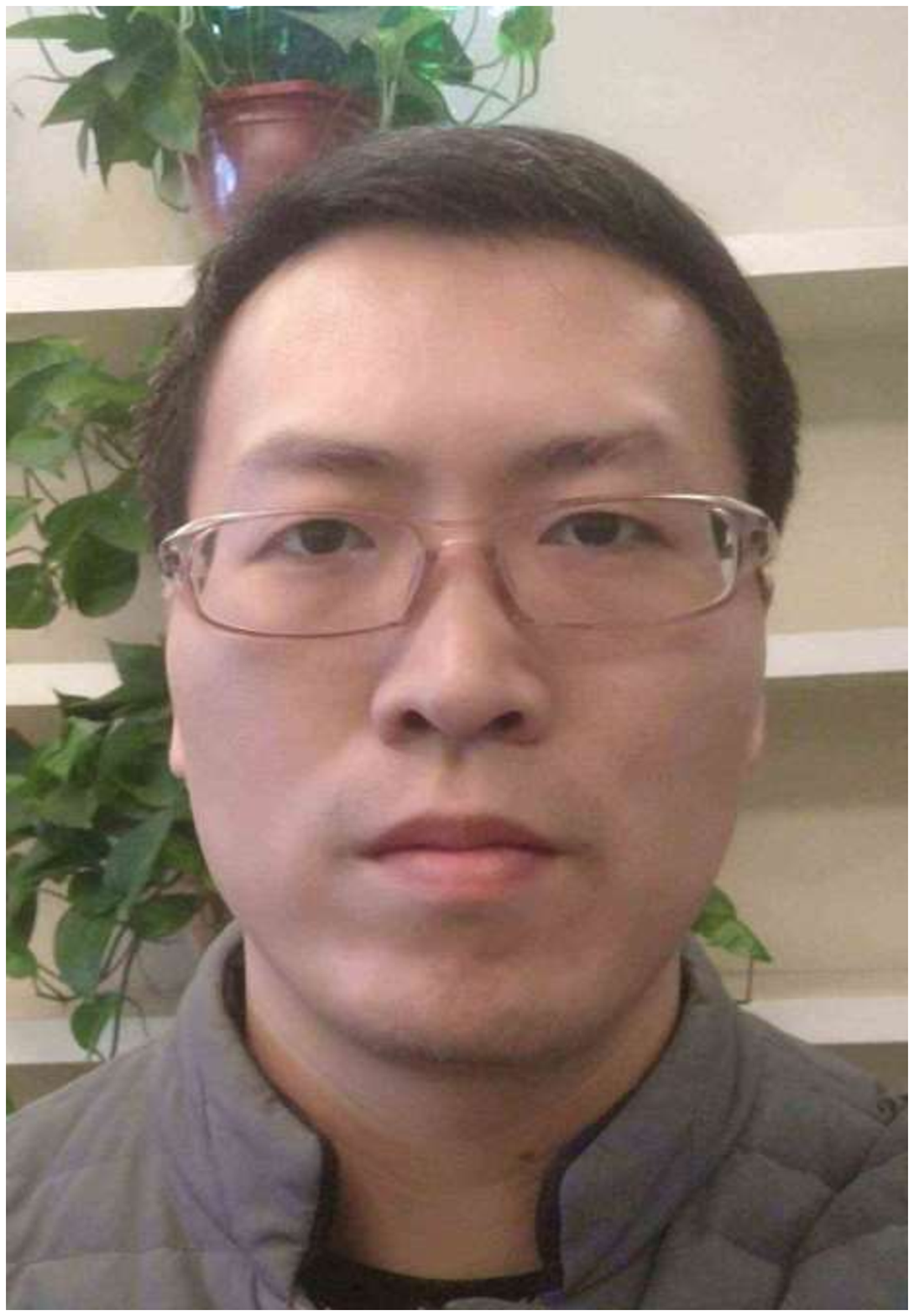}}] {Yanan Li} received his in Bachelor and Master degree from Henan Normal University of China in 2004 and 2007, respectively. He is currently working towards the PhD degree in the School of Mathematics and Statistics at Xi'an Jiaotong University. Before that, he worked as a lecturer in Henan Polytechnic University from 2007 to 2017. His research interests include differential privacy, federated learning, and edge computing.
\end{IEEEbiography}

\begin{IEEEbiography}[{\includegraphics[width=1in,height=1.25in, clip,keepaspectratio]{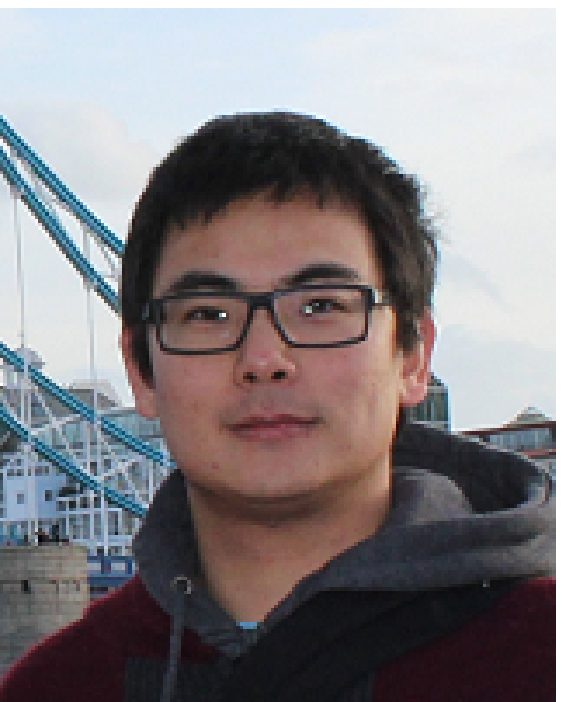}}] {Shusen Yang} received his PhD in Computing from Imperial College London in 2014. He is currently a professor in the Institute of Information and System Science at Xi'an Jiaotong University (XJTU). Before joining XJTU, Shusen worked as a Lecturer (Assistant Professor) at University of Liverpool from 2015 to 2016, and a Research Associate at Intel Collaborative Research Institute ICRI from 2013 to 2014. His research interests include mobile networks, networks with human in the loop, data-driven networked systems and edge computing. Shusen achieves  ``1000 Young Talents Program'' award, and holds an honorary research fellow at Imperial College London. Shusen is a senior member of IEEE and a member of ACM.
\end{IEEEbiography}

\begin{IEEEbiography}[{\includegraphics[width=1in,height=1.25in,clip,keepaspectratio]{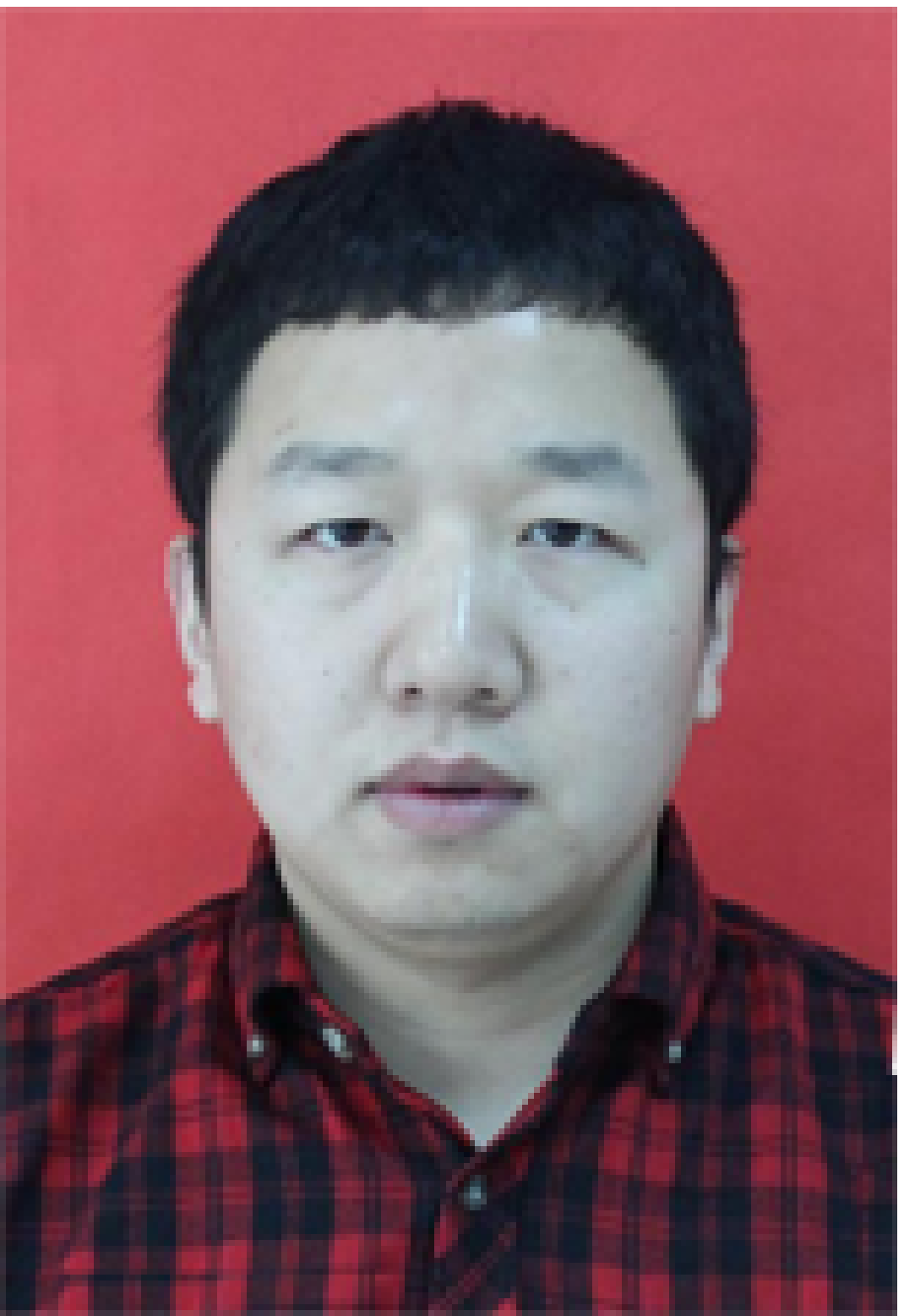}}]{Xuebin Ren} received his PhD degree in the Department of Computer Science and Technology from Xi'an Jiaotong University, China in 2017. Currently, he is a lecturer in Xi'an Jiaotong University and a member of National Engineering Laboratory for Big Data Analytics (NEL-BDA). He has been a visiting PhD student in the Department of Computing at Imperial College London from 2016 to 2017. His	research interests focus on data privacy protection, federated learning and privacy-preserving machine learning. He is a member of the IEEE and the ACM.
\end{IEEEbiography}

\begin{IEEEbiography}[{\includegraphics[width=1in,height=1.25in, clip,keepaspectratio]{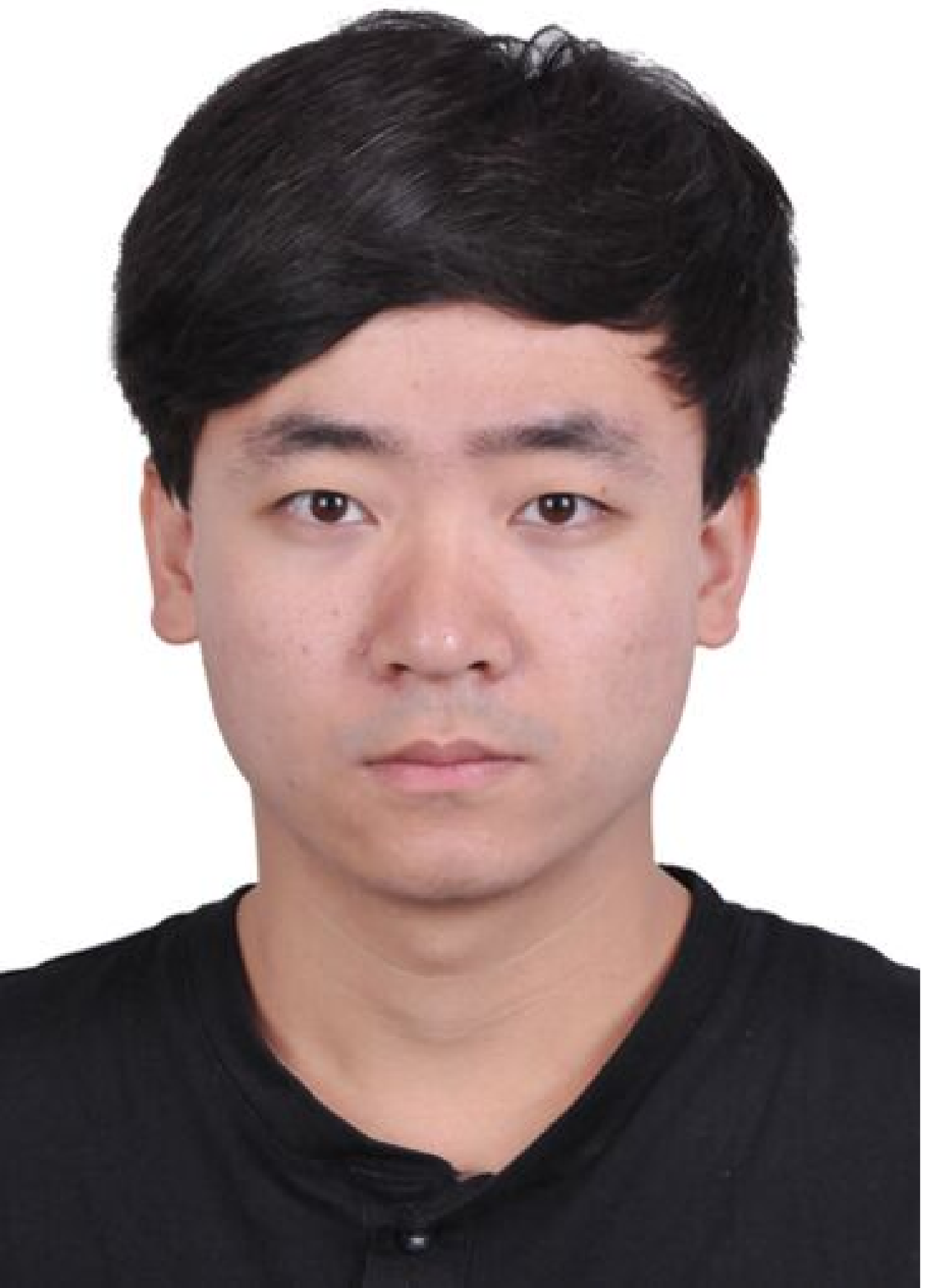}}] {Cong Zhao} received his Ph.D. degree from
	Xi’an Jiaotong University (XJTU) in 2017. He is currently a Research
	Associate in the Department of Computing at Imperial College London. His research interests include edge computing, computing economics, and people-centric sensing.
\end{IEEEbiography}

\clearpage
\appendices
\section{Proof of Theorems}
\label{appendix:proofs}

\subsection{Proof of Lemma \ref{lemma:convex_errorbound}}
\label{proof:lemma_convex_errorbound}

We give a lemma before the formal proof of Lemma \ref{lemma:convex_errorbound}.
\begin{lemma}\label{lemma:for_convex}
Let Assumption \ref{assumption:unbiased_smooth_variance}, $\|x-x^*\|\leq R$ and $\|\nabla f(x)\|\leq G$ hold. Then, we have
\begin{align*}
& \sum_{t=1}^{T}\mathbb{E}\langle \nabla f(x_{t})-\nabla f(x_{t-\tau(t)}),x_{t+1}-x^* \rangle\\
&\leq  Rc\tau_{max} +\frac{L(\tau_{max}+1)^2}{2}\sum_{t=1}^{T}\gamma_{t}^2\mathbb{E}\|\tilde{g}_{t-\tau(t)}\|^{2}.
\end{align*}
\end{lemma}

\begin{proof}
The proof follows by using a few Bregman divergence identities to rewrite the inner production. Let $D_f(\cdot,\cdot)$ is the Bregman divergence of $f$ \cite{bregman1967relaxation} which is defined as
\begin{align}\label{eq: definition of bregman divergence}
D_f(x,y):=f(x)-f(y)-\langle \nabla f(y),x-y \rangle.
\end{align}
Based on the following well-known four term equality, a consequence of straightforward algebra: for any $a,b,c,d$,
\begin{align*}
&\langle \nabla f(a)-\nabla f(b), c-d \rangle \\
&= D_{f}\langle d,a \rangle - D_{f}\langle d,b \rangle - D_{f}\langle c,a \rangle + D_{f}\langle c,b \rangle.
\end{align*}
We have
\begin{align}
&\langle \nabla f(x_{t})-\nabla f(x_{t-\tau(t)}),x_{t+1}-x^* \rangle\notag\\
& =D_{f}(x^*,x_{t})-D_{f}(x^*,x_{t-\tau(t)})-D_{f}(x_{t+1},x_{t}) \notag\\
&\quad +D_{f}(x_{t+1},x_{t-\tau(t)})\notag\\
&\leq D_{f}(x^*,x_{t})-D_{f}(x^*,x_{t-\tau(t)})+L/2\|x_{t+1}-x_{t-\tau(t)}\|^2.\label{eq:proof_lemma1_main}
\end{align}
In the last inequality, we drop the non-negative term $D_{f}(x_{t+1},x_{t})$, and use
\begin{align*}
D_{f}(x_{t+1},x_{t-\tau(t)})\leq L/2\|x_{t_{t+1}}-x_{t-\tau(t)}\|^2,
\end{align*}
which is derived from Eq.\;(\ref{eq: definition of bregman divergence}) and smooth gradient.
	
Taking expectation on both sides of Eq.\;(\ref{eq:proof_lemma1_main}), and summation $t$ from 1 to $T$, we have
\begin{align}
& \sum_{t=1}^{T}\mathbb{E}\langle \nabla f(x_{t})-\nabla f(x_{t-\tau(t)}),x_{t+1}-x^* \rangle \notag\\
&\leq \sum_{t=1}^{T} \mathbb{E}[D_{f}(x^*,x_{t})-D_{f}(x^*,x_{t-\tau(t)})] \notag\\
&\quad +\frac{L}{2}\sum_{t=1}^{T} \mathbb{E}\|\sum_{k=t-\tau(t)}^{k=t}x_{k}-x_{k+1}\|^2 \notag\\
&\leq \sum_{t=T-\tau_{max}+1}^{T}\mathbb{E}D_{f}(x^*,x_{t}) \notag\\
&+ \frac{L}{2}\sum_{t=1}^{T}(\tau(t)+1)\sum_{k=t-\tau(t)}^{t}\mathbb{E}\|x_{k}-x_{k+1}\|^2. \label{eq:proof_lemma1_main2}
\end{align}
For Bregman divergence $D_{f}(x^*,x_{t})$ in Eq.\;(\ref{eq:proof_lemma1_main2}), we have
\begin{align}
D_{f}(x^*,x_{t}) &=  f(x^*)-f(x_{t})-\langle \nabla f(x_{t}),x^*-x_{t} \rangle \notag\\
&\leq \|\nabla f(x_{t})\|_{*}\|x^*-x_{t}\|\leq RG.\label{eq:proof_lemma1_main3}
\end{align}
Next, we bound the remaining term in Eq.\;(\ref{eq:proof_lemma1_main2}).
\begin{align*}
&\sum_{t=1}^{T}(\tau(t)+1)\sum_{k=t-\tau(t)}^{t}\mathbb{E}\|x_{k}-x_{k+1}\|^2 \notag\\
&\leq \sum_{t=1}^{T}(\tau(t)+1)\sum_{k=t-\tau(t)}^{t} \gamma_{k}^2\mathbb{E}\|\tilde{g}_{k-\tau(k)} \|^2 \notag\\
&\leq (\tau_{max}+1)\sum_{t=1}^{T}\sum_{k=t-\tau(t)}^{t} \gamma_{k}^2 (\mathbb{E}\|\tilde{g}_{k-\tau(k)}\|^2) \notag\\
&\leq (\tau_{max}+1)^2 \sum_{t=1}^{T}\gamma_{t}^2 (\mathbb{E}\|\tilde{g}_{t-\tau(t)}\|^2).
\end{align*}
Substituting this result and Eq.\;(\ref{eq:proof_lemma1_main3}) into Eqs.\;(\ref{eq:proof_lemma1_main2}) completes the proof.
\end{proof}

Now, we prove Lemma \ref{lemma:convex_errorbound}.

Based on the $L$-Lipschitz continuity of gradient and convexity of function, we have
\begin{align}
&\mathbb{E}f(x_{t+1})-f(x^*) \notag\\
& \leq \mathbb{E}\langle \nabla f(x_{t}),x_{t+1}-x^* \rangle + \frac{L}{2}\mathbb{E}\|x_{t+1}-x_{t}\|^2 \notag\\
& = \underbrace{\mathbb{E}\langle \nabla f(x_{t})-\nabla f(x_{t-\tau(t)}),x_{t+1}-x^* \rangle}_{T_1} \notag\\
&\quad + \underbrace{\mathbb{E}\langle \nabla f(x_{t-\tau(t)}) - \tilde{g}_{t-\tau(t)},x_{t+1}-x^* \rangle}_{T_2} \notag\\
&\quad + \underbrace{\mathbb{E}\langle \tilde{g}_{t-\tau(t)},x_{t+1}-x^* \rangle}_{T_3} + L\gamma_{t}^{2}/2\mathbb{E}\|\tilde{g}_{t-\tau(t)}\|^2. \label{eq:proof_convex_general_main}
\end{align}
With respect to $T_1$, by Lemma \ref{lemma:for_convex}, we have
\begin{align}
\sum_{t=1}^{T}T_1\leq  RG\tau_{max} +\frac{L(\tau_{max}+1)^2}{2}\sum_{t=1}^{T}\gamma_{t}^2\mathbb{E}\|\tilde{g}_{t-\tau(t)}\|^{2}.\label{eq:proof_convex_general_T1}
\end{align}
With respect to $T_2$, we have
\begin{align}
T_2& = \mathbb{E}\langle \nabla f(x_{t-\tau(t)}) - \tilde{g}_{t-\tau(t)},x_{t+1}-x_{t} \rangle\notag\\
&\qquad + \mathbb{E}\langle \nabla f(x_{t-\tau(t)}) - \tilde{g}_{t-\tau(t)},x_{t}-x^* \rangle\notag\\
& = \mathbb{E}\langle \nabla f(x_{t-\tau(t)}) - \tilde{g}_{t-\tau(t)},-\gamma_{t}\tilde{g}_{t-\tau(t)} \rangle\notag\\
&= -\gamma_{t}\mathbb{E}\langle \nabla f(x_{t-\tau(t)}),\tilde{g}_{t-\tau(t)} \rangle + \gamma_{t}\mathbb{E}\|\tilde{g}_{t-\tau(t)}\|^2\notag\\
& = -\gamma_{t}\|\nabla f(x_{t-\tau(t)})\|^2 + \gamma_{t}\mathbb{E}\|\tilde{g}_{t-\tau(t)}-\nabla f(x_{t-\tau(t)})\|^2\notag\\
&\quad + \gamma_{t}\|\nabla f(x_{t-\tau(t)})\|^2\notag\\
& = \gamma_{t}\mathbb{E}\|\tilde{g}_{t-\tau(t)}-\nabla f(x_{t-\tau(t)})\|^2.\label{eq:proof_convex_general_T2}
\end{align}
The second equality in Eq.\;(\ref{eq:proof_convex_general_T2}) follows
\begin{align*}
\mathbb{E}\langle \nabla f(x_{t-\tau(t)}) - \bar{g}_{t-\tau(t)},x_{t}-x^* \rangle=0.
\end{align*}
The fourth equality in Eq.\;(\ref{eq:proof_convex_general_T2}) follows
\begin{align*}
&\mathbb{E}\langle \nabla f(x_{t-\tau(t)}),\tilde{g}_{t-\tau(t)} \rangle=\|\nabla f(x_{t-\tau(t)})\|^2,\\
&\mathbb{E}\langle\tilde{g}_{t-\tau(t)}-\nabla f(x_{t-\tau(t)}),\nabla f(x_{t-\tau(t)})\rangle=0.
\end{align*}
The last equality in Eq.\;(\ref{eq:proof_convex_general_T2}) follows
\begin{align*}
&\mathbb{E}\langle \tilde{g}_{t-\tau(t)}-\nabla f(x_{t-\tau(t)}),\tilde{g}_{t-\tau(t)}-\nabla f(x_{t-\tau(t)}) \rangle\\
&=\mathbb{E}\|\tilde{g}_{t-\tau(t)}\|^{2}+\mathbb{E}\|\nabla f(t-\tau_{(t)})\|^{2}-2\mathbb{E}\|\nabla f(t-\tau_{(t)})\|^{2}.
\end{align*}
With respect to $T_3$, we have
\begin{align}
T_3& = \mathbb{E}\langle \tilde{g}_{t-\tau(t)},x_{t+1}-x_{t} \rangle + \mathbb{E}\langle \tilde{g}_{t-\tau(t)},x_{t}-x^* \rangle\notag\\
&= \mathbb{E}\langle \tilde{g}_{t-\tau(t)},-\gamma_{t}\tilde{g}_{t-\tau(t)} \rangle + 1/\gamma_{t}\mathbb{E}\langle \gamma_{t}\tilde{g}_{t-\tau(t)},x_{t}-x^* \rangle\notag\\
&= -\gamma_{t}\mathbb{E}\|\tilde{g}_{t-\tau(t)}\|^2 + \frac{1}{2\gamma_{t}}\mathbb{E}( \gamma_{t}^2\|\tilde{g}_{t-\tau(t)}\|^2\notag\\
& \quad + \|x_{t}-x^*\|^2-\|x_{t+1}-x^*\|^2)\notag\\
& = -\frac{\gamma_{t}}{2}\mathbb{E}\|\tilde{g}_{t-\tau(t)}\|^2 + \frac{1}{2\gamma_{t}}\mathbb{E}[\|x_{t}-x^*\|^2-\|x_{t+1}-x^*\|^2].\label{eq:proof_convex_general_T3}
\end{align}
The third equality uses the fact $\langle a,b \rangle = 1/2[\|a\|^2+\|b\|^2-\|a-b\|^2]$.

Taking summation on both sides of Eq.\;(\ref{eq:proof_convex_general_main}) from 1 to $T$, and replacing $T_1,T_2,T_3$ with upper bound of Eqs.\;(\ref{eq:proof_convex_general_T1}), (\ref{eq:proof_convex_general_T2}) and (\ref{eq:proof_convex_general_T3}), we have
\begin{align*}
& \sum_{t=1}^{T}\mathbb{E}f(x_{t+1})-f(x^*) \\
&\leq RG\tau_{max} +\frac{L(\tau_{max}+1)^2}{2}\sum_{t=1}^{T}\gamma_{t}^2\mathbb{E}\|\tilde{g}_{t-\tau(t)}\|^{2}\\
&\quad + \sum_{t=1}^{T}\gamma_{t}\mathbb{E}\|\tilde{g}_{t-\tau(t)}-\nabla f(x_{t-\tau(t)})\|^2 \\
&\quad - \sum_{t=1}^{T}\frac{1}{2}(\gamma_{t} - L\gamma_{t}^2)\mathbb{E}\|\tilde{g}_{t-\tau(t)}\|^2\\
&\quad +\sum_{t=1}^{T}\frac{1}{2\gamma_{t}}[\|x_{t}-x^*\|^2 - \|x_{t+1}-x^*\|^2].
\end{align*}

\subsection{Proof of Theorem \ref{theorem:convex_errorbound}}
\label{proof:theorem_convex_errorbound}

When $\gamma_{t}$ is set as $\gamma_{t}^{-1}=L(\tau_{max}+1)+\sqrt{\Delta_{b}+1}\sqrt{t}$, obviously for all $t\in\mathbb{N}^{+}$, $\gamma_{t}\in(0,1/L)$.Therefore we drop the minus term
$\frac{1}{2}\sum_{t=1}^{T}\left(\gamma_{t}-L\gamma_{t}^{2}\right)\mathbb{E}\|\tilde{g}_{t-\tau_{(t)}}\|^{2}.$
Due to
\begin{align*}
&\mathbb{E}\|\tilde{g}_{t-\tau(t)}-\nabla f(x_{t-\tau(t)})\|^{2}\\
&=\mathbb{E}\|g_{t-\tau(t)}-\nabla f(x_{t-\tau(t)})\|^{2}+\mathbb{E}\|\eta_{t}\|^{2}\\
&\leq \sigma^{2}/b + 2\Delta S^{2}/\varepsilon_{k(t)}^{2}\leq \Delta_{b}.
\end{align*}
and
\begin{align*}
&\mathbb{E}\|\tilde{g}_{t-\tau_{(t)}}\|^{2}\\
&=\mathbb{E}\|\tilde{g}_{t-\tau_{(t)}}-\nabla f(x_{t-\tau(t)})\|^{2}+\mathbb{E}\|\nabla f(x_{t-\tau(t)})\|^{2}\\
&\leq \Delta_{b}+G^{2},
\end{align*}
we have
\begin{align}
& \frac{1}{T}\sum_{t=1}^{T}\mathbb{E}f(x_{t+1})-f(x^*) \notag\\
&\leq \frac{RG\tau_{max}}{T} +\frac{L(\tau_{max}+1)^2(\Delta_{b}+G^{2})}{2T}\sum_{t=1}^{T}\gamma_{t}^2\notag\\
&\quad + \frac{\Delta_{b}}{T}\sum_{t=1}^{T}\gamma_{t}
 +\frac{1}{T}\sum_{t=1}^{T}\frac{1}{2\gamma_{t}}(\|x_{t}-x^*\|^2 - \|x_{t+1}-x^*\|^2).\label{eq:convex_errorbound}
\end{align}
By observing that
\begin{align}
\sum_{t=1}^{T}\gamma_{t}^{2} &\leq\sum_{t=1}^{T}\frac{1}{(\Delta_{b}+1)t}\leq\frac{\log T}{\Delta_{b}+1},\label{eq:convex_error_t1}\\
\sum_{t=1}^{T}\gamma_{t} &\leq\sum_{t=1}^{T}\frac{1}{\sqrt{\Delta_{b}+1}}\frac{1}{\sqrt{t}}
\leq \frac{2\sqrt{T}}{\sqrt{\Delta_{b}+1}},\label{eq:convex_error_t2}
\end{align}
and
\begin{align}
&\sum_{t=1}^{T}\frac{1}{2\gamma_{t}}(\|x_{t}-x^*\|^2 - \|x_{t+1}-x^*\|^2)\notag\\
&\leq \frac{R^{2}}{2\gamma_{1}}+ \frac{\sqrt{\Delta_{b}+1}R^{2}/2}{\sqrt{T}},\label{eq:convex_error_t3}
\end{align}
we complete the proof after returning Eqs.\;(\ref{eq:convex_error_t1}), (\ref{eq:convex_error_t2}) and (\ref{eq:convex_error_t3}) back into Eq.\;(\ref{eq:convex_errorbound}).

\subsection{Proof of Theorem \ref{theorem:errorbound_gradient}}
\label{appendix:proofs_errorbound_gradient}

The essential idea is using the properties of smooth function and several inequality to separate the gradient. Recall the update formula
\begin{align*}
x_{t+1}-x_{t}=-\gamma_{t}\tilde{g}_{t-\tau(t)}, \tilde{g}_{t-\tau(t)}=g_{t-\tau(t)}+\eta_{t},
\end{align*}
where $\eta_{t}$ follows the density function Eq.\;(\ref{eq:noise_density}),
\begin{align*}
\mathbb{E}(\eta)=0, \mathbb{E}(\|\eta\|^{2})=2\Delta S_{k(t)}^2/\varepsilon_{k(t)}^2.
\end{align*}
Based on the Lipschitz continuity of gradient, we have
\begin{align*}
& f(x_{t+1})-f(x_{t}) \\
&\leq \langle \nabla f(x_{t}),x_{t+1}-x_{t} \rangle + L/2\|x_{t+1}-x_{t}\|^2\\
&\leq -\gamma_{t}\langle \nabla f(x_{t}),(\tilde{g}_{t-\tau(t)}) \rangle+ L\gamma_{t}^2/2\mathbb{E}\|\tilde{g}_{t-\tau(t)}\|^2.
\end{align*}
Taking expectation respect to $\eta_{t}$ and $\xi$, we have
\begin{align*}
&\mathbb{E}\langle \nabla f(x_{t}),(\tilde{g}_{t-\tau(t)}) \rangle=\langle \nabla f(x_{t}),\nabla f(x_{t-\tau(t)}) \rangle\\
&=\frac{1}{2}\left( \|\nabla f(x_{t})\|^{2}+\|\nabla f(x_{t-\tau(t)})\|^{2}\right.\\
&\quad \left. -\|\nabla f(x_{t})-\nabla f(x_{t-\tau(t)})\|^{2} \right),
\end{align*}
where we use the unbiased estimation in the first equality and the second equality uses the fact that
\begin{align*}
\langle a,b \rangle =\frac{1}{2}\left( \|a\|^{2}+\|b\|^{2}-\|a-b\|^{2} \right).
\end{align*}
So we have
\begin{align}
&\mathbb{E} [f(x_{t_{t+1}}) - f(x_{t})]\notag\\
&\leq -\frac{\gamma_{t}}{2}\left( \|\nabla f(x_{t})\|^{2}+\|\nabla f(x_{t-\tau(t)})\|^{2}\right. \notag\\
& -\underbrace{\|\nabla f(x_{t})-\nabla f(x_{t-\tau(t)})\|^{2}}_{T_1} ) + L\gamma_{t}^2/2\underbrace{\mathbb{E}\|\tilde{g}_{t-\tau(t)}\|^2}_{T_2}.\label{eq:proof_th1_error}
\end{align}
Next we estimate the upper bound of $T_1$ and $T_2$. For $T_2$, we have
\begin{align}
T_2&=\mathbb{E}\|\tilde{g}_{t-\tau(t)}\|^2\notag\\
&=\mathbb{E}\|\tilde{g}_{t-\tau(t)}-\nabla f(x_{t-\tau(t)})+\nabla f(x_{t-\tau(t)})\|^2\notag\\
&=\mathbb{E}\|\tilde{g}_{t-\tau(t)}-\nabla f(x_{t-\tau(t)})\|^{2}+\mathbb{E}\|\nabla f(x_{t-\tau(t)})\|^2\notag\\
&=\mathbb{E}\|g_{t-\tau(t)}-\nabla f(x_{t-\tau(t)})\|^{2}+\mathbb{E}\|\eta_{t}\|^{2}+\mathbb{E}\|\nabla f(x_{t-\tau(t)})\|^2\notag\\
&\leq \sigma^{2}/b+2\Delta S_{k(t)}^{2}/\varepsilon_{k(t)}^{2}+\mathbb{E}\|\nabla f(x_{t-\tau(t)})\|^2\notag\\
&\leq \Delta_{b}+\mathbb{E}\|\nabla f(x_{t-\tau(t)})\|^2,\label{eq:proof_th1_T2}
\end{align}
where the third equality is due to
\begin{align*}
&\mathbb{E}_{\xi|\eta_{t}}\langle \mathbb{E}_{\eta_{t}}\tilde{g}_{t-\tau(t)}-\nabla f(x_{t-\tau(t)}),\nabla f(x_{t-\tau(t)}) \rangle\\
&=\mathbb{E}_{\xi}\langle g_{t-\tau(t)}-\nabla f(x_{t-\tau(t)}),\nabla f(x_{t-\tau(t)}) \rangle\\
&=\langle \mathbb{E}_{\xi} g_{t-\tau(t)}-\nabla f(x_{t-\tau(t)}),\nabla f(x_{t-\tau(t)}) \rangle=0
\end{align*}
and the fourth equality is due to
\begin{align*}
&\mathbb{E}\|\tilde{g}_{t-\tau(t)}-\nabla f(x_{t-\tau(t)})\|^{2}\\
&=\mathbb{E}\|g_{t-\tau(t)}-\nabla f(x_{t-\tau(t)})+\eta_{t}\|^{2}\\
&=\mathbb{E}\|g_{t-\tau(t)}-\nabla f(x_{t-\tau(t)})\|^{2}+\mathbb{E}\|\eta_{t}\|^{2}\\
&\leq \sigma^{2}/b+2\Delta S_{k}^{2}/\varepsilon_{k}^{2}.
\end{align*}
With respect to $T_1$, we have
\begin{align}
T_1&=\|\nabla f(x_{t})-\nabla f(x_{t-\tau(t)})\|^{2}\leq L^{2}\|x_{t}-x_{t-\tau(t)}\|^{2}\notag\\
&=L^{2}\|\sum_{j=t-\tau(t)}^{t-1}x_{j+1}-x_{j}\|^{2}=L^{2}\|\sum_{j=t-\tau_{(t)}}^{t-1}\gamma_{j}\tilde{g}_{j-\tau(j)}\|^{2}\notag\\
&\leq 2L^{2}\underbrace{\|\sum_{j=t-\tau(t)}^{t-1}\gamma_{j}\left(\tilde{g}_{j-\tau(j)}-\nabla f(x_{j}-\tau(j))\right)\|^{2}}_{T_3}\notag\\
&\quad +2L^{2}\underbrace{\|\sum_{j=t-\tau(t)}^{t-1}\gamma_{j}\nabla f(x_{j}-\tau(j))\|^{2}}_{T_4}\label{eq:proof_th1_T1}
\end{align}
where the last inequality is derived from the fact that $\|a\|^{2}=\|a-b+b\|^{2}\leq \|a-b\|^{2}+\|b\|^{2}$.
With respect to $T_3$, we have
\begin{align*}
\mathbb{E}T_3&=\mathbb{E}_{\xi|\eta_{j}}(\mathbb{E}_{\eta_{j}}T_3)\\
&=\mathbb{E}\sum_{j=t-\tau(t)}^{t-1}\gamma_{j}^{2}\|\tilde{g}_{j-\tau(j)}-\nabla f(x_{j-\tau(j)})\|^{2}\\
&=\mathbb{E}\sum_{j=t-\tau(t)}^{t-1}\gamma_{j}^{2}\left(\|g_{j-\tau(j)}-\nabla f(x_{j-\tau(j)})\|^{2}+\mathbb{E}\|\eta_{j}\|^{2}\right)\\
&\leq \mathbb{E}\sum_{j=t-\tau_{max}}^{t-1}\gamma_{j}^{2}\left(\|g_{j-\tau(j)}-\nabla f(x_{j-\tau(j)})\|^{2}+\mathbb{E}\|\eta_{j}\|^{2}\right)\\
&\leq\sum_{j=t-\tau_{max}}^{t-1}\gamma_{j}^{2}\left(\sigma^{2}/b+2\Delta S_{k(j)}^{2}/\varepsilon_{k(j)}^{2}\right)\leq \Delta_{b}\sum_{j=t-\tau_{max}}^{t-1}\gamma_{j}^{2}.
\end{align*}
With respect to $T_4$, we have
\begin{align*}
&\mathbb{E}T_4\leq \tau(t)\sum_{j=t-\tau(t)}^{t-1}\gamma_{j}^{2}\|\nabla f(x_{j-\tau(j)})\|^{2}\\
&\leq \tau_{max}\sum_{j=t-\tau_{max}}^{t-1}\gamma_{j}^{2}\|\nabla f(x_{j-\tau(j)})\|^{2}
\end{align*}
Taking full expectation on both sides of Eq.\;(\ref{eq:proof_th1_T1}) and replacing $\mathbb{E}T_3$ and $\mathbb{E}T_4$ with their upper bound, we have
\begin{align}
\mathbb{E}T_1\leq 2L^{2}&\left( \Delta_{b}\sum_{j=t-\tau_{max}}^{t-1}\gamma_{j}^{2}\right.\notag\\
&\left.+\tau_{max}\sum_{j=t-\tau_{max}}^{t-1}\gamma_{j}^{2}\|\nabla f(x_{j-\tau(j)})\|^{2} \right)\label{eq:proof_th1_ET1}.
\end{align}
Taking Eqs.\;(\ref{eq:proof_th1_ET1}) and (\ref{eq:proof_th1_T2}) back into Eq.\;(\ref{eq:proof_th1_error}), we have
\begin{align*}
&\mathbb{E}f(x_{t+1})-f(x_{t})\\
&\leq-\frac{\gamma_{t}}{2}\left(\mathbb{E}\|\nabla f(x_{t})\|^{2}+\mathbb{E}\|\nabla f(x_{t-\tau(t)})\|^{2}\right)\\
&+L^{2}\gamma_{t}\left( \Delta_{b}\sum_{j=t-\tau_{max}}^{t-1}\gamma_{j}^{2}+\tau_{max}\sum_{j=t-\tau_{max}}^{t-1}\gamma_{j}^{2}\mathbb{E}\|\nabla f(x_{j-\tau(j)})\|^{2} \right)\\
&+\frac{L\gamma_{t}^{2}}{2}\left(\Delta_{b}+\mathbb{E}\|\nabla f(x_{t-\tau(t)})\|^{2}\right)\\
&\leq -\frac{\gamma_{t}}{2}\mathbb{E}\|\nabla f(x_{t})\|^{2}+\left(\frac{L\gamma_{t}^{2}-\gamma_{t}}{2}\right)\mathbb{E}\|\nabla f(x_{t-\tau(t)})\|^{2}\\
&+\Delta_{b}\left( \frac{L\gamma_{t}^{2}}{2}+L^{2}\gamma_{t}\sum_{j=t-\tau_{max}}^{t-1}\gamma_{j}^{2} \right)\\
&+L^{2}\gamma_{t}\tau_{max}\sum_{j=t-\tau_{max}}^{t-1}\gamma_{j}^{2}\mathbb{E}\|\nabla f(x_{j-\tau(j)})\|^{2}.
\end{align*}
Taking summation on $t$ from 1 to $T$ and rearranging terms, we have
\begin{align*}
&\mathbb{E}f(x_{T+1})-f(x_{1})\\
&\leq -\frac{1}{2}\sum_{t=1}^{T}\gamma_{t}\mathbb{E}\|\nabla f(x_{t})\|^{2}\\
&+\sum_{t=1}^{T}\left( \gamma_{t}^{2}(\frac{L}{2}+L^{2}\tau_{max}\sum_{\rho=1}^{\tau_{max}}\gamma_{t+\rho})-\frac{\gamma_{t}}{2} \right)\mathbb{E}\|\nabla f(x_{t-\tau(t)})\|^{2}\\
&+\Delta_{b}\sum_{t=1}^{T}\left( \frac{L\gamma_{t}^{2}}{2}+L^{2}\gamma_{t}\sum_{j=t-\tau_{max}}^{t-1}\gamma_{j}^{2} \right).
\end{align*}
If $\gamma_{t}$ is set as a constant $1/(2L(\tau_{max}+1))$ (Eq.\;(\ref{eq:stepsize_constant})), then
\begin{align*}
\gamma_{t}^{2}\left(\frac{L}{2}+L^{2}\tau_{max}\sum_{\rho=1}^{\tau_{max}}\gamma_{t+\rho}\right)-\frac{\gamma_{t}}{2}<0
\end{align*}
is always hold. Dropping this minus term and taking summation on $t$ from $1$ to $T$, we have
\begin{align*}
&\sum_{t=1}^{T}\gamma\mathbb{E}\|\nabla f(x_{t})\|^{2}\leq 2(f(x_{1})-f(x^{*}))\\
&+\Delta_{b}\sum_{t=1}^{T}\left( {L\gamma^{2}}+2L^{2}\gamma^{3}\tau_{max} \right),
\end{align*}
where we ue the fact that $f(x_{1})-f(x_{T+1})\leq f(x_{1})-f(x*)$. Dividing $T$ on both sides with $\gamma T$, we have
\begin{align*}
&\frac{1}{T}\sum_{t=1}^{T}\mathbb{E}\|\nabla f(x_{t})\|^{2}\leq\frac{2(f(x_{1})-f(x^{*}))+\Delta_{b}T2L\gamma^{2}}{T\gamma},
\end{align*}
in where $2L^{2}\gamma^{3}\tau_{max}\leq L\gamma^{2}$. Therefore,
\begin{align*}
\frac{1}{T}\sum_{t=1}^{T}\mathbb{E}\|\nabla f(x_{t})\|^{2}\leq\frac{2(f(x_{1})-f(x^{*}))}{T\gamma}+2\Delta_{b}L\gamma
\end{align*}
completes the proof.

\end{document}